\documentclass[nohyperref]{article}

\usepackage{microtype}
\usepackage{graphicx}
\usepackage{subfigure}
\usepackage{booktabs} % for professional tables

\usepackage{url}            % simple URL typesetting
\usepackage{amsfonts}       % blackboard math symbols
\usepackage{nicefrac}       % compact symbols for 1/2, etc.
\usepackage{xcolor}         % colors

\usepackage{bm}
\usepackage{bbm}
\usepackage{appendix}
\usepackage{hyperref}

\newcommand{\Real}{\mathbb{R}}
\usepackage[accepted]{icml2023}

\usepackage{amsmath}
\usepackage{amssymb}
\usepackage{mathtools}
\usepackage{amsthm}

\usepackage[capitalize,noabbrev]{cleveref}

\theoremstyle{plain}
\newtheorem{theorem}{Theorem}[section]
\newtheorem{proposition}[theorem]{Proposition}
\newtheorem{lemma}[theorem]{Lemma}

\theoremstyle{definition}
\newtheorem{definition}[theorem]{Definition}
\newtheorem{assumption}[theorem]{Assumption}
\theoremstyle{remark}

\newtheorem{example}[section]{Example}

\usepackage[textsize=tiny]{todonotes}

\newcommand{\ie}{\textit{i}.\textit{e}.}
\newcommand{\eg}{\textit{e}.\textit{g}.}

\icmltitlerunning{Policy Gradient in Robust MDPs}

\begin{document}

\twocolumn[
%\icmltitle{On the Convergence of Policy Gradient in Robust MDPs}
% What do you think about this title?
\icmltitle{Policy Gradient in Robust MDPs with Global Convergence Guarantee}



\icmlsetsymbol{equal}{*}

\begin{icmlauthorlist}
\icmlauthor{Qiuhao Wang}{city_ds}
\icmlauthor{Chin Pang Ho}{city_ds}
\icmlauthor{Marek Petrik}{unh}
\end{icmlauthorlist}

\icmlaffiliation{city_ds}{School of Data Science, City University of Hong Kong}
\icmlaffiliation{unh}{Department of Computer Science, University of New Hampshire}
%\icmlaffiliation{comp}{Company Name, Location, Country}
%\icmlaffiliation{sch}{School of ZZZ, Institute of WWW, Location, Country}

\icmlcorrespondingauthor{Chin Pang Ho}{clint.ho@cityu.edu.hk}
\icmlcorrespondingauthor{Marek Petrik}{mpetrik@cs.unh.edu}

\icmlkeywords{Machine Learning, ICML}

\vskip 0.3in
]



\printAffiliationsAndNotice{}  % leave blank if no need to mention equal contribution

\begin{abstract}
  Robust Markov decision processes~(RMDPs) provide a promising framework for computing reliable policies in the face of model errors. Many successful reinforcement learning algorithms build on variations of policy-gradient methods, but adapting these methods to RMDPs has been challenging. As a result, the applicability of RMDPs to large, practical domains remains limited.  This paper proposes a new Double-Loop Robust Policy Gradient (DRPG), the first generic policy gradient method for RMDPs. In contrast with prior robust policy gradient algorithms, DRPG monotonically reduces approximation errors to guarantee convergence to a globally optimal policy in tabular RMDPs. We introduce a novel parametric transition kernel and solve the inner loop robust policy via a gradient-based method. Finally, our numerical results demonstrate the utility of our new algorithm and confirm its global convergence properties.
\end{abstract}

\section{Introduction}\label{sec1}

Markov decision process~(MDP) is a standard model in dynamic decision-making and reinforcement learning~\citep{puterman2014markov,sutton2018reinforcement}. However, a fundamental challenge with using MDPs in many applications is that model parameters, such as the transition function, are rarely known precisely. Robust Markov decision processes (RMDPs) have emerged as an effective and promising approach for mitigating the impact of model ambiguity. RMDPs assume that the transition function resides in a predefined \emph{ambiguity set} and seek a policy that performs best for the worst-case transition function in the ambiguity set. Compared to MDPs, the performance of RMDPs is less sensitive to the parameter errors that arise when one estimates the transition function from empirical data, as is common in reinforcement learning~\citep{xu2009parametric,ICML2012Petrik_283,ghavamzadeh2016safe}. 

As is common in recent literature on RMDPs, we assume that the RMDP's ambiguity set satisfies certain rectangularity assumptions~\citep{wiesemann2013robust,ho2021partial,Panaganti2021}. Albeit general RMDPs are NP-hard to solve~\citep{wiesemann2013robust}, they become tractable under rectangularity assumptions and can be solved using dynamic programming~\citep{iyengar2005robust,nilim2005robust,kaufman2013robust,ho2021partial}. The simplest rectangularity assumption is known as $(s,a)$-rectangularity and allows the adversarial nature to choose the worst transition probability for each state and action independently. Because the $(s,a)$-rectangularity assumption can be too restrictive, we assume the more-general $s$-rectangular ambiguity set~\citep{le2007robust,wiesemann2013robust,derman2021twice,wang2022geometry}, which restricts the adversarial nature to choose a transition probability without observing the action. Our results also readily extend to other notions of rectangularity, including $k$-rectangular~\citep{mannor2016robust}, and $r$-rectangular RMDPs~\citep{goyal2022robust}. 

Policy gradient techniques have gained considerable popularity in reinforcement learning due to their remarkable empirical performance and flexibility in large and complex domains~\cite{silver2014deterministic,xu2014reinforcement}. By parameterizing policies, policy gradient methods easily scale to large state and action spaces, and they also easily leverage generic optimization techniques~\cite{konda1999actor,bhatnagar2009natural,petrik2014raam,pirotta2015policy,schulman2015trust,schulman2017proximal,behzadian2021fast}. In addition, recent work shows that many policy gradient algorithms are guaranteed to find a globally-optimal policy in tabular MDPs even though they optimize a non-convex objective function~\cite{agarwal2021theory,bhandari2021linear}. 

As our first contribution, we propose a new policy gradient method for solving $s$-rectangular RMDPs. We call this method the \emph{Double-Loop Robust Policy Gradient} (DRPG), because it is inspired by double-loop algorithms designed for solving saddle point problems~\citep{jin2020local,luo2020stochastic,razaviyayn2020nonconvex,zhang2020single}. In particular, DRPG solves RMDPs using two nested loops: an outer loop updates policies, and an inner loop approximately computes the worst-case transition probabilities. While the outer loop resembles policy gradient updates in regular MDPs, the inner loop must optimize over an infinite number of transition probabilities in the ambiguity set. To effectively optimize the continuous transition probabilities, we use a projected gradient method with a finite but complete parametrization in tabular MDPs. To scale the algorithm to large problems, we propose to use a parametrization based on KL-divergence ambiguity sets.

As our second contribution, we show that DRPG is guaranteed to converge to a globally optimal policy in $s$-rectangular RMDPs. While this result mirrors similar known results for ordinary MDPs, the robust setting involves several additional non-trivial challenges. Unlike in ordinary MDPs, the RMDP return is not differentiable in terms of the policy~\citep{razaviyayn2020nonconvex}, which precludes us from leveraging MDP results. Since the RMDP return is not convex, it also does not admit subgradients. Instead, we show that it is sufficient to approximate it by its Moreau envelope, which is differentiable. An additional challenge is that solving the inner loop optimally in every policy carries an unacceptable computational policy, but solving it approximately may cause oscillations. We address this problem by proposing a schedule of decreasing approximation errors that are sufficient to converge to the optimal solution. In fact, the policy updates are guaranteed to converge to the optimal policy as long as the inner loop can be solved with sufficient precision, even when the RMDP is non-rectangular. 

Despite the recent advances in robust reinforcement learning \cite{roy2017reinforcement, badrinath2021robust, wang2021online, panaganti2022sample}, policy gradient methods for solving RMDPs have received only limited attention. A concurrent work proposes a policy gradient method for solving RMDPs with a particular R-contamination ambiguity sets~\citep{pmlr-v162-wang22at}. While this algorithm is compellingly simple, the R-contamination set is very limited in comparison with the general sets that we consider. In fact, we show in Proposition~\ref{prop:r-contamination-useless} that RMDPs with R-contamination ambiguity sets simply equal to ordinary MDPs with a reduced discount factor; please see Appendix~\ref{sec:R_contamination} for more details. Another recent work develops an extended mirror descent method for solving RMDPs~\citep{li2022first}; however, their results are limited to $(s,a)$-rectangular MDPs only, and their algorithm requires the exact robust Q function to update the policy at every iteration. On the other hand, our proposed algorithm is compatible with any compact ambiguity set, and we do not require an exact optimal solution when solving the inner maximization problem. Moreover, by parameterizing the inner problem, the proposed algorithm is scalable to large problems.

While this paper exclusively focuses on RMDPs, it is worth mentioning that there is an active line of research studying a related model, called distributionally robust MDPs, which assumes the transition kernel is random and governed by an unknown probability distribution that lies in an ambiguity set \cite{ruszczynski2010risk,xu2010distributionally,shapiro2016rectangular,chen2019distributionally,clement2021first,shapiro2021distributionally,liu2022distributionally}.

The remainder of the paper is organized as follows. Section~\ref{sec:prel-backgr} outlines RMDP and optimization properties that are needed for our results. Then, Section~\ref{sec:main-results} describes the outer loop of DRPG, our proposed algorithm, and shows its global convergence guarantee. The algorithms for solving the inner loop are then described in Section~\ref{sec:inner-loop-maxim}. Finally, in Section~\ref{sec:experiments}, we present experimental results that illustrate the effective empirical performance of DRPG.

\textbf{Notation:} We reserve lowercase letters for scalars, lowercase bold characters for vectors, and uppercase bold characters for matrices. We denote $\Delta^{S}$ as the probability simplex in $\mathbb{R}^{S}_{+}$. For vectors, we use $\|\cdot\|$ to denote the $l_{2}$-norm. For a differentiable function $f(x,y)$, we use $\nabla_{x}f(x,y)$ to denote the partial gradient of $f$ with respect to $x$. The symbol $\bm{e}$ denotes a vector of all ones of the size appropriate to the context.

\section{Notations and Settings}\label{sec:prel-backgr}

% In this section, we introduce RMDPs and their properties. We also give a brief overview of the properties of non-convex minimax problems, which play an important role in proving the convergence of our proposed algorithm.

An ordinary MDP is specified by a tuple $\langle\mathcal{S},\mathcal{A},\bm{p},\bm{c},\gamma, \bm{\rho}\rangle$, where $\mathcal{S}=\{1,2,\cdots,S\}$ and $\mathcal{A}=\{1,2,\cdots,A\}$ are the finite state and action sets, respectively. The discount factor is $\gamma\in(0,1)$ and the distribution of the initial state is $\bm{\rho}\in\Delta^{S}$. The probability distribution of transiting from a current state $s$ to a next state $s'$ after taking an action $a$ is denoted as a vector $\bm{p}_{sa}\in\Delta^{S}$ and in a part of the transition kernel $\bm{p}:=(\bm{p}_{sa})_{s\in\mathcal{S},a\in\mathcal{A}}\in(\Delta^{S})^{S\times A}$. The cost of the aforementioned transition is denoted as $c_{sas'}$ for each $(s,a,s')\in\mathcal{S}\times\mathcal{A}\times\mathcal{S}$. It is well-known that translating the costs by a constant or multiplying them by a positive scalar does not change the set of optimal policies. Therefore, we can assume without loss of generality that the cost function is bounded in $[0,1]$.
\begin{assumption}[Bounded cost]\label{assump_1}
For any $\left(s,a,s'\right)\in\mathcal{S}\times\mathcal{A}\times\mathcal{S}$, the cost $c_{sas'}\in[0,1]$.
\end{assumption}

Given a stationary randomized policy $\bm{\pi}:=(\bm{\pi}_{s})_{s\in\mathcal{S}}$ that lies in the policy space $\Pi = (\Delta^{A})^{S}$, $\bm{\pi}$ maps from state $s\in\mathcal{S}$ to a distribution over action $a\in\mathcal{A}$, and the quality of a policy $\bm{\pi}$ is evaluated by the \emph{value function} $\bm{v}^{\bm{\pi},\bm{p}}\in\mathbb{R}^{S}$, defined as
\begin{equation*}
v^{\bm{\pi},\bm{p}}_{s} = \mathbb{E}_{\bm{\pi},\bm{p}}\left[\sum_{t=0}^{\infty} \gamma^{t}\cdot  c_{s_{t}a_{t}s_{t+1}} \mid s_{0}=s\right],
\end{equation*}
where $a_{t}$ follows the distribution $\bm{\pi}_{s_{t}}$, and $\mathbb{E}_{\bm{\pi},\bm{p}}$ denotes expectation with respect to the distribution induced by $\bm{\pi}$ and transition function $\bm{p}$ conditioned on the initial state event $\{s_{0} = s\}$. Similarly, the value of taking action $a$ at state $s$ is referred as the \emph{action value function} as below
\begin{equation*}
q^{\bm{\pi},\bm{p}}_{sa} = \mathbb{E}_{\bm{\pi},\bm{p}}\left[\sum_{t=0}^{\infty} \gamma^{t} c_{s_{t}a_{t}s_{t+1}} \mid s_{0}=s,a_{0}=a \right],
\end{equation*}
where it is known that $v^{\bm{\pi},\bm{p}}_{s} = \sum_{a\in \mathcal{A}}\pi_{sa}q^{\bm{\pi},\bm{p}}_{sa}$~\citep{puterman2014markov,sutton2018reinforcement}. The objective of an MDP is to compute the optimal policy $\bm{\pi}^\star$ that yields the minimum expected cost, \ie,
\begin{equation}\label{eq:NMdps}
    \bm{\pi}^\star = \arg\min_{\bm{\pi} \in \Pi} \mathbb{E}_{\bm{\pi},\bm{p}}\left[\sum_{t=0}^{\infty} \gamma^{t} c_{s_{t}a_{t}s_{t+1}}|s_{0}\sim\bm{\rho}\right].
\end{equation}
%In an ordinary MDP, the transition kernel and cost function are assumed to be known precisely.
In most domains, the exact transition kernel and cost function are not known precisely and must be estimated from data. These estimation errors often result in policies that perform poorly when deployed. To compute reliable policies with model errors, RMDPs, defined as $\langle\mathcal{S},\mathcal{A},\mathcal{P},\bm{c},\gamma, \bm{\rho}\rangle$, aim to optimize the worst-case performance with respect to plausible errors~\citep{iyengar2005robust,nilim2005robust,wiesemann2013robust}, \ie
\begin{equation}\label{prob_RMDP}
\min_{\bm{\pi}\in\Pi}\max_{\bm{p}\in\mathcal{P}} J_{\bm{\rho}}(\bm{\pi},\bm{p}):= \bm{\rho}^{\top}\bm{v}^{\bm{\pi},\bm{p}}=\sum_{s \in \mathcal{S}} \rho_{s}v^{\bm{\pi},\bm{p}}_{s},
\end{equation}
where $\mathcal{P}$ is known as the \emph{ambiguity set}. By carefully calibrating $\mathcal{P}$ so that it contains the unknown true transition kernel, the optimal policy in~\eqref{prob_RMDP} can achieve reliable performance in practice~\citep{Russell2019a,Behzadian2021,Panaganti2022}.

Note that, at this point, there is no need to assume that the RMDP in~\eqref{prob_RMDP} is rectangular, such as $(s,a)$-rectangular or $s$-rectangular~\citep{iyengar2005robust,nilim2005robust,wiesemann2013robust,ho2021partial}. We do not need these assumptions to describe or analyze DRPG and only require $\mathcal{P}$ to be compact. Rectangularity assumptions will be helpful, however, when developing algorithms for solving the inner maximization problem.

Given a specific policy and transition kernel, the \emph{occupancy measure} represents the frequencies of visits to states~\citep{puterman2014markov}, which is defined as follow.
\begin{definition}[Occupancy measure]\label{def:occu}
The discounted state occupancy measure $d_{\bm{\rho}}^{\bm{\pi},\bm{p}}\colon \mathcal{S} \rightarrow [0,1]$ for an initial distribution $\bm{\rho}$, a policy $\bm{\pi}\in\Pi$, and a transition kernel $\bm{p}$ is defined as
\begin{equation}
d_{\bm{\rho}}^{\bm{\pi},\bm{p}}(s')
=(1-\gamma) \sum_{s\in\mathcal{S}}\sum_{t=0}^{\infty} \gamma^{t} \rho(s)p^{\bm{\pi}}_{ss'}(t).
\end{equation}
Here, $p^{\bm{\pi}}_{ss'}(t)$ is the probability of arriving in a state $s^{\prime}$ after transiting $t$ time steps from state $s$ over the policy $\bm{\pi}$ and the transition kernel $\bm{p}$.
\end{definition}

The non-convex minimax problem in~\eqref{prob_RMDP} can be reformulated as an equivalent problem of minimizing the worst-case return:
\begin{equation}\label{prob_RMDP2}
    \min_{\bm{\pi}\in\Pi} \Big\{ \Phi(\bm{\pi}) := \max_{\bm{p}\in\mathcal{P}}J_{\bm{\rho}}(\bm{\pi},\bm{p}) \Big\}.
\end{equation}
Then, it may seem natural to solve~\eqref{prob_RMDP2} by a gradient descent on the function $\Phi$. This is, in general, not possible because the function $\Phi$ is not differentiable. In addition, since $\Phi$ is neither convex nor concave, its subgradient does not exist either~\citep{nouiehed2019solving,lin2020gradient}. These complications motivate the need for the  double-loop iterative scheme to solve RMDPs in Section~\ref{sec:main-results}. 

Next, we introduce two crucial definitions on smoothness and Lipschitz continuity, which we need to analyze DRPG.
\begin{definition}\label{def:Lip}
A function $h: \mathcal{X}\rightarrow\mathbb{R}$ is \emph{$L$-Lipschitz} if for any $\bm{x}_{1},\bm{x}_{2}\in\mathcal{X}$, we have that $\|h(\bm{x}_{1})-h(\bm{x}_{2})\|\leq L\|\bm{x}_{1}-\bm{x}_{2}\|$, and \emph{$\ell$-smooth} if for any $\bm{x}_{1},\bm{x}_{2}\in\mathcal{X}$, we have $\|\nabla h(\bm{x}_{1})-\nabla h(\bm{x}_{2})\|\leq \ell\|\bm{x}_{1}-\bm{x}_{2}\|$.
\end{definition}
To discuss the global optimality of RMDPs, we introduce the following definition of weak convexity officially.
\begin{definition}[Weak Convexity]
The function $h:\mathcal{X}\rightarrow\mathbb{R}$ is $\ell$-weakly convex if for any $\bm{g} \in \partial h(\bm{x})$ and $\bm{x}, \bm{x}'\in\mathcal{X}$,
\begin{equation*}
    h(\bm{x}')-h(\bm{x}) \geq\langle \bm{g}, \bm{x}'-\bm{x}\rangle-\frac{\ell}{2}\|\bm{x}'-\bm{x}\|^{2}.
\end{equation*}
\end{definition}
Here, $\partial h(\bm{x})$ represents the Fréchet sub-differential (See Definition~\ref{def:F-subdiff} in the appendix) of $h(\cdot)$ at $\bm{x} \in\mathcal{X}$, which generalizes the notion of gradient for the non-smooth function~\citep{vial1983strong,davis2019stochastic,thekumparampil2019efficient}.

\section{Solving the Outer Loop} \label{sec:main-results}

In this section, we describe a policy gradient approach that solves the minimization problem in~\eqref{prob_RMDP2}. Surprisingly, we show that a form of gradient descent applied to~\eqref{prob_RMDP2} converges to a globally-optimal solution, even though the objective function is neither convex nor concave. This result is inspired by the recent analysis of policy gradient methods for ordinary MDPs~\citep{agarwal2021theory,bhandari2021linear}. For now, we assume that there exists an oracle that solves the inner maximization problem. We provide the discussion and algorithms for solving the inner problem in \cref{sec:inner-loop-maxim}.

The remainder of the section is organized as follows. In \cref{sec:double-loop-robust}, we describe our new policy gradient scheme and then, in \cref{sec:convergence-analysis}, we show that our scheme is guaranteed to converge to the global solution. To the best of our knowledge, this is the first generic robust policy gradient algorithm with global convergence guarantees. 

\subsection{Double-Loop Robust Policy Gradient Method (DRPG)} \label{sec:double-loop-robust}

\begin{algorithm}[t]
\caption{Double-Loop Robust Policy Gradient (DRPG)}
\label{alg:DL-RPG}
\begin{algorithmic}
\STATE {\bfseries Input:} initial policy $\bm{\pi}_{0}$, iteration time $T$, tolerance sequence  $\{\epsilon_{t}\}_{t\geq0}$ such that $\epsilon_{t+1}\leq\gamma\epsilon_{t}$, step size sequence $\{\alpha_{t}\}_{t\geq0}$
\FOR{$t = 0,1,\dots,T-1$}
\STATE Find $\bm{p}_{t}$ so that $J_{\bm{\rho}}(\bm{\pi}_{t},\bm{p}_{t}) \geq \max_{\bm{p}\in\mathcal{P}} J_{\bm{\rho}}(\bm{\pi}_{t},\bm{p}) - \epsilon_{t}$.
\STATE Set $\bm{\pi}_{t+1} \gets \text{Proj}_{\Pi}(\bm{\pi}_{t} - \alpha_{t}\nabla_{\bm{\pi}}J_{\bm{\rho}}(\bm{\pi}_{t},\bm{p}_{t}))$. (Eq.~\eqref{eq:prox-update})
\ENDFOR
%\STATE {\color{blue}{\bfseries Output:} $\bm{\pi}_{t^{\star}}\in\mathop{\arg\min}_{\left\{\{\bm{\pi}_{t}\}^{T-1}_{t=0}\right\}}J_{\bm{\rho}}(\bm{\pi}_{t},\bm{p}_{t})$}
\STATE {\bfseries Output:} $\bm{\pi}_{t^{\star}} \in \{\bm{\pi}_{0},\dots,\bm{\pi}_{T-1}\}$ s.t. $ J_{\bm{\rho}}(\bm{\pi}_{t^{\star}},\bm{p}_{t}) = \min_{t' \in \{0,\dots,T-1\}} J_{\bm{\rho}}(\bm{\pi}_{t'},\bm{p}_{t})$
\end{algorithmic}
\end{algorithm}

We now describe the proposed policy gradient scheme summarized in Algorithm~\ref{alg:DL-RPG}, named \emph{Double-Loop Robust Policy Gradient}~(DRPG). We refer to DRPG as a ``double loop'' method in order to be consistent with the terminology in game theory literature~\citep{nouiehed2019solving,thekumparampil2019efficient,jin2020local,zhang2020single}. 

The inner loop of DRPG updates the worst-case transition probabilities $\bm{p}_t$ while the outer loop updates the policies $\bm{\pi}_t$. Specifically, DRPG iteratively takes steps along the policy gradient to search for an optimal policy in~\eqref{prob_RMDP}. At each iteration $t$, we first solve the inner maximization problem to some specific precision $\epsilon_t$; that is, for a policy $\bm{\pi}_{t}$ at iteration $t$, we seek for any transition kernel $\bm{p}_{t}$ such that
\begin{equation*}
J_{\bm{\rho}}(\bm{\pi}_{t},\bm{p}_{t}) \geq \max_{\bm{p}\in\mathcal{P}} J_{\bm{\rho}}(\bm{\pi}_{t},\bm{p}) - \epsilon_{t}~.
\end{equation*}
Once $\bm{p}_t$ is computed, DRPG then takes a \emph{projected gradient step} to minimize $J_{\bm{\rho}}(\bm{\pi}, \bm{p}_t)$ subject to a constraint $\pi\in \Pi$.

When chosen appropriately, the sequence $\epsilon_t$ allows for quick policy updates in the initial stages of the algorithm without putting the global convergence in jeopardy. Similar algorithms studied in the context of zero-sum games do not include this tolerance $\epsilon_t$~\citep{nouiehed2019solving,thekumparampil2019efficient}. The adaptive tolerance sequence $\{\epsilon_{t}\}_{t\geq0}$ is inspired by prior work on algorithms for RMDPs \citep{ho2021partial}. The convergence analysis below provides further guidance on appropriate choices of $\epsilon_t$.
%This setting aims to provide guidance on updating the chosen tolerance every iteration after the initial $\epsilon_{0}$ is chosen, and the value $\epsilon_{0}$ plays a role in the further convergence analysis. 

DRPG updates policies using projected gradient descent. The well-known proximal representation of projected gradient is~\citep{Bertsekas2016}:
\begin{align}
    \bm{\pi}_{t+1} &\in \arg\min_{\bm{\pi}\in\Pi} \; \langle \nabla_{\bm{\pi}}J_{\bm{\rho}}(\bm{\pi}_{t},\bm{p}_{t}), \bm{\pi}-\bm{\pi}_{t}\rangle + \frac{1}{2\alpha_{t}}\|\bm{\pi}-\bm{\pi}_{t}\|^{2}\notag\\
\label{eq:prox-update}
&= \text{Proj}_{\Pi}\left(\bm{\pi}_{t} - \alpha_{t}\nabla_{\bm{\pi}}J_{\bm{\rho}}(\bm{\pi}_{t},\bm{p}_{t})\right),
\end{align}
where $\text{Proj}_{\Pi}$ is the projection operator onto $\Pi$ and $\alpha_{t}>0$ is the step size. This projected gradient update on $\bm{\pi}_{t}:= (\bm{\pi}_{t,s})_{s\in\mathcal{S}} \in(\Delta^{A})^{S}$ can be further decoupled to multiple projection updates that across states and take the form as
\begin{equation*}
    \bm{\pi}_{t+1,s} = \text{Proj}_{\Delta^{A}}\left(\bm{\pi}_{t,s} - \alpha_{t}\nabla_{\bm{\pi}_{s}}J_{\bm{\rho}}(\bm{\pi}_{t},\bm{p}_{t})\right),\;\; \forall s\in\mathcal{S},
\end{equation*}
which can also be seen as a gradient step followed by a projection onto $\Delta^{A}$ for each state $s\in \mathcal{S}$. Note that the gradient $\nabla_{\bm{\pi}}J_{\bm{\rho}}(\bm{\pi}_{t},\bm{p}_{t})$ used in DRPG is identical to the the gradient in ordinary MDPs, e.g.,~\cite{agarwal2021theory,bhandari2021linear},
\begin{equation}\label{eq:sec3_lem3.1}
    \frac{\partial J_{\bm{\rho}}(\bm{\pi},\bm{p})}{\partial \pi_{sa}} = \frac{1}{1-\gamma} \cdot d_{\bm{\rho}}^{\bm{\pi},\bm{p}}(s)\cdot q^{\bm{\pi},\bm{p}}_{sa}.
\end{equation}
Actor-critic RL algorithms are typically based on this form of the policy gradient.

An alternative to double-loop algorithms is to use single-loop algorithms. Single-loop algorithms interleave gradient updates to the inner and outer optimization problems~\citep{mokhtari2020unified,zhang2020single}. Interleaving gradient updates is fast but prone to instabilities and oscillations. The most-common approach to preventing such instabilities is to resort to two-scale step size updates~\citep{heusel2017gans,daskalakis2020independent,russel2020robust}. We focus in this work on double-loop algorithms because of their conceptual simplicity and good empirical behavior.

\subsection{Convergence Analysis}
\label{sec:convergence-analysis}

We now turn to analyzing the convergence behavior of DRPG. First, recall that we assume that $\mathcal{P}$ is compact. Virtually all ambiguity sets considered in prior work, such as $L_{1}$-ambiguity sets, $L_{\infty}$-ambiguity sets, $L_{2}$-ambiguity sets, and KL-ambiguity sets, are compact.

Then, the following lemma helps us to derive the weak convexity of this non-convex, non-differentiable (i.e., non-smooth) objective function $\Phi(\bm{\pi})$.
\begin{lemma}\label{lem_sec3_1}
The objective function $J_{\bm{\rho}}(\bm{\pi},\bm{p})$ in~\eqref{prob_RMDP} is $L_{\bm{\pi}}$-Lipschitz and $\ell_{\bm{\pi}}$-smooth in $\bm{\pi}$ with
\begin{equation*}
    L_{\bm{\pi}}:=  \frac{\sqrt{A}}{(1-\gamma)^{2}}, \quad \ell_{\bm{\pi}}:= \frac{2 \gamma A}{(1-\gamma)^{3}}.
\end{equation*}
Furthermore, the objective $\Phi(\bm{\pi})$ is $\ell_{\bm{\pi}}$-weakly convex and $L_{\bm{\pi}}$-Lipschitz.
\end{lemma}
The proof of this lemma, as well as of all the remaining auxiliary results, are provided in the appendix. Lemma~\ref{lem_sec3_1} establishes some general continuity properties of $\Phi(\bm{\pi})$ and serves as an important stepping stone for deriving the global convergence of Algorithm~\ref{alg:DL-RPG}; however, weak convexity alone is insufficient to guarantee that gradient-based updates converge to a global optimum. 

Recent work~\cite{agarwal2021theory} proved the global convergence of policy gradient methods in ordinary MDP relying on a ``gradient dominance condition''. Informally speaking, a function $h(\bm{x})$ is said to satisfy the gradient dominance condition if $h(\bm{x})-h(\bm{x}^{\star}) = \mathcal{O}(G(\bm{x}))$, where $G(\cdot)$ is a suitable notion that measures the gradient of $h$. By having a gradient dominance condition, one can prevent the gradient from vanishing before reaching a globally optimal point.

Despite the non-smoothness of $\Phi(\bm{\pi})$, weakly convex problems naturally admit an implicit smooth approximation through the Moreau envelope~\cite{davis2019stochastic,mai2020convergence}. Inspired by the idea of gradient dominance, we introduce the gradient of the Moreau envelope and show that $\Phi(\bm{\pi})$ satisfies a particular variant of the gradient dominance condition in the next theorem.
\begin{theorem}\label{the_sec3_1_GD}
Denote $\bm{\pi}^{\star}$ as the global optimal policy for RMDPs. Then, for any policy $\bm{\pi}$, we have
\begin{equation}\label{eq:the_3_1_GD}
    \Phi(\bm{\pi}) - \Phi(\bm{\pi}^{\star}) \leq \left(\frac{D\sqrt{SA}}{1-\gamma}+\frac{L_{\bm{\pi}}}{2\ell_{\bm{\pi}}}\right)\|\nabla \Phi_{\frac{1}{2\ell_{\bm{\pi}}}}(\bm{\pi})\|,
\end{equation}
where $ \Phi_{\lambda}(\bm{\pi})$ is the Moreau envelope function of $\Phi(\bm{\pi})$ (see Definition~\ref{def:Moreau}) and $D:=\sup_{\bm{\pi}\in\Pi,\bm{p}\in\mathcal{P}}\left\|\nicefrac{\bm{d}_{\bm{\rho}}^{\bm{\pi},\bm{p}}}{\bm{\rho}}\right\|_{\infty} <\infty$ for every $\bm{\rho}$ with $\min_{s\in\mathcal{S}}\rho_{s}>0$.
\end{theorem}
Here, $\left\|\nicefrac{\bm{d}_{\bm{\rho}}^{\bm{\pi},\bm{p}}}{\bm{\rho}}\right\|_{\infty}$ is formally named as \emph{distribution mismatch coefficient} which is often assumed to be bounded~\citep{scherrer2014approximate,chen2019information,mei2020global,agarwal2021theory,leonardos2021global}.

This gradient-dominance type property implies that any first-order stationary point of the Moreau envelope results in an approximately global optimal
policy. We are now ready to state our main result.
\begin{theorem}[Global convergence for DRPG]\label{the_sec3_1}
Denote $\bm{\pi}_{t^{\star}}$ as the policy that Algorithm~\ref{alg:DL-RPG} outputs. Then, for a constant step size $\alpha:= \frac{\delta}{\sqrt{T}}$ with any $\delta>0$ and the initial tolerance $\epsilon_{0}\leq\sqrt{T}$, we have
\begin{equation}
      \Phi(\bm{\pi}_{t^{\star}}) - \min_{\bm{\pi}\in\Pi}\Phi(\bm{\pi}) \leq \epsilon,
\end{equation}
and $T$ is chosen to be a large enough such that
\begin{align}\label{chose_T}
        T &\geq \frac{\left(\frac{D\sqrt{SA}}{1-\gamma}+\frac{L_{\bm{\pi}}}{2\ell_{\bm{\pi}}}\right)^{4}\left(\frac{4\ell_{\bm{\pi}}S}{\delta} + 2\delta\ell_{\bm{\pi}} L_{\bm{\pi}}^{2} + \frac{4\ell_{\bm{\pi}}}{1-\gamma}\right)^{2}}{\epsilon^{4}}\notag\\
        &= \mathcal{O}(\epsilon^{-4}).
\end{align}
\end{theorem}
Compared to the ordinary MDPs, the convergence analysis for solving RMDPs poses additional difficulties as objective function $\Phi(\bm{\pi})$ is not only non-convex but also non-differentiable \citep{nouiehed2019solving,lin2020gradient}. Theorem~\ref{the_sec3_1} shows that the proposed Algorithm~\ref{alg:DL-RPG} converges to the global optimal for RMDPs by the following strategy. We first show the existence of an $\epsilon$-first order stationary point (see Definition~\ref{def:FOSP}) of $\Phi(\bm{\pi})$. More concretely, we prove the gradient of the Moreau envelope is smaller than $\epsilon$ on the output policy.
Then, by applying the derived gradient dominance condition (Theorem~\ref{the_sec3_1_GD}), we finally complete the proof as this stationary point is arbitrarily close to the global optimal solution. 

Theorem~\ref{the_sec3_1} shows that DRPG converges to an $\epsilon$ global optimum within $\mathcal{O}(\epsilon^{-4})$ steps, which has a slower rate compared to standard policy gradient methods~\cite{agarwal2021theory}. The additional complexity arises from this need to control the approximation error in order to avoid looping. In particular, computational errors at the inner loops could break the convergence of the outer loop. Similar behaviors are also observed in policy iteration for robust MDPs~\cite{condon1990algorithms,ho2021partial}. Nevertheless, our analysis matches and is consistent with the other minimax convergence results obtained in non-convex non-concave minimax optimization~\cite{davis2019stochastic,jin2020local}, and provides a conservative convergence guarantee. 

DRPG relies on an oracle that outputs at least one worst-case transition kernel for any given $\bm{\pi}$. In fact, solving the inner loop problem could still be NP-hard for non-rectangular cases~\cite{wiesemann2013robust}. The following section proposes an algorithm for solving the inner loop problem.

\section{Solving the Inner Loop}\label{sec:inner-loop-maxim}

So far, we have described the outline of DRPG and proved its global convergence. In Algorithm~\ref{alg:DL-RPG}, the transition kernel $\bm{p}_{t}$ is obtained by approximately solving the inner maximization problem with a fixed outer policy $\bm{\pi}_{k}\in\Pi$:

\begin{equation}\label{prob_inRMDP}
    \max_{\bm{p}\in\mathcal{P}}  J_{\bm{\rho}}(\bm{\pi}_{k},\bm{p}) = \max_{\bm{p}\in\mathcal{P}} \bm{\rho}^{\top}\bm{v}^{\bm{\pi}_{k},\bm{p}}.
\end{equation}
Whereas assumptions of boundness and compactness are used to ensure the inner maximum existing for the maximization problem, solving this maximization problem is still computationally challenging due to its non-convexity~\cite{wiesemann2013robust}. This section discusses two solution methods for solving the inner maximization problem, which we refer to as the \emph{robust policy evaluation problem}. Note that the convergence results in \cref{sec:main-results} are independent of the method used to solve this robust policy evaluation problem.

We now introduce two broad classes of ambiguity sets that are considered in the rest of this section. An ambiguity set $\mathcal{P}$ is $(s,a)$-rectangular~\citep{iyengar2005robust,nilim2005robust,le2007robust} if it is a Cartesian product of sets $\mathcal{P}_{s,a}\subseteq\Delta^{S}$ for each state $s\in\mathcal{S}$ and action $a\in\mathcal{A}$, \ie,
\begin{equation*}
\mathcal{P}=\{\bm{p}\in(\Delta^{S})^{S\times A}\mid\bm{p}_{s,a}\in\mathcal{P}_{s,a},\;\forall s\in\mathcal{S},a\in\mathcal{A}\},
\end{equation*}
whereas an ambiguity set $\mathcal{P}$ is $s$-rectangular~\citep{wiesemann2013robust} if it is defined as a Cartesian product of
sets $\mathcal{P}_{s}\subseteq(\Delta^{S})^{A}$, \ie,
\begin{equation*}
\mathcal{P}=\{\bm{p}\in(\Delta^{S})^{S\times A}\mid(\bm{p}_{s,a})_{a\in\mathcal{A}}\in\mathcal{P}_{s},\;\forall s\in\mathcal{S}\}.
\end{equation*}

\subsection{Value-iteration Approach}
\label{sec:inner_with_rect}

The optimum of the inner problem~\eqref{prob_inRMDP} is attained by solving $\bm{v}^{\bm{\pi}_{k}}:=\min_{\bm{p}\in\mathcal{P}}\bm{v}^{\bm{\pi}_{k},\bm{p}}$, which is commonly defined as the robust value function~\citep{iyengar2005robust,nilim2005robust,wiesemann2013robust}. The robust value function $\bm{v}^{\bm{\pi}}$ of a rectangular RMDP for a policy $\bm{\pi}\in\Pi$ can be computed using the robust Bellman policy update $\mathcal{T}_{\bm{\pi}}:\mathbb{R}^{S}\rightarrow\mathbb{R}^{S}$~\cite{ho2021partial}. Specifically, for $(s,a)$-rectangular RMDPs, the operator $\mathcal{T}_{\bm{\pi}}$ is
defined for each state $s\in\mathcal{S}$
\begin{equation*}
    (\mathcal{T}_{\bm{\pi}}\bm{v})_{s} := \sum_{a\in\mathcal{A}} \left(\pi_{sa}\cdot\max_{\bm{p}_{sa}\in\mathcal{P}_{sa}}\bm{p}_{sa}^{\top}(\bm{c}_{sa}+\gamma \bm{v})\right),
\end{equation*}
while for $s$-rectangular RMDPs, the the operator $\mathcal{T}_{\bm{\pi}}$ is defined as
\begin{equation*}
    (\mathcal{T}_{\bm{\pi}}\bm{v})_{s} := \max_{\bm{p}_{s}\in\mathcal{P}_{s}}\left\{\sum_{a\in\mathcal{A}}\pi_{sa}\cdot\bm{p}_{sa}^{\top}(\bm{c}_{sa}+\gamma \bm{v})\right\}.
\end{equation*}

For rectangular RMDPs, $\mathcal{T}_{\bm{\pi}}$ is a contraction and the robust value function is the unique solution to $\bm{v}^{\bm{\pi}} = \mathcal{T}_{\bm{\pi}}\bm{v}^{\bm{\pi}}$. To solve the robust value function, the state-of-the-art method is to compute the sequence $\bm{v}_{t+1}^{\bm{\pi}} = \mathcal{T}_{\bm{\pi}}\bm{v}_{t}^{\bm{\pi}}$ with any initial values $\bm{v}_{0}^{\bm{\pi}}$, which is similar to the policy evaluation for ordinary MDPs.

Note that computing the value function update $\bm{v}_{t}^{\bm{\pi}}$ to $\bm{v}_{t+1}^{\bm{\pi}}$ requires solving an optimization problem. For the common ambiguity sets which are constrained by the support information and one additional convex constraint (\eg\; $L_{1}$-norm ball), one has to solve $A$ convex optimization problems with $\mathcal{O}(S)$ variables and $\mathcal{O}(S)$ constraints for all $s\in\mathcal{S}$ at each iteration~\cite{grand2021scalable}. Examples of common ambiguity sets are provided in~\Cref{sec:examples_sets}.

\subsection{Gradient-based Approach}
Unlike the extensive study of efficient value-based methods~\citep{iyengar2005robust,nilim2005robust,wiesemann2013robust,petrik2014raam,ho2018fast,behzadian2021fast}, there has been little work on designing gradient-based algorithms to compute the robust value function. In this subsection, a first gradient-based algorithm is proposed in Algorithm~\ref{alg:PGD_inner} to solve the inner-loop robust policy evaluation problem with a global convergence guarantee, under the assumptions of having rectangular and convex ambiguity set.

\begin{algorithm}[t]
\caption{Projected gradient descent for the inner problem}
\label{alg:PGD_inner}
\begin{algorithmic}
\STATE {\bfseries Input:} Target fixed policy $\bm{\pi}_{k}$, initial transition kernel $\bm{p}_{0}$, iteration time $T_{k}$, step size sequence $\{\beta_{t}\}_{t\geq0}$
\FOR{$t = 0,1,\dots,T_{k}-1$}
\STATE Set $\bm{p}_{t+1} \gets \text{Proj}_{\mathcal{P}}(\bm{p}_{t} + \beta_{t}\nabla_{\bm{p}}J_{\bm{\rho}}(\bm{\pi}_{k},\bm{p}_{t}))$. 
\ENDFOR
\STATE {\bfseries Output:} 
$\bm{p}_{t^{\star}}\in\{\bm{p}_{0},\dots,\bm{p}_{T_{k}-1}\}$ s.t. 
$J_{\bm{\rho}}(\bm{\pi}_{k},\bm{p}_{t^{\star}}) = \min_{t \in \{0,\dots,T_{k}-1\}}J_{\bm{\rho}}(\bm{\pi}_{k},\bm{p}_{t})$
\end{algorithmic}
\end{algorithm}

Note that the inner problem \eqref{prob_inRMDP} could be regarded as a constrained non-concave maximization problem when the outer policy $\bm{\pi}_{k}$ is fixed. Therefore, the most intuitive approach to solve \eqref{prob_inRMDP} is to iteratively update the variable by following its ascent direction within the feasible set. 

To maximize $J_{\rho}(\bm{\pi}_{k},\bm{p})$, Algorithm~\ref{alg:PGD_inner} iteratively computes the \emph{projected gradient step} on $\bm{p}$; that is, at iteration $t$, we compute
\begin{equation}
    \bm{p}_{t+1} = \text{Proj}_{\mathcal{P}}(\bm{p}_{t} + \beta_{t}\nabla_{\bm{p}}J_{\bm{\rho}}(\bm{\pi}_{k},\bm{p}_{t})),
\end{equation}
which depends on the explicit form of $\mathcal{P}$. Although $(s,a)$-rectangular ambiguity sets can be viewed as a special case of $s$-rectangular ambiguity
sets in general~\citep{wiesemann2013robust,ho2021partial}, the implementations of the projected gradient step for two rectangular ambiguity sets are different.

For $(s,a)$-rectangular RMDPs, this projected gradient update can be decoupled to multiple projection updates that across state-action pairs such as
\begin{equation*}
    \bm{p}_{t+1,sa} = \text{Proj}_{\mathcal{P}_{s,a}}(\bm{p}_{t,sa} + \beta_{t}\nabla_{\bm{p}_{sa}}J_{\bm{\rho}}(\bm{\pi}_{k},\bm{p}_{t})).
\end{equation*}
Similarly, for $s$-rectangular RMDPs, the projected gradient update can be computed across states as 
\begin{equation*}
    \bm{p}_{t+1,s} = \text{Proj}_{\mathcal{P}_{s}}(\bm{p}_{t,s} + \beta_{t}\nabla_{\bm{p}_{s}}J_{\bm{\rho}}(\bm{\pi}_{k},\bm{p}_{t})).
\end{equation*}
If the ambiguity set is convex, the projected update can be implemented by solving a convex optimization problem with a quadratic objective.

\subsection{Inner Loop Global Optimality}
To establish some general convergence properties of Algorithm~\ref{alg:PGD_inner}, we first derive some continuity properties for the inner objective~\eqref{prob_inRMDP}. Then, we prove the global optimality of Algorithm~\ref{alg:PGD_inner} by introducing a particular gradient dominance condition for the inner problem.

The next lemma derives the gradient for the inner loop.

\begin{lemma}[Differentiability]\label{lem_sec4_1}
The partial derivative of $J_{\bm{\rho}}(\bm{\pi},\bm{p})$ has the explicit form for any $(s,a,s')\in\mathcal{S}\times\mathcal{A}\times\mathcal{S}$,  
\begin{equation*}
    \frac{\partial J_{\bm{\rho}}(\bm{\pi},\bm{p})}{\partial p_{sas'}} = \frac{1}{1-\gamma}d_{\bm{\rho}}^{\bm{\pi},\bm{p}}(s)\pi_{sa}\left(c_{sas'}+\gamma v^{\bm{\pi},\bm{p}}_{s'}\right).
\end{equation*} 
Moreover, $J_{\bm{\rho}}(\bm{\pi},\bm{p})$ is $L_{\bm{p}}$-Lipschitz in $\bm{p}$ with $L_{\bm{p}}:= \frac{\sqrt{SA}}{(1-\gamma)^{2}}$.
\end{lemma}
If a function is smooth, then a gradient update with a sufficiently small step size is guaranteed to improve the objective
value. As it turns out, inner problem is $\ell_{\bm{p}}$-smooth.
\begin{lemma}[Smoothness]\label{lem_sec4_2}
The function $J_{\bm{\rho}}(\bm{\pi},\bm{p})$ is $\ell_{\bm{p}}$-smooth in $\bm{p}$ with $\ell_{\bm{p}}:= \frac{2\gamma S^{2}}{(1-\gamma)^{3}}$.
\end{lemma}
Due to the non-convexity of $J_{\bm{\rho}}$, smoothness is not sufficient to establish the global convergence guarantee. We notice that the inner problem can be interpreted as having an adversarial nature to maximize the total reward (decision maker's cost) by selecting a proper transition kernel from the ambiguity set $\mathcal{P}$~\cite{lim2013reinforcement,goyal2022robust}. Hence we leverage the idea from the convergence analysis of the classical policy gradient~\cite{agarwal2021theory} and derive our global convergence guarantee by first deriving the following inner problem's gradient dominance condition.
\begin{lemma} [Gradient dominance]\label{the_sec4_GD}
For any fixed $\bm{\pi}\in\Pi$, $J_{\bm{\rho}}(\bm{\pi},\bm{p})$ satisfies the following condition for any $\bm{p}\in\mathcal{P}$ such that
\begin{equation*}
J_{\bm{\rho}}(\bm{\pi},\bm{p}^{\star})-J_{\bm{\rho}}(\bm{\pi},\bm{p}) 
\leq \frac{D}{1-\gamma}\max_{\bar{\bm{p}}\in\mathcal{P}} \left\langle\bar{\bm{p}}-\bm{p},\nabla_{\bm{p}} J_{\bm{\rho}}(\bm{\pi},\bm{p})\right\rangle,
\end{equation*}
where $J_{\bm{\rho}}(\bm{\pi},\bm{p}^{\star}):= \max_{\bm{p}\in\mathcal{P}}J_{\bm{\rho}}(\bm{\pi},\bm{p})$.
\end{lemma}
Using this notion of gradient dominance, we now give an iteration complexity bound for Algorithm~\ref{alg:PGD_inner}.
\begin{theorem}\label{the_sec4_global}
Let $\bm{p}_{t^{\star}}$ be the point obtained by Algorithm~\ref{alg:PGD_inner} and $\epsilon_{k}>0$ be the desired precision. Algorithm~\ref{alg:PGD_inner} with constant step size $\beta = \frac{(1-\gamma)^{3}}{2\gamma S^{2}}$ satisfies
\begin{equation}
\max_{\bm{p}\in\mathcal{P}}J_{\bm{\rho}}(\bm{\pi}_{k},\bm{p}) - J_{\bm{\rho}}(\bm{\pi}_{k},\bm{p}_{t^{\star}}) \leq \epsilon_{k},
\end{equation}
whenever
\begin{equation}
    T \geq \frac{32\gamma S^{3}AD^{2}}{(1-\gamma)^{6}\epsilon_{k}^{2}} = \mathcal{O}(\epsilon_{k}^{-2}).
\end{equation}
\end{theorem} 

\subsection{Scalability of Parametric Transition}
In standard policy-gradient methods, one considers a family of policies parametrized by lower-dimensional parameter vectors to limit the number of variables when scaling to large problems. The projected gradient step in Algorithm~\ref{alg:PGD_inner} needs to update each $p_{sas'}$, which is difficult with large state and action spaces. To overcome this problem, we provide a new approach to transition probability parameterization. To the best of our knowledge, comparable parameterizations for the inner problem have not been studied previously.

We parameterize transition kernel with the following form for any $(s,a,s')\in\mathcal{S}\times\mathcal{A}\times\mathcal{S}$,
\begin{equation}\label{sec4_parameterization}
p^{\bm{\xi}}_{sas'} := \frac{\bar{p}_{sas'}\cdot\exp(\frac{\bm{\theta}^{\top}\bm{\phi}(s')}{\lambda_{sa}})}{\sum_{k}\bar{p}_{sak}\cdot\exp({\frac{\bm{\theta}^{\top}\bm{\phi}(k)}{\lambda_{sa}}})},
\end{equation}
where $\bm{\phi}(s):=\left[\phi_{1}(s),\cdots,\phi_{m}(s)\right]$ is a $m$-dimensional feature vector corresponding to the state $s\in\mathcal{S}$, $\bm{\xi}:=(\bm{\theta},\bm{\lambda
})$ is the collection of parameters, consisting of the strictly positive parameter $\bm{\lambda}:=\{\lambda_{sa}>0 \mid  \forall (s,a)\in\mathcal{S}\times\mathcal{A}\}$ and the unconstrained parameter $\bm{\theta}:=\left[\theta_{1},\cdots,\theta_{m}\right]$. The symbol $\bar{\bm{p}}$ represents the nominal transition kernel, which is typically estimated from the empirical sample of state transitions. 

The parameterization in~\eqref{sec4_parameterization} is motivated by the form of the worst-case transition probabilities in RMDPs with KL-divergence constrained $(s,a)$-rectangular ambiguity sets~\cite{nilim2005robust}. In fact, the worst-case transitions has an identical form to~\eqref{sec4_parameterization} when linear approximation $\bm{\theta}^{\top}\bm{\phi}(s)$ is applied.

Then, the RMDPs problem then becomes,
\begin{equation*}
\min_{\bm{\pi}\in\Pi}\max_{\bm{\xi}\in\Xi} J_{\bm{\rho}}(\bm{\pi}, \bm{\xi}),
\end{equation*}
where $\Xi$ is the ambiguity set for the parameter $\bm{\xi}$. In practice, $\Xi$ could be constructed via distance-type constraint; that is, we consider
\begin{equation*}
    \Xi:=\{\bm{\xi} \mid D\left(\bm{\xi}\|\bm{\xi}_{c}\right) \leq \kappa\},
\end{equation*}
where $D(\cdot\|\cdot)$ represents a distance function, such as $L_{1}$-norm and $L_{\infty}$-norm, $\bm{\xi}_{c}$ is the user-specified empirical estimation of $\bm{\xi}$, and $\kappa\in\mathbb{R}_{++}$ is a given radius.

To apply the gradient-based update on parameterized transition, we introduce the following lemma to derive the gradient of the inner problem, which is similar to the classical policy gradient theorem~\cite{sutton1999policy}
\begin{lemma}\label{the_sec4_parame}
    Consider a map $\bm{\xi}\mapsto p^{\bm{\xi}}_{sas'}$ that is differentiable for any $(s,a,s')$ . Then, the partial gradient of $J_{\rho}(\bm{\pi},\bm{\xi})$ on $\bm{\xi}$ is
    \begin{align}
    \frac{\partial J_{\rho}(\bm{\pi},\bm{\xi})}{\partial \bm{\xi}} = \frac{1}{1-\gamma}    \mathbb{E}_{\begin{subarray}{l}
    s\sim\bm{d}^{\bm{\pi},\bm{\xi}}_{\rho}\\
    a\sim\bm{\pi}_{s\cdot}\\
    s'\sim\bm{p}_{sa\cdot}
    \end{subarray}}\left[\frac{\partial\log p^{\bm{\xi}}_{sas'}}{\partial \bm{\xi}}\left(c_{sas'}+\gamma v^{\bm{\pi},\bm{\xi}}_{s'}\right)\right].
\end{align}
Moreover, when parameterization~\eqref{sec4_parameterization} is applied, the score function $\frac{\partial\log p^{\bm{\xi}}_{sas'}}{\partial \bm{\xi}}$ has the analytical form:
\begin{align}
    \frac{\partial\log p^{\bm{\xi}}_{sas'}}{\partial \theta_{i}} &= \frac{\phi_{i}(s')}{\lambda_{sa}} - \sum_{j}p^{\bm{\xi}}_{saj}\cdot\frac{\phi_{i}(j)}{\lambda_{sa}},\\
    \frac{\partial\log p^{\bm{\xi}}_{sas'}}{\partial \lambda_{sa}} &= \sum_{j}p^{\bm{\xi}}_{saj}\cdot\frac{\bm{\theta}^{\top}\bm{\phi}(j)}{\lambda^{2}_{sa}}-\frac{\bm{\theta}^{\top}\bm{\phi}(s')}{\lambda_{sa}^{2}}.
\end{align}
\end{lemma}

\section{Experiments}\label{sec:experiments}
In this section, we demonstrate the global convergence of DRPG and verify the robustness of the policies computed by DRPG. All algorithms are implemented in Python 3.8.8, and performed on a computer with an i7-11700 CPU with 16GB RAM. We use Gurobi 9.5.2 to solve any linear or quadratic optimization problems involved. To facilitate the reproducibility of the domains, the full source code, which was used to generate
them, is available at \url{https://github.com/JerrisonWang/ICML-DRPG}. The repository also
contains CSV files with the precise specification of the RMDPs being solved.

\begin{figure}[t]
\centering
%\vspace{.3in}
\includegraphics[height=2in,width=0.4\textwidth]{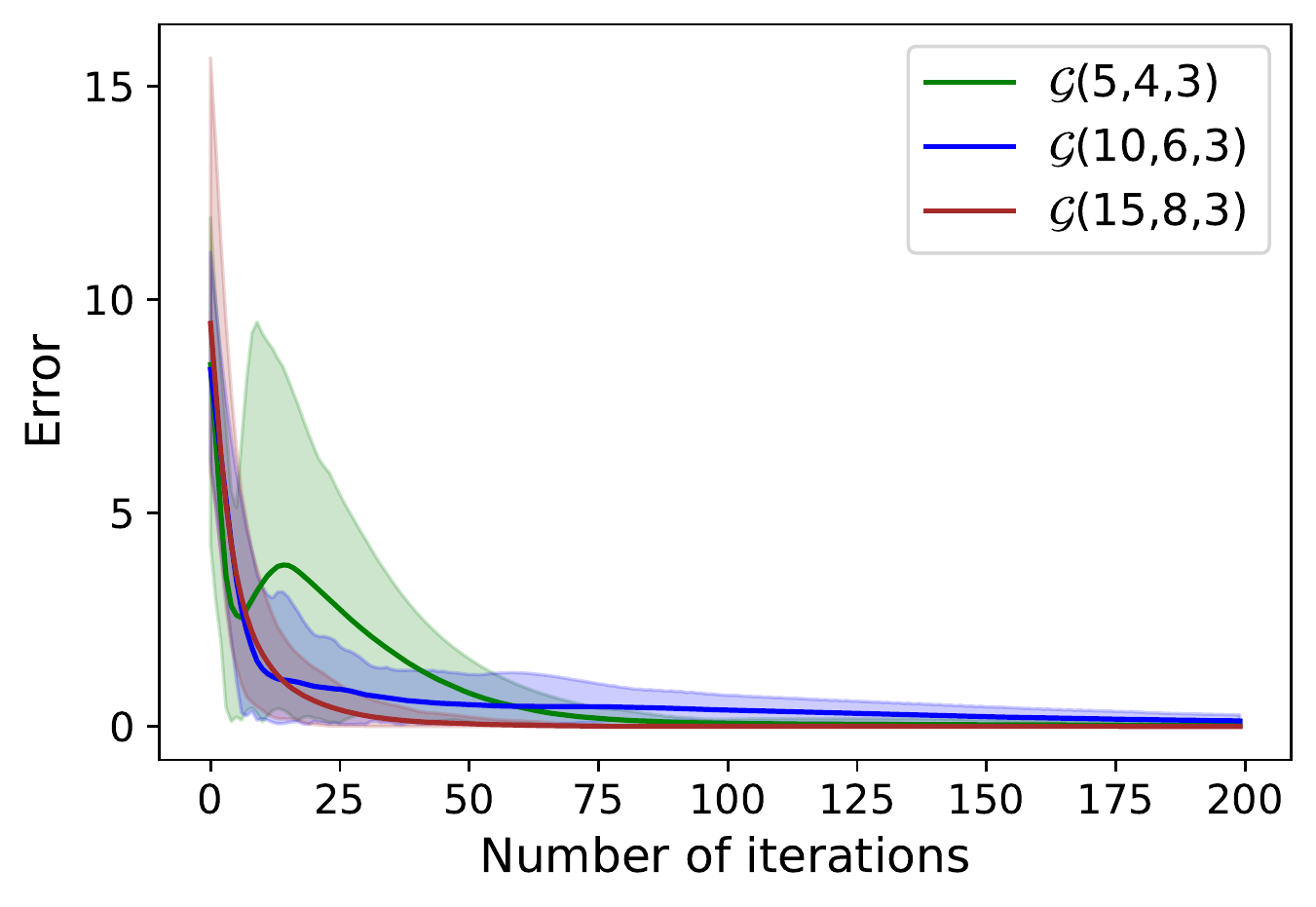}
%\vspace{.3in}
\caption{The error of value functions computed by DRPG for three Garnet problems with different sizes.}
\label{fig_ex1_error}
\end{figure}
\subsection{Experimental Setup}
To demonstrate the convergence behavior, we test our algorithm on random GARNET MDPs, one of the widely-used benchmarks for RL algorithms, with three different problem sizes and two settings on the ambiguity sets: $(s, a)$- and $s$-rectangular ambiguity
sets. We then apply DRPG with inner parameterization on the practical inventory management problem to demonstrate its convergence and robustness. %Next, we give a brief introduction on both problems.

Garnet MDPs are a class of abstract, but representative, finite MDPs that can be generated randomly~\citep{archibald1995generation}. A general GARNET $\mathcal{G}(|\mathcal{S}|,|\mathcal{A}|,b)$ is characterized by three parameters, where $|\mathcal{S}|$ is the number of states, $|\mathcal{A}|$ is the number of actions, and $b$ is a branching factor which determines the number of possible next states for each state-action pair and controls the level of connectivity of underlying Markov chains.

In our inventory management problem~\citep{porteus2002foundations,ho2018fast}, a retailer orders, stores, and sells a single product
over an infinite time horizon. The states and actions of the MDP represent the inventory levels and
the order quantities in any given time period, respectively. The stochastic demands drive
the stochastic state transitions. Any items held in inventory incur deterministic per-period holding costs. The retailer's goal is to find a policy that minimizes the total cost without knowing the exact transition kernel.

More details on the problem settings, parameter choice, and feature selection are available in the \cref{app: Details in experiment}.

\subsection{Results and Discussion}

In each of our GARNET problems, we compare the objective values of DRPG at different iterations with the optimal objective value $J^{\star}$, which is computed by robust value iteration. Robust value iteration solves the robust Bellman equation by iteratively applying robust Bellman updates. For each setup of our GARNET problems, we solve 50 sample instances using both DRPG and robust value iteration. Figure~\ref{fig_ex1_error} shows how the error (i.e., $|J(\bm{\pi}_{t},\bm{p}_{t}) - J^{\star}|$) decreases when DRPG is performed. The upper and lower envelopes of the curves correspond to the 95 and 5 percentiles of the 50 samples, respectively. As expected, the error decreases to zero as the iteration step increases, which confirms the convergence behavior of DRPG. Similar results are observed for the s-rectangular case. 

The results of our numerical study on the inventory management problem are provided in Figure~\ref{fig_ex1_sa_comparison}. We run DRPG with inner parametrization and compare the performance with the non-robust policy gradient. At each iteration $t$, we consider the policy $\bm{\pi}_{t}$ obtained by DRPG, and then we compute its worst-case expected return $\Phi(\bm{\pi}_{t})=\max_{\bm{p}\in\mathcal{P}}J(\bm{\pi}_{t},\bm{p})$. We do the same for the non-robust policy gradient method. As we can see, DRPG obtains a policy that performs much better than the non-robust policy gradient, which demonstrates the robustness of our method. Different step sizes are chosen for DRPG, and they lead to different convergence behaviors; yet, in both cases, DRPGs outperform the non-robust policy gradient method.

\begin{figure}[t]
\centering
%\vspace{.3in}
\includegraphics[height=2in,width=0.4\textwidth]{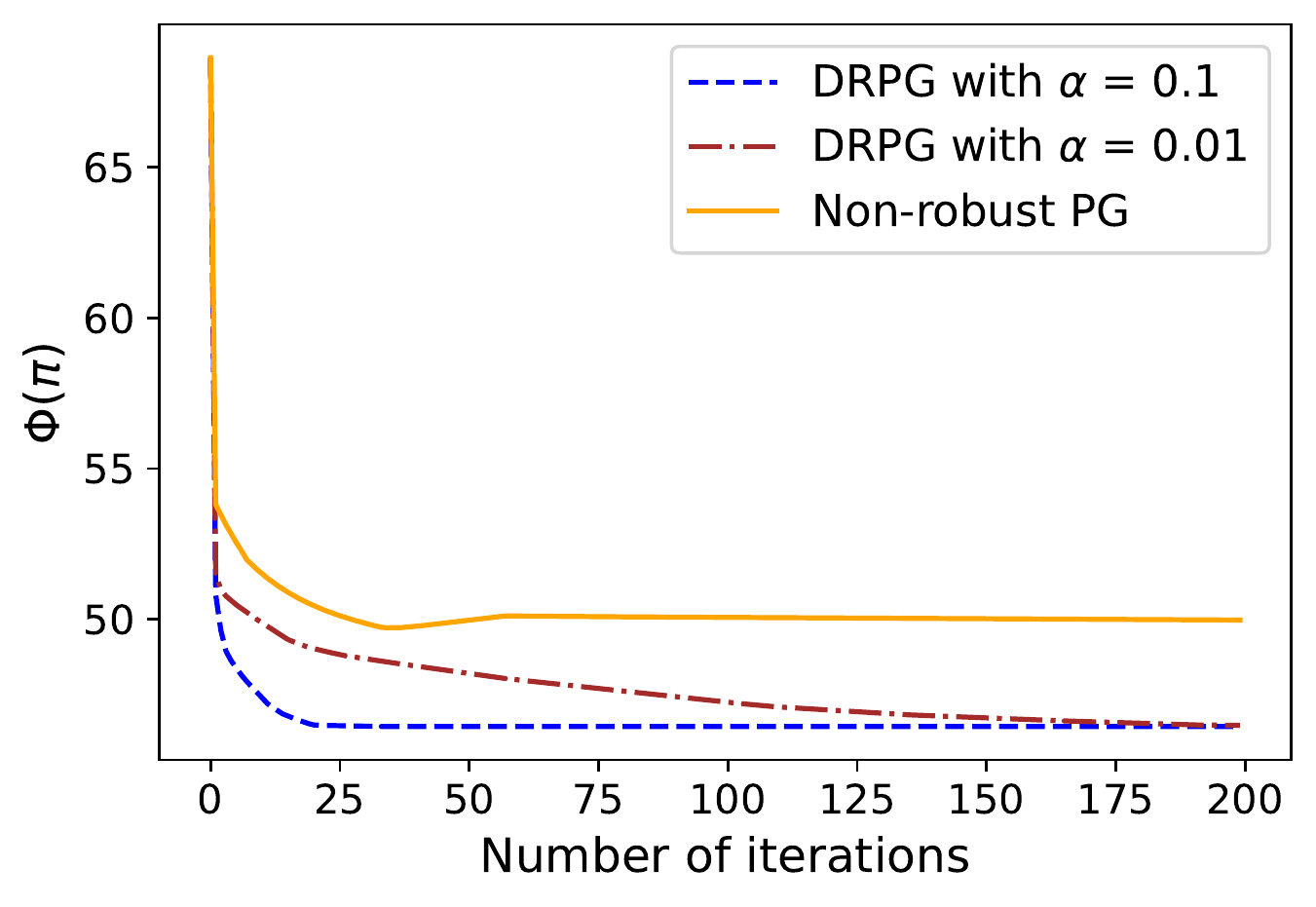}
%\vspace{.3in}
\caption{DRPG with parameterization v.s. Non-robust Policy Gradient
on the Inventory Management Problem}
\label{fig_ex1_sa_comparison}
\end{figure}

\section{Conclusion}
We proposed a new policy optimization algorithm DRPG to solve RMDPs over general compact ambiguity sets. By selecting a suitable step size and an adaptive decreasing tolerance sequence, our algorithm converges to the global optimal policy under mild conditions. Moreover, we provide the first gradient-based solution method with a novel parameterization for solving the inner maximization. In our experiments, our results demonstrate the global convergence of DRPG and its reliable performance against the non-robust approach. Future work should address extensions to related models (\eg, distributionally RMDP) and scalable model-free algorithms.

\section*{Acknowledgements}

We thank the anonymous reviewers for their comments. This work was supported, in part, by NSF grants 2144601 and 1815275, the CityU Start-Up Grant (Project No. 9610481), the National Natural Science Foundation of China (Project No. 72032005), and Chow Sang Sang Group Research Fund sponsored by Chow Sang Sang Holdings International Limited (Project No. 9229076).

\bibliography{main}
\bibliographystyle{icml2023}

%%%%%%%%%%%%%%%%%%%%%%%%%%%%%%%%%%%%%%%%%%%%%%%%%%%%%%%%%%%%%%%%%%%%%%%%%%%%%%%
%%%%%%%%%%%%%%%%%%%%%%%%%%%%%%%%%%%%%%%%%%%%%%%%%%%%%%%%%%%%%%%%%%%%%%%%%%%%%%%
% APPENDIX
%%%%%%%%%%%%%%%%%%%%%%%%%%%%%%%%%%%%%%%%%%%%%%%%%%%%%%%%%%%%%%%%%%%%%%%%%%%%%%%
%%%%%%%%%%%%%%%%%%%%%%%%%%%%%%%%%%%%%%%%%%%%%%%%%%%%%%%%%%%%%%%%%%%%%%%%%%%%%%%
\newpage
\appendix
\onecolumn

\section{Examples of Common Ambiguity Sets}\label{sec:examples_sets}

We discuss a particularly popular class of rectangular ambiguity sets which are defined by norm constraints bounding the distance of any feasible transition probabilities from a nominal (average) state distribution. It is usually referred to as an $L_{1}$-constrained ambiguity set~\citep{petrik2014raam,ghavamzadeh2016safe,ho2021partial} or $L_{\infty}$-constrained ambiguity set~\citep{delgado2016real,behzadian2021fast}. For such rectangular ambiguity sets, problem~\eqref{prob_inRMDP} can be solved efficiently by updating the value function with the robust Bellman operator $\mathcal{T}_{\bm{\pi}}:\mathbb{R}^{S}\rightarrow\mathbb{R}^{S}$. Below, we show forms of Bellman operator within different rectangular conditions.

\begin{example}
$L_{1}$-constrained $(s,a)$-rectangular ambiguity sets generally assume the uncertain in transition probabilities is independent for each state-action pair and are defined as 
\begin{equation*}
\mathcal{P}=\underset{s \in \mathcal{S},a\in\mathcal{A}}{\times} \mathcal{P}_{s,a} \quad \text { where } \quad \mathcal{P}_{s,a}:=\left\{\bm{p}\in\Delta^{S} \mid \|\bm{p}-\bar{\bm{p}}_{sa}\|_{1} \leq \kappa_{sa}\right\}.
\end{equation*}
For $(s,a)$-rectangular RMDPs constrained by the $L_{1}$-norm, $\mathcal{T}_{\bm{\pi}}$ is defined for each $s\in\mathcal{S}$ as
\begin{equation*}
    (\mathcal{T}_{\bm{\pi}}\bm{v}^{\bm{\pi},\bm{p}})_{s} := \sum_{a\in\mathcal{A}} \left(\pi_{sa}\cdot\max_{\bm{p}_{sa}\in\mathcal{P}_{sa}}\left\{\bm{p}_{sa}^{\top}(\bm{c}_{sa}+\gamma \bm{v}^{\bm{\pi},\bm{p}})\mid\|\bm{p}_{sa}-\bar{\bm{p}}_{sa}\|_{1} \leq \kappa_{sa}\right\}\right).
\end{equation*}
\end{example}

\begin{example}
$L_{\infty}$-constrained s-rectangular ambiguity sets generally assume the uncertain in transition probabilities is independent for each state-action pair and are defined as 
\begin{equation*}
\mathcal{P}=\underset{s \in \mathcal{S}}{\times} \mathcal{P}_{s} \quad \text { where } \quad \mathcal{P}_{s}:=\left\{\left(\bm{p}_{s 1}, \ldots, \bm{p}_{s A}\right)\in(\Delta^{S})^{A} \mid \sum_{a\in\mathcal{A}}\|\bm{p}_{sa}-\bar{\bm{p}}_{sa}\|_{\infty} \leq \kappa_{s}\right\}.
\end{equation*}
For $s$-rectangular RMDPs constrained by the $L_{\infty}$-norm, $\mathcal{T}_{\bm{\pi}}$ is defined for each $s\in\mathcal{S}$ as
\begin{align*}
    (\mathcal{T}_{\bm{\pi}}\bm{v}^{\bm{\pi},\bm{p}})_{s} := \max_{\bm{p}_{s}\in\mathcal{P}_{s}}\left\{\sum_{a\in\mathcal{A}}\pi_{sa}\cdot\bm{p}_{sa}^{\top}(\bm{c}_{sa}+\gamma \bm{v}^{\bm{\pi},\bm{p}})\mid\sum_{a\in\mathcal{A}}\|\bm{p}_{sa}-\bar{\bm{p}}_{sa}\|_{\infty} \leq \kappa_{s}\right\}.
\end{align*}
\end{example}

There exists an unique solution to the Bellman equation $\bm{v}^{\bm{\pi},\bm{p}} = \mathcal{T}_{\bm{\pi}}\bm{v}^{\bm{\pi},\bm{p}}$, which is called the robust value function~\citep{iyengar2005robust,wiesemann2013robust}. Specially, both $L_{1}$-constrained ambiguity sets and $L_{\infty}$-constrained ambiguity sets are in fact polyhedral, which implies the worst-case transition probabilities in bellman updates can be computed as the solution of linear programs (LPs). Instead, RMDPs with other distance-type ambiguity sets, such as $L_{2}$-constrained ambiguity sets can compute an Bellman update $\mathcal{T}_{\bm{\pi}}$ by solving convex optimization problems.

\section{Technical Lemmas and Definitions}
\label{sec:techn-lemm-defin}
As promised, we first introduce the definition of the Fréchet sub-differential for general functions.
\begin{definition}\label{def:F-subdiff}
The Fréchet sub-differential of a function $h:\mathcal{X}\rightarrow\mathbb{R}$ at point $\bm{x}\in\mathcal{X}$ is defined as the set
\(\partial h(\bm{x}) = \{u|\liminf_{\bm{x}'\rightarrow \bm{x}} h(\bm{x}')-h(\bm{x})-\langle \bm{u},\bm{x}'-\bm{x}\rangle/\|\bm{x}'-\bm{x}\|\geq0\}.
\)
\end{definition}
Then, a common lemma is provided to illustrate a basic property that a smooth function satisfies.
\begin{lemma}\label{lem_sec2_1}
Let $h\colon \mathcal{X} \to \Real$ be $\ell$-smooth, then it is a $\ell$-weakly convex function.
\end{lemma}
\begin{proof}[Proof of \cref{lem_sec2_1}]
Let $r(t):=h(x+t(x'-x))$, for any $x,x'\in\mathcal{X}$. The following holds true
\begin{equation*}
    h(x)=r(0)\;\text{and}\;h(x')=r(1).
\end{equation*}
Then, we observe that
\begin{equation*}
    h(x')-h(x) = r(1)-r(0) = \int_{0}^{1}\nabla r(t)dt,
\end{equation*}
where
\begin{equation*}
    \nabla r(t) = \nabla h(x+t(x'-x))^{\top}(x'-x).
\end{equation*}
We complete the proof as
\begin{align*}
    \|h(x')-h(x)-&\nabla h(x)^{\top}(x'-x)\|\\
    &\leq \left\|\int_{0}^{1}\nabla r(t)dt-\nabla h(x)^{\top}(x'-x)\right\|\\
    &\leq \int_{0}^{1}\left\|\nabla r(t)-\nabla h(x)^{\top}(x'-x)\right\|dt\\
    &= \int_{0}^{1}\left\|\nabla h(x+t(x'-x))^{\top}(x'-x)-\nabla h(x)^{\top}(x'-x)\right\|dt\\
    &\leq\int_{0}^{1}\left\|\nabla h(x+t(x'-x))-\nabla h(x)\right\|\cdot\|(x'-x)\|dt\\
    &\leq \int_{0}^{1}t\ell \|x'-x\|^{2}dt = \frac{\ell}{2}\|x'-x\|^{2}.
\end{align*}
\end{proof}

For smooth function $h(x)$, a point $x\in\mathcal{X}$ is defined as the first-order stationary point (FOSP) when $0\in\partial h(x)$. However, this notion of stationarity can be very restrictive when optimizing nonsmooth functions~\citep{lin2020gradient}. In respond
to this issue, an alternative measure of the first-order stationarity is proposed based on the construction of the Moreau envelope~\citep{thekumparampil2019efficient}.
\begin{definition}\label{def:Moreau}
For function $h:\mathcal{X}\rightarrow\mathbb{R}$ and $\lambda>0$, the Moreau envelope function of $h$ is given by 
\begin{equation}
    h_{\lambda}(x):=\min _{x'\in\mathcal{X}} \, \left\{h(x')+\frac{1}{2\lambda}\left\|x-x'\right\|^2\right\}.
\end{equation}
\end{definition}

\begin{definition}\label{def:FOSP}
Given an $\ell$-weakly convex function $h$, we say that $x^{\star}$ is an $\epsilon$-first order stationary point ($\epsilon$-FOSP) if, $\|\nabla h_{\frac{1}{2\ell}}(x^\star)\|\leq \epsilon$, where $h_{\frac{1}{2\ell}}(x)$ is the Moreau envelope function of $h$ with parameter $\lambda=\frac{1}{2\ell}$.
% h_{1/2\ell}
\end{definition}

The following lemma connects $\ell$-weakly convex function and its Moreau envelope function and will be useful in our proofs.
\begin{lemma}\citep[Proposition 13.37]{rockafellar2009variational}\label{lem_sec2_4}
Assume $h\colon \mathcal{X}\rightarrow\mathbb{R}$ is a $\ell$-weakly convex function. Then, for $\lambda < \ell$, the Moreau envelope function $h_{\lambda}$ is $C^{1}$-smooth with the gradient given by,
\begin{equation*}
    \nabla h_{\lambda}(x)=\lambda^{-1}\left(x-\mathop{\arg\min}\limits_{x'} \left(h(x')+\frac{1}{2 \lambda}\left\|x-x'\right\|^{2} \right)\right).
\end{equation*}
\end{lemma}

\begin{lemma}\label{lem_sec2_3}
Assume the function  $h:\mathcal{X}\subseteq\mathbb{R}^{n}\rightarrow\mathbb{R}$ is $\ell$-weakly convex and not differentiable at any point. Let $\lambda < \frac{1}{\ell}$ and $\hat{\bm{x}}_{\lambda} = \arg\min_{\bm{\bm{x}}'\in\mathcal{X}} h(\bm{\bm{x}}^{\prime})+\frac{1}{2 \lambda}\left\|\bm{x}-\bm{x}'\right\|^{2}$. Then we have
\begin{equation*}
    \frac{1}{\lambda}\|\hat{\bm{x}}_{\lambda} - \bm{x}\|=\|\nabla h_{\lambda}(\bm{x})\|.
\end{equation*}
As a result, $\|\nabla h_{\lambda}(\bm{x})\|\leq \epsilon$ implies$\|\hat{x}_{\lambda} - \bm{x}\|\leq \lambda\epsilon$ and $\exists \bm{\xi}\in\partial h(\hat{\bm{x}}_{\lambda})$ such that
\begin{equation*}
-\bm{\xi}\in \mathcal{N}_{\mathcal{X}}(\hat{\bm{x}}_{\lambda}) + \frac{1}{ \lambda}\left(\hat{\bm{x}}_{\lambda}-\bm{x}\right) \subseteq \mathcal{N}_{\mathcal{X}}(\hat{\bm{x}}_{\lambda}) + \frac{1}{ \lambda}\left\|\hat{\bm{x}}_{\lambda}-\bm{x}\right\|\mathcal{B}(1),
\end{equation*}
where $\mathcal{N}_{\mathcal{X}}(\hat{\bm{x}}_{\lambda})$ denotes the normal cone of $\mathcal{X}$ at $\hat{\bm{x}}_{\lambda}$ and $\mathcal{B}(r):= \{\bm{x}\in\mathbb{R}^{n}: \|\bm{x}\|\leq r\}$.
\end{lemma}
\begin{proof}
Here, we consider the function $f(\bm{x}) = h(\bm{x}) + \mathbb{I}_{\mathcal{X}}(\bm{x})$ where $\mathbb{I}$ is the indicate function and here $f(\bm{x}):= \mathbb{R}^{n}\rightarrow\mathbb{R}$. The Moreau envelope function of $f(\bm{x})$ is defined as
\begin{align*}
    f_{\lambda}(\bm{x}) &= \min _{\bm{x}'\in\mathbb{R}^{n}} \, \left\{h(\bm{x}')+\mathbb{I}_{\mathcal{X}}(\bm{x}')+\frac{1}{2\lambda}\left\|\bm{x}-\bm{x}'\right\|^2\right\}, \quad\forall \bm{x}\in\mathbb{R}^{n}\\
    &= \min _{\bm{x}'\in\mathcal{X}} \, \left\{h(\bm{x}')+\frac{1}{2\lambda}\left\|\bm{x}-\bm{x}'\right\|^2\right\}, \quad\forall \bm{x}\in\mathbb{R}^{n}.
\end{align*}
The gradient of the moreau envelope $f_{\lambda}(\bm{x})$ is well defined (Lemma \ref{lem_sec2_4}) as
\begin{align*}
    \nabla f_{\lambda}(\bm{x})=\lambda^{-1}\left(\bm{x}-\hat{\bm{x}} \right),
\end{align*}
where
\begin{align*}
    \hat{\bm{x}} :&= \mathop{\arg\min}\limits_{\bar{\bm{x}}\in\mathbb{R}^{n}} \left(\underbrace{h(\bar{\bm{x}})+\mathbb{I}_{\mathcal{X}}(\bar{\bm{x}})+\frac{1}{2 \lambda}\left\|\bm{x}-\bar{\bm{x}}\right\|^{2}}_{:=\phi_{\bm{x}}(\bar{\bm{x}})} \right)\\
    &= \mathop{\arg\min}\limits_{\bar{\bm{x}}\in\mathcal{X}} \left(h(\bar{\bm{x}})+\frac{1}{2 \lambda}\left\|\bm{x}-\bar{\bm{x}}\right\|^{2} \right)
\end{align*}
Then, we consider the optimality of the function $\phi_{\bm{x}}(\bm{y}) = h(\bm{y})+\mathbb{I}_{\mathcal{X}}(\bm{y})+\frac{1}{2 \lambda}\left\|\bm{x}-\bm{y}\right\|^{2}$. Notice that, for any $\bm{x}\in\mathbb{R}^{n}$, $\hat{\bm{x}}$ is the optimal solution of $\phi_{\bm{x}}(\bm{y})$, then for some $\bm{\xi}\in\partial h(\hat{\bm{x}})$, we have
\begin{align}\label{eq:lem_station}
    \phi_{\bm{x}}(\hat{\bm{x}}(\bm{x})) = \min_{\bm{y}\in\mathbb{R}^{n}} \phi_{\bm{x}}(\bm{y}) &\Longleftrightarrow \phi_{\bm{x}}(\hat{\bm{x}}(\bm{x})) = \min_{\bm{y}\in\mathbb{R}^{n}} h(\bm{y})+\mathbb{I}_{\mathcal{X}}(\bm{y})+\frac{1}{2 \lambda}\left\|\bm{x}-\bm{y}\right\|^{2}\notag\\
    &\Longleftrightarrow 0 \in \partial\left(h(\bm{y})+\mathbb{I}_{\mathcal{X}}(\bm{y})+\frac{1}{2 \lambda}\left\|\bm{x}-\bm{y}\right\|^{2}\right)\Big\vert_{\bm{y}= \hat{\bm{x}}},\notag\\
    &\Longleftrightarrow 0 \in \bm{\xi}+\mathcal{N}_{\mathcal{X}}(\hat{\bm{x}})+\frac{1}{ \lambda}\left(\hat{\bm{x}}-\bm{x}\right)\notag\\
    &\Longleftrightarrow - \bm{\xi}\in\mathcal{N}_{\mathcal{X}}(\hat{\bm{x}})+\frac{1}{ \lambda}\left(\hat{\bm{x}}-\bm{x}\right).
\end{align}
The above equation (\ref{eq:lem_station}) implies that, for any $\bm{z}\in\mathbb{R}^{n}$,
\begin{align}
    \langle\bm{\xi}+\frac{1}{ \lambda}\left(\hat{\bm{x}}-\bm{x}\right), \bm{z}-\hat{\bm{x}}\rangle \geq 0
    &\Longleftrightarrow \langle-\bm{\xi}, \bm{z}-\hat{\bm{x}}\rangle \leq \langle\frac{1}{ \lambda}\left(\hat{\bm{x}}-\bm{x}\right), \bm{z}-\hat{\bm{x}}\rangle,\;\forall \bm{z}\in\mathbb{R}^{n}\notag\\
    &\Longleftrightarrow \langle-\bm{\xi}, \bm{z}-\hat{\bm{x}}\rangle \leq \frac{1}{ \lambda}\|\hat{\bm{x}}-\bm{x}\|\cdot\|\bm{z}-\hat{\bm{x}}\|,\;\forall \bm{z}\in\mathbb{R}^{n}\notag\\
    &\Longleftrightarrow \langle-\bm{\xi}, \bm{z}-\hat{\bm{x}}\rangle \leq \frac{1}{ \lambda}\|\hat{\bm{x}}-\bm{x}\|,\;\forall \bm{z}\in\mathbb{R}^{n},\;\|\bm{z}-\hat{\bm{x}}\| = 1.
\end{align}

\end{proof}

The above Lemma~\ref{lem_sec2_3} implies that if $\|\nabla h_{\lambda}(\bm{x})\|$ is small enough, then $\bm{x}$ is an approximate stationary point of the original constrained optimization $\min_{\mathcal{X}} h(x)$, by the definition of $\epsilon$-FOSP. This motivates us to consider the optimality of the Moreau envelope function of $\Phi(\bm{\pi})$ instead of the optimality of $\Phi(\bm{\pi})$ directly.

\section{Proofs of \cref{sec:main-results}}
\begin{proof}[Proof of \cref{lem_sec3_1}]
\label{app: Proof of sec3}
First, we first derive the form of partial derivative for $\pi_{sa}$ to obtain (\ref{eq:sec3_lem3.1}). While this form was known (Agarwal et al., 2019), we included a proof for the sake of completeness. Notice that,
\begin{equation*}
    \frac{\partial J_{\bm{\rho}}(\bm{\pi},\bm{p})}{\partial \pi_{sa}} = \sum_{\hat{s} \in \mathcal{S}} \frac{\partial v^{\bm{\pi},\bm{p}}_{\hat{s}}}{\partial \pi_{sa}} \rho_{\hat{s}}.
\end{equation*}
Then, we discuss $\frac{\partial v^{\bm{\pi},\bm{p}}_{\hat{s}}}{\partial \pi_{sa}}$ over two cases: $\hat{s}\neq s$ and $\hat{s}= s$
\begin{align*}
    \frac{\partial v^{\bm{\pi},\bm{p}}_{\hat{s}}}{\partial \pi_{sa}}\Big\vert_{\hat{s}\neq s} &= \frac{\partial}{\partial \pi_{sa}}\left[\sum_{\hat{a}}\pi_{\hat{s}\hat{a}}\sum_{s' \in \mathcal{S}} p_{\hat{s}\hat{a}s'}\left(c_{\hat{s}\hat{a}s'}+\gamma v^{\bm{\pi},\bm{p}}_{s'}\right)\right]
    = \gamma\sum_{\hat{a}}\pi_{\hat{s}\hat{a}}\sum_{s' \in \mathcal{S}} p_{\hat{s}\hat{a}s'} \frac{\partial v^{\bm{\pi},\bm{p}}_{s'}}{\partial \pi_{sa}}; \\
    \frac{\partial v^{\bm{\pi},\bm{p}}_{\hat{s}}}{\partial \pi_{sa}}\Big\vert_{\hat{s}= s} &= \frac{\partial}{\partial \pi_{sa}}\left[\sum_{\hat{a}}\pi_{s\hat{a}}\underbrace{\sum_{s' \in \mathcal{S}} p_{s\hat{a}s'}\left(c_{s\hat{a}s'}+\gamma v^{\bm{\pi},\bm{p}}_{s'}\right)}_{q^{\bm{\pi},\bm{p}}_{s\hat{a}}}\right]
    = q^{\bm{\pi},\bm{p}}_{sa}+\gamma\sum_{\hat{a}}\pi_{s\hat{a}}\sum_{s' \in \mathcal{S}} p_{s\hat{a}s'} \frac{\partial v^{\bm{\pi},\bm{p}}_{s'}}{\partial \pi_{sa}}; 
\end{align*}
Condense the notation
\begin{align}\label{inter_no_1}
    \sum_{\hat{a}}\pi_{s\hat{a}} p_{s\hat{a}s'} &= p^{\bm{\pi}}_{ss'}(1)\\
    p^{\bm{\pi}}_{ss'}(t-1)\cdot\sum_{a}&\pi_{s'a} p_{s'as''} = p^{\bm{\pi}}_{ss''}(t)
\end{align}
Then, combining these two equations, we can obtain,
\begin{align*}
    \frac{\partial v^{\bm{\pi},\bm{p}}_{\hat{s}}}{\partial \pi_{sa}}\Big\vert_{\hat{s}\neq s} 
    &= \gamma\sum_{s'\neq s}p^{\bm{\pi}}_{\hat{s}s'}(1)\frac{\partial v^{\bm{\pi},\bm{p}}_{s'}}{\partial \pi_{sa}}+ \gamma\sum_{s'= s}p^{\bm{\pi}}_{\hat{s}s'}(1)\frac{\partial v^{\bm{\pi},\bm{p}}_{s'}}{\partial \pi_{sa}}\\
    &=\gamma^{2}\sum_{s'\neq s}p^{\bm{\pi}}_{\hat{s}s'}(1)\sum_{\hat{a}}\pi_{s'\hat{a}}\sum_{s'' \in \mathcal{S}} p_{s'\hat{a}s''} \frac{\partial v^{\bm{\pi},\bm{p}}_{s''}}{\partial \pi_{sa}}\\
    &+ \gamma p^{\bm{\pi}}_{\hat{s}s}(1)\left(q^{\bm{\pi},\bm{p}}_{sa}+\gamma\sum_{\hat{a}}\pi_{s\hat{a}}\sum_{s' \in \mathcal{S}} p_{s\hat{a}s'} \frac{\partial v^{\bm{\pi},\bm{p}}_{s'}}{\partial \pi_{sa}}\right)
    \\
    &= \gamma p^{\bm{\pi}}_{\hat{s}s}(1)q^{\bm{\pi},\bm{p}}_{sa}+ \gamma^{2}\sum_{s'}p^{\bm{\pi}}_{\hat{s}s'}(2)\frac{\partial v^{\bm{\pi},\bm{p}}_{s'}}{\partial \pi_{sa}}\\
    &= \gamma p^{\bm{\pi}}_{\hat{s}s}(1)q^{\bm{\pi},\bm{p}}_{sa}+ \gamma^{2}p^{\bm{\pi}}_{\hat{s}s}(2)q^{\bm{\pi},\bm{p}}_{sa}+\gamma^{3}\sum_{s'}p^{\bm{\pi}}_{\hat{s}s'}(3)\frac{\partial v^{\bm{\pi},\bm{p}}_{s'}}{\partial \pi_{sa}}\\
    &= \cdots\\
    &= \sum_{t=1}^{\infty}\gamma^{t}p^{\bm{\pi}}_{\hat{s}s}(t)q^{\bm{\pi},\bm{p}}_{sa} = \sum_{t=0}^{\infty}\gamma^{t}p^{\bm{\pi}}_{\hat{s}s}(t)q^{\bm{\pi},\bm{p}}_{sa}.
\end{align*}
The last equality is from the initial assumption $\hat{s} \neq s$, \ie, $p^{\bm{\pi}}_{\hat{s}s}(0)=0$, and similarly for the case $\hat{s}= s$ we have,
\begin{equation*}
    \frac{\partial v^{\bm{\pi},\bm{p}}_{\hat{s}}}{\partial \pi_{sa}}\Big\vert_{\hat{s}= s} = \sum_{t=0}^{\infty}\gamma^{t}p^{\bm{\pi}}_{ss}(t)q^{\bm{\pi},\bm{p}}_{sa}.
\end{equation*}
Hence, the partial derivative is obtained
\begin{align*}
    \frac{\partial J_{\bm{\rho}}(\bm{\pi},\bm{p})}{\partial \pi_{sa}} &= \left(\frac{\partial v^{\bm{\pi},\bm{p}}_{s}}{\partial \pi_{sa}} \rho_{s}+\sum_{\hat{s} \neq s} \frac{\partial v^{\bm{\pi},\bm{p}}_{\hat{s}}}{\partial \pi_{sa}} \rho_{\hat{s}}\right)
    = \frac{1}{1-\gamma}\left(\underbrace{(1-\gamma) \sum_{\hat{s}\in\mathcal{S}}\sum_{t=0}^{\infty} \gamma^{t} \rho_{\hat{s}}p^{\bm{\pi}}_{\hat{s}s}(t)}_{d_{\bm{\rho}}^{\bm{\pi},\bm{p}}(s)}\right)q^{\bm{\pi},\bm{p}}_{sa}.
\end{align*}
After deriving the form of partial derivative, we next prove that $J_{\bm{\rho}}(\bm{\pi},\bm{p})$ is $L_{\bm{\pi}}$-Lipschitz in $\bm{\pi}$ by showing the boundedness of $\nabla_{\bm{\pi}}J_{\bm{\rho}}(\bm{\pi},\bm{p})$. The uniformly bounded cost $c_{sas'}$ implies that, the absolute value of the action value function is bounded for any policy $\bm{\pi}$ and transition kernel $\bm{p}$,
\begin{align*}
    \left| q^{\bm{\pi},\bm{p}}_{sa}\right| = \left|\mathbb{E}_{\bm{\pi},\bm{p}}\left[\sum_{t=0}^{\infty} \gamma^{t} c_{s_{t}a_{t}s_{t+1}} \mid s_{0}=s,a_{0}=a \right]\right| \leq \sum_{t=0}^{\infty}\gamma^{t} = \frac{1}{1-\gamma}.
\end{align*}

Then, by vectorizing the $\bm{\pi}$ as a $SA$-dimensional vector, we have
\begin{align*}
    \|\nabla_{\bm{\pi}}J_{\bm{\rho}}(\bm{\pi},\bm{p})\| &= \sqrt{\sum_{s,a}\left(\frac{\partial J_{\bm{\rho}}(\bm{\pi},\bm{p})}{\partial \pi_{sa}}\right)^{2}}\\
    &= \frac{1}{1-\gamma}\sqrt{\sum_{a}\sum_{s}\left(d_{\bm{\rho}}^{\bm{\pi},\bm{p}}(s)q^{\bm{\pi},\bm{p}}_{sa}\right)^{2}}\\
    &\leq \frac{1}{(1-\gamma)^{2}}\sqrt{\sum_{a}\sum_{s}(d_{\bm{\rho}}^{\bm{\pi},\bm{\xi}}(s))^{2}} \leq \frac{\sqrt{A}}{(1-\gamma)^{2}},
\end{align*}
where the last inequality holds since the discounted state occupancy measure satisfies
\begin{equation*}
    \sum_{s}(d_{\bm{\rho}}^{\bm{\pi},\bm{\xi}}(s))^{2} \leq \left(\sum_{s}(d_{\bm{\rho}}^{\bm{\pi},\bm{\xi}}(s))\right)^{2}=1.
\end{equation*}
About the smoothness of $J_{\bm{\rho}}(\bm{\pi},\bm{p})$, it can be immediately proved by~\citep[Lemma 54]{agarwal2021theory}.
Finally, we turn to derive the continuity of $\Phi(\bm{\pi})$.
1. We first show $\Phi(\bm{\pi})$ is $L_{\bm{\pi}}$-Lipschitz if $J_{\bm{\rho}}(\bm{\pi},\bm{p})$ is $L_{\bm{\pi}}$-Lipschitz in $\bm{\pi}$. For any $\bm{\pi}_{1},\bm{\pi}_{2}\in\Pi$, we let $\bm{p}_{1}:=\arg\max_{\bm{p}\in\mathcal{P}}J_{\bm{\rho}}(\bm{\pi}_{1},\bm{p})$ and $\bm{p}_{2}:=\arg\max_{\bm{p}\in\mathcal{P}}J_{\bm{\rho}}(\bm{\pi}_{2},\bm{p})$, then
\begin{align*}
    \Phi(\bm{\pi}_{1})- \Phi(\bm{\pi}_{2})&=\max_{\bm{p}\in\mathcal{P}}J_{\bm{\rho}}(\bm{\pi}_{1},\bm{p}) - \max_{\bm{p}\in\mathcal{P}}J_{\bm{\rho}}(\bm{\pi}_{2},\bm{p})\\
    &= J_{\bm{\rho}}(\bm{\pi}_{1},\bm{p}_{1})-J_{\bm{\rho}}(\bm{\pi}_{2},\bm{p}_{2})\\
    &\leq
    J_{\bm{\rho}}(\bm{\pi}_{1},\bm{p}_{1})-J_{\bm{\rho}}(\bm{\pi}_{2},\bm{p}_{1})\\
    &\leq L_{\bm{\pi}}\|\bm{\pi}_{1}-\bm{\pi}_{2}\|.
\end{align*}
2. Then, \citep[Lemma 3]{thekumparampil2019efficient}
shows that, $\Phi(\bm{\pi}) = \max_{\bm{p}\in\mathcal{P}}J_{\bm{\rho}}(\bm{\pi},\bm{p})$ is $\ell_{\bm{\pi}}$-weakly convex if $J_{\bm{\rho}}(\bm{\pi},\bm{p})$ is $\ell_{\bm{\pi}}$-smooth.
Combining the results of these two parts, this lemma is proved.
\end{proof}
The following lemma is helpful throughout in the convergence analysis of policy optimization.\\
\begin{lemma}(The performance difference lemma)\label{lem_sec3_3}
For any $\bm{\pi},\bm{\pi}'\in\Pi$, $\bm{p}\in\mathcal{P}$ and  $\bm{\rho}\in\Delta^{S}$, we have
\begin{equation}
   J_{\bm{\rho}}(\bm{\pi},\bm{p})-J_{\bm{\rho}}(\bm{\pi}',\bm{p}) =\frac{1}{1-\gamma} \sum_{s, a} d_{\bm{\rho}}^{\bm{\pi},\bm{p}}(s) \pi_{sa} \left(q^{\bm{\pi}',\bm{p}}_{sa} - v^{\bm{\pi}',\bm{p}}_{s}\right).
\end{equation}
\end{lemma}
Generally, the term $q^{\bm{\pi},\bm{p}}_{sa} - v^{\bm{\pi},\bm{p}}_{s}$ is defined as the~\emph{advantage function}.
\begin{proof}[Proof of \cref{lem_sec3_3}]
By the definition of $J_{\bm{\rho}}(\bm{\pi},\bm{p})$ in (\ref{prob_RMDP}), we have
\begin{equation*}
    J_{\bm{\rho}}(\bm{\pi},\bm{p})-J_{\bm{\rho}}(\bm{\pi}',\bm{p}) = \sum_{s}\rho_{s}\left( v^{\bm{\pi},\bm{p}}_{s}-v^{\bm{\pi}',\bm{p}}_{s}\right).
\end{equation*}
We introduce the \emph{advantage function}$A^{\bm{\pi},\bm{p}}_{sa}:=q^{\bm{\pi},\bm{p}}_{sa}- v^{\bm{\pi},\bm{p}}_{s}$ for convenience, and observe that, for any $s\in\mathcal{S}$,
\begin{align*} 
&v^{\bm{\pi},\bm{p}}_{s}-v^{\bm{\pi}',\bm{p}}_{s} \\
&= v^{\bm{\pi},\bm{p}}_{s}-  \sum_{a}\pi_{sa}\sum_{s'}p_{sas'}\left(c_{sas'}+\gamma v^{\bm{\pi}',\bm{p}}_{s'}\right)
+\sum_{a}\pi_{sa}\sum_{s'}p_{sas'}\left(c_{sas'}+\gamma v^{\bm{\pi}',\bm{p}}_{s'}\right)-
v^{\bm{\pi}',\bm{p}}_{s}\\
&=\sum_{a}\pi_{sa}\sum_{s'}p_{sas'}\left(c_{sas'}+\gamma v^{\bm{\pi},\bm{p}}_{s'}\right)-  \sum_{a}\pi_{sa}\sum_{s'}p_{sas'}\left(c_{sas'}+\gamma v^{\bm{\pi}',\bm{p}}_{s'}\right)\\
&\;\;\;\;+\sum_{a}\pi_{sa}\underbrace{\sum_{s'}p_{sas'}\left(c_{sas'}+\gamma v^{\bm{\pi}',\bm{p}}_{s'}\right)}_{q^{\bm{\pi}',\bm{p}}_{sa}}-
v^{\bm{\pi}',\bm{p}}_{s}\\
&= \gamma\sum_{a}\pi_{sa}\sum_{s'}p_{sas'}\left(v^{\bm{\pi},\bm{p}}_{s'}-v^{\bm{\pi}',\bm{p}}_{s'} \right) + \sum_{a}\pi_{sa} \left(q^{\bm{\pi}',\bm{p}}_{sa}- v^{\bm{\pi}',\bm{p}}_{s}\right)\\
&= \gamma\sum_{a}\pi_{sa}\sum_{s'}p_{sas'}\left(v^{\bm{\pi},\bm{p}}_{s'}-v^{\bm{\pi}',\bm{p}}_{s'} \right) + \sum_{a}\pi_{sa} A^{\bm{\pi}',\bm{p}}_{sa}\\
&\overset{(a)}{=} \gamma\sum_{s'} p^{\bm{\pi}}_{ss'}(1)\left(\gamma\sum_{s''}p^{\bm{\pi}}_{s's''}(1)\left(v^{\bm{\pi},\bm{p}}_{s''}-v^{\bm{\pi}',\bm{p}}_{s''} \right)+\sum_{a'}\pi_{s'a'} A^{\bm{\pi}',\bm{p}}_{s'a'} \right) + \sum_{a}\pi_{sa} A^{\bm{\pi}',\bm{p}}_{sa}\\
&= \sum_{a}\pi_{sa} A^{\bm{\pi}',\bm{p}}_{sa} + \gamma\sum_{s'} p^{\bm{\pi}}_{ss'}(1)\sum_{a'}\pi_{s'a'} A^{\bm{\pi}',\bm{p}}_{s'a'} + \gamma^{2}\sum_{s'} p^{\bm{\pi}}_{ss'}(2)\left(v^{\bm{\pi},\bm{p}}_{s'}-v^{\bm{\pi}',\bm{p}}_{s'} \right)\\
&= \cdots\\
&= \sum_{t=0}^{\infty}\gamma^{t}\sum_{s'}p^{\bm{\pi}}_{ss'}(t)\left( \sum_{a'}\pi_{s'a'} A^{\bm{\pi}',\bm{p}}_{s'a'}\right),
\end{align*}
where $p^{\bm{\pi}}_{ss'}(t)$ is defined in (\ref{inter_no_1}), and $(a)$ uses the recursion. We then obtain 
\begin{align*}
    J_{\bm{\rho}}(\bm{\pi},\bm{p})-J_{\bm{\rho}}(\bm{\pi}',\bm{p}) &= \sum_{s}\rho_{s}\left( v^{\bm{\pi},\bm{p}}_{s}-v^{\bm{\pi}',\bm{p}}_{s}\right)\\
    &= \sum_{s}\rho_{s}\sum_{t=0}^{\infty}\gamma^{t}\sum_{s'}p^{\bm{\pi}}_{ss'}(t)\left( \sum_{a'}\pi_{s'a'} A^{\bm{\pi}',\bm{p}}_{s'a'}\right)\\
    &= \sum_{s'}\left(\sum_{s}\sum_{t=0}^{\infty}\gamma^{t}\rho_{s}p^{\bm{\pi}}_{ss'}(t)\right)\left( \sum_{a'}\pi_{s'a'} A^{\bm{\pi}',\bm{p}}_{s'a'}\right)\\
    &= \frac{1}{1-\gamma} \sum_{s, a} d_{\bm{\rho}}^{\bm{\pi},\bm{p}}(s) \pi_{sa} A^{\bm{\pi}',\bm{p}}_{sa}.
\end{align*}
The last equality is obtained by the definition of state occupancy measure (See Definition~\ref{def:occu}).
\end{proof}

Then, we introduce the gradient dominance condition for non-RMDPs proposed in~\citep{agarwal2021theory}, which will be used in the proof of Theorem~\ref{the_sec3_1_GD}.
\begin{lemma}[Gradient dominance]\label{lem_sec3_4}
For any $\bm{p}\in\mathcal{P}$ and  $\bm{\rho}\in\Delta^{S}$, we have
\begin{equation}
    J_{\bm{\rho}}(\bm{\pi},\bm{p})-J_{\bm{\rho}}(\bm{\pi}^{\star},\bm{p}) \leq \frac{D}{1-\gamma} \max _{\bar{\bm{\pi}}\in\Pi}(\bm{\pi}-\bar{\bm{\pi}})^{\top} \nabla_{\bm{\pi}} J_{\bm{\rho}}(\bm{\pi},\bm{p}),
\end{equation}
where $\bm{\pi}^{\star}$ is one of optimal policies over $\bm{p}$, \ie, $\bm{\pi}^{\star}\in \arg\min_{\bm{\pi}\in\Pi}J_{\bm{\rho}}(\bm{\pi},\bm{p})$.
\end{lemma} 
\begin{proof}[Proof of \cref{lem_sec3_4}]
From the Lemma~\ref{lem_sec3_3}, we have
\begin{align*}
    J_{\bm{\rho}}(\bm{\pi}^{\star},\bm{p})-J_{\bm{\rho}}(\bm{\pi},\bm{p})&= \frac{1}{1-\gamma} \sum_{s, a} d_{\bm{\rho}}^{\bm{\pi}^{\star},\bm{p}}(s) \pi^{\star}_{sa} \left(q^{\bm{\pi},\bm{p}}_{sa} - v^{\bm{\pi},\bm{p}}_{s}\right) \\
    &= \frac{1}{1-\gamma} \sum_{s, a} d_{\bm{\rho}}^{\bm{\pi}^{\star},\bm{p}}(s) \pi^{\star}_{sa} A^{\bm{\pi},\bm{p}}_{sa}\\
    &\geq \frac{1}{1-\gamma} \sum_{s, a} d_{\bm{\rho}}^{\bm{\pi}^{\star},\bm{p}}(s) \pi^{\star}_{sa} \min_{\bar{a}}A^{\bm{\pi},\bm{p}}_{s\bar{a}}\\
    &= \frac{1}{1-\gamma} \sum_{s} d_{\bm{\rho}}^{\bm{\pi}^{\star},\bm{p}}(s)  \min_{\bar{a}}A^{\bm{\pi},\bm{p}}_{s\bar{a}}.\\
\end{align*}
Then, we multiply $-1$ on both sides
\begin{align}\label{eq:lem3.4_1}
    0\leq J_{\bm{\rho}}(\bm{\pi},\bm{p})-J_{\bm{\rho}}(\bm{\pi}^{\star},\bm{p}) &\leq \frac{1}{1-\gamma} \sum_{s} d_{\bm{\rho}}^{\bm{\pi}^{\star},\bm{p}}(s)  -(\min_{\bar{a}}A^{\bm{\pi},\bm{p}}_{s\bar{a}})\notag\\
    &=\frac{1}{1-\gamma} \sum_{s} d_{\bm{\rho}}^{\bm{\pi}^{\star},\bm{p}}(s)  \max_{\bar{a}}(-A^{\bm{\pi},\bm{p}}_{s\bar{a}})\notag\\
    &=\frac{1}{1-\gamma} \sum_{s}
    \frac{d_{\bm{\rho}}^{\bm{\pi}^{\star},\bm{p}}(s)}{d_{\bm{\rho}}^{\bm{\pi},\bm{p}}(s)} d_{\bm{\rho}}^{\bm{\pi},\bm{p}}(s) \max_{\bar{a}}(-A^{\bm{\pi},\bm{p}}_{s\bar{a}})\notag\\
    &\leq \frac{1}{1-\gamma}
    \left(\max_{s}\frac{d_{\bm{\rho}}^{\bm{\pi}^{\star},\bm{p}}(s)}{d_{\bm{\rho}}^{\bm{\pi},\bm{p}}(s)}\right) \sum_{s}d_{\bm{\rho}}^{\bm{\pi},\bm{p}}(s) \max_{\bar{a}}(-A^{\bm{\pi},\bm{p}}_{s\bar{a}})\notag\\
    &= \left\|\frac{d_{\bm{\rho}}^{\bm{\pi}^{\star},\bm{p}}}{d_{\bm{\rho}}^{\bm{\pi},\bm{p}}}\right\|_{\infty} \frac{1}{1-\gamma}\sum_{s}d_{\bm{\rho}}^{\bm{\pi},\bm{p}}(s) \max_{\bar{a}}(-A^{\bm{\pi},\bm{p}}_{s\bar{a}})\notag\\
    &\leq \frac{D}{1-\gamma} \left(\frac{1}{1-\gamma}\sum_{s}d_{\bm{\rho}}^{\bm{\pi},\bm{p}}(s) \max_{\bar{a}}(-A^{\bm{\pi},\bm{p}}_{s\bar{a}})\right)
\end{align}
The last inequality (\ref{eq:lem3.4_1}) is due to the fact~\citep{kakade2002approximately}
\begin{equation*}
\left\|\frac{d_{\bm{\rho}}^{\bm{\pi}^{\star},\bm{p}}}{d_{\bm{\rho}}^{\bm{\pi},\bm{p}}}\right\|_{\infty}
\leq\frac{1}{1-\gamma}\left\|\frac{d_{\bm{\rho}}^{\bm{\pi}^{\star},\bm{p}}}{\bm{\rho}}\right\|_{\infty}
\leq\frac{D}{1-\gamma}.
\end{equation*}
Notice that, the term in (\ref{eq:lem3.4_1}) is equivalent to
\begin{align*}
    \frac{1}{1-\gamma}
    \sum_{s}d_{\bm{\rho}}^{\bm{\pi},\bm{p}}(s) \max_{\bar{a}}(-A^{\bm{\pi},\bm{p}}_{s\bar{a}}) &= \max_{\bar{\bm{\pi}}\in\Pi} \frac{1}{1-\gamma}
    \sum_{s,a}d_{\bm{\rho}}^{\bm{\pi},\bm{p}}(s)\bar{\pi}_{sa} (-A^{\bm{\pi},\bm{p}}_{sa})\\
    &= \max_{\bar{\bm{\pi}}\in\Pi} \frac{1}{1-\gamma}
    \sum_{s,a}d_{\bm{\rho}}^{\bm{\pi},\bm{p}}(s)(\bar{\pi}_{sa}-\pi_{sa}) (-A^{\bm{\pi},\bm{p}}_{sa})\\
    &= \max_{\bar{\bm{\pi}}\in\Pi} \frac{1}{1-\gamma}
    \sum_{s,a}d_{\bm{\rho}}^{\bm{\pi},\bm{p}}(s)(\pi_{sa}-\bar{\pi}_{sa}) q^{\bm{\pi},\bm{p}}_{sa}\\
    &= \max _{\bar{\bm{\pi}}\in\Pi}(\bm{\pi}-\bar{\bm{\pi}})^{\top} \nabla_{\bm{\pi}} J_{\bm{\rho}}(\bm{\pi},\bm{p}).
\end{align*}
The first equality holds since the optimal $\bar{\bm{\pi}}$ is a deterministic policy, \ie, for some $\bar{a}\in\mathcal{A}$, $\bar{\pi}_{s\bar{a}}=1$. The second step is supported by the property $\sum_{a}\pi_{sa}A^{\bm{\pi},\bm{p}}_{sa}=0$. The third step follows as $\sum_{a}(\pi_{sa}-\bar{\pi}_{sa})v_{s}^{\bm{\pi},\bm{p}}=0$ and the last equation is obtained from Lemma~\ref{lem_sec3_1}. Thus, we obtain that
\begin{equation*}
    J_{\bm{\rho}}(\bm{\pi},\bm{p})-J_{\bm{\rho}}(\bm{\pi}^{\star},\bm{p}) \leq \frac{D}{1-\gamma} \max _{\bar{\bm{\pi}}\in\Pi}(\bm{\pi}-\bar{\bm{\pi}})^{\top} \nabla_{\bm{\pi}} J_{\bm{\rho}}(\bm{\pi},\bm{p}) .
\end{equation*}
\end{proof}

Before providing the proof of \cref{the_sec3_1}, we introduce the below intermediate results which are helpful to our proof. We first introduce a common property for strongly convex functions.
\begin{lemma}\label{lem:appB1}
Let $h:\mathcal{X}\rightarrow\mathbb{R}$ be a $\ell$-strongly convex function. Then for any $\bm{x},\bm{y}\in\mathcal{X}$, $h(\bm{x})$, we have
\begin{equation}\label{eq1:lem:appB1}
    h(\bm{y})-h(\bm{x}) \leq \nabla h(\bm{y})^{\top}(\bm{y}-\bm{x})-\frac{\ell}{2}\|\bm{x}-\bm{y}\|^{2}.
\end{equation}
Moreover, by taking $\bm{y}=\bm{x}^\star:=\arg\min_{\bm{x}\in\mathcal{X}}h(\bm{x})$ as the minimum point of $h(\bm{x})$, we get 
\begin{equation}\label{eq2:lem:appB1}
    h(\bm{x})-\min_{\bm{x}\in\mathcal{X}}h(\bm{x}) \geq \frac{\ell}{2}\|\bm{x}-\bm{x}^\star\|^{2}.
\end{equation}
\end{lemma}
\begin{proof}[Proof of \cref{lem:appB1}]
The inequality (\ref{eq1:lem:appB1}) is a basic property that the strongly convex function hold, whereas the second inequality is obtained by the first-order optimality condition for the convex optimization problem, \ie, $\nabla h(\bm{x}^\star)^{\top}(\bm{x}-\bm{x}^\star)\geq0$.
\end{proof}
We also need to introduce the following Danskin's Theorem, which helps prove our global convergence theorem. 
\begin{proposition}\citep[Proposition B.25]{Bertsekas2016}\label{prop1}
Let $\mathcal{Z}\subseteq\mathbb{R}^{m}$ be a compact set, and let $\phi:\mathbb{R}^{n}\times \mathcal{Z}\rightarrow\mathbb{R}$ be continuous function and such that $\phi(\cdot,z):\mathbb{R}^{n}\rightarrow\mathbb{R}$ is convex for each $z\in\mathcal{Z}$. If $\phi(\cdot,z)$ is differentiable for all $z\in\mathcal{Z}$ and $\nabla \phi(\bm{x},\cdot)$ is continuous on $\mathcal{Z}$ for each $\bm{x}$, then for $f(\bm{x}):=\max_{z\in\mathcal{Z}}\phi(\bm{x},z)$ and any $\bm{x}\in\mathbb{R}^{n}$,
\begin{equation*}
\partial f(\bm{x}) = \operatorname{conv}\left\{\nabla_{x}\phi(\bm{x},z) \mid z \in \mathop{\arg\max}\limits_{z\in\mathcal{Z}}\phi(\bm{x},z)\right\}.
\end{equation*}
\end{proposition}

Notice that, Lemma~\ref{lem_sec3_1} successfully proves that $J_{\bm{\rho}}(\bm{\pi},\bm{p})$ is $\ell_{\bm{\pi}}$-smooth and $L_{\bm{\pi}}$-Lipschitz in $\bm{\pi}$. We want to emphasize that, these results also leads to the fact that $J_{\bm{\rho}}(\bm{\pi},\bm{p})$ is $\ell_{\bm{\pi}}$-weakly convex in $\bm{\pi}$ by applying the Lemma~\ref{lem_sec2_1}. 

Now, we are ready to prove Theorem~\ref{the_sec3_1_GD} and Theorem~\ref{the_sec3_1}.
\begin{proof}[Proof of \cref{the_sec3_1_GD}]
Since $J_{\bm{\rho}}(\bm{\pi},\bm{p})$ is non-concave in $\bm{p}$ and the ambiguity set $\mathcal{P}$ is only assumed as a compact set, there may exists multiple inner maxima. In particular, we denote $\bm{p}^{(k)}$ as the $k$-th element of the set $\arg\max_{\bm{p}\in\mathcal{P}}J_{\bm{\rho}}(\bm{\pi},\bm{p})$ for fixed policy $\bm{\pi}\in\Pi$. 
Then, we apply Lemma~\ref{lem_sec3_4} to obtain
\begin{align}\label{eq:part2_1}
    \Phi(\bm{\pi}) - \Phi(\bm{\pi}^{\star}) &= J(\bm{\pi},\bm{p}^{(k)}) - J(\bm{\pi}^{\star},\bm{p}^{\star}) \notag\\
    &= J(\bm{\pi},\bm{p}^{(k)}) - \min_{\bm{\pi}\in\Pi}\max_{\bm{p}\in\mathcal{P}}J(\bm{\pi},\bm{p}) \notag\\
    &\leq J(\bm{\pi},\bm{p}^{(k)}) - \min_{\bm{\pi}\in\Pi}J(\bm{\pi},\bm{p}^{(k)})\notag\\
    &\leq \frac{D}{1-\gamma} \max _{\bar{\bm{\pi}}}(\bm{\pi}-\bar{\bm{\pi}})^{\top} \nabla_{\bm{\pi}} J(\bm{\pi},\bm{p}^{(k)}).
\end{align}
As we mentioned before this proof that, $J_{\bm{\rho}}(\bm{\pi},\bm{p})$ is $\ell_{\bm{\pi}}$-weakly convex in $\bm{\pi}$, it implies that $\tilde{J}_{\bm{\rho}}(\bm{\pi},\bm{p}):=J_{\bm{\rho}}(\bm{\pi},\bm{p}) + \frac{\ell_{\bm{\pi}}}{2}\|\bm{\pi}\|^{2}$ is convex in $\bm{\pi}$ and $\nabla_{\bm{\pi}}\tilde{J}_{\bm{\rho}}(\bm{\pi},\bm{p}) = \nabla_{\bm{\pi}}J_{\bm{\rho}}(\bm{\pi},\bm{p})+\ell_{\bm{\pi}}\bm{\pi}$, referring to~\citep[Corollary 1.12.2]{kruger2003frechet}. Let $\tilde{\Phi}(\bm{\pi}):=\max_{\bm{p}\in\mathcal{P}}\tilde{J}_{\bm{\rho}}(\bm{\pi},\bm{p})$. Due to the convexity of $\tilde{J}_{\bm{\rho}}(\bm{\pi},\bm{p})$ and the compactness of $\mathcal{P}$, we can apply Proposition~\ref{prop1} to attain
\begin{align}\label{eq:Danskin}
    \partial \tilde{\Phi}(\bm{\pi})&=\operatorname{conv}\left\{\nabla_{\bm{\pi}}\tilde{J}_{\bm{\rho}}(\bm{\pi},\bm{p}) \mid \bm{p} \in \mathop{\arg\max}\limits_{\bm{p}\in\mathcal{P}}\tilde{J}_{\bm{\rho}}(\bm{\pi},\bm{p})\right\}\notag\\
    \Longrightarrow \partial \Phi(\bm{\pi})+\ell_{\bm{\pi}}\bm{\pi}&=\operatorname{conv}\left\{\nabla_{\bm{\pi}}J_{\bm{\rho}}(\bm{\pi},\bm{p})+\ell_{\bm{\pi}}\bm{\pi} \mid \bm{p} \in \mathop{\arg\max}\limits_{\bm{p}\in\mathcal{P}}J_{\bm{\rho}}(\bm{\pi},\bm{p})\right\}\notag\\
   \Longrightarrow \partial \Phi(\bm{\pi})&=\operatorname{conv}\left\{\nabla_{\bm{\pi}}J_{\bm{\rho}}(\bm{\pi},\bm{p}) \mid \bm{p} \in \mathop{\arg\max}\limits_{\bm{p}\in\mathcal{P}}J_{\bm{\rho}}(\bm{\pi},\bm{p})\right\}.
\end{align}
Assume the set $\arg\max_{\bm{p}\in\mathcal{P}}J_{\bm{\rho}}(\bm{\pi},\bm{p})$ contains $N$ finite components, then, Proposition~\ref{prop1} implies that, for any $\bm{\pi}\in\Pi$, there exists a sequence $\{\beta_{k}\}^{N}_{k=1}$ with $\sum_{k}\beta_{k}=1$ such that for any sub-gradient $\bm{\xi}\in\partial \Phi(\bm{\pi})$, it can be represented by a convex combination, \ie,
\begin{equation*}
    \bm{\xi} = \sum_{k=1}^{N} \beta_{k} \nabla_{\bm{\pi}}J_{\bm{\rho}}(\bm{\pi},\bm{p}^{(k)})\quad\bm{p}^{(k)}\in \mathop{\arg\max}\limits_{\bm{p}\in\mathcal{P}}J_{\bm{\rho}}(\bm{\pi},\bm{p}), \;k=1,2,\cdots,N
\end{equation*}
Let us define $\tilde{\bm{\pi}} = \arg\min_{\hat{\bm{\pi}}\in\Pi} \Phi(\hat{\bm{\pi}})+\ell_{\bm{\pi}}\|\bm{\pi}-\hat{\bm{\pi}}\|^{2}$ and Lemma~\ref{lem_sec2_3} implies that there exists $\bar{\bm{\xi}}\in\partial \Phi(\tilde{\bm{\pi}})$ such that it satisfies $-\bar{\bm{\xi}} \subseteq \mathcal{N}_{\mathcal{X}}(\tilde{\bm{\pi}}) + 2\ell_{\bm{\pi}}\left\|\tilde{\bm{\pi}}-\bm{\pi}\right\|\cdot\mathcal{B}(1)$. Then by assuming $\arg\max_{\bm{p}\in\mathcal{P}}J_{\bm{\rho}}(\tilde{\bm{\pi}},\bm{p})$ contains $\bar{N}$ finite components, there exists a specific sequence $\{\bar{\beta}_{k}\}_{k=1}^{\bar{N}}$ with $\sum_{k}\bar{\beta}_{k}=1$ such that
\begin{equation}\label{def_part2}
    \bar{\bm{\xi}} = \sum_{k}^{\bar{N}} \bar{\beta}_{k}\nabla_{\bm{\pi}} J_{\bm{\rho}}(\tilde{\bm{\pi}},\tilde{\bm{p}}^{(k)}), \quad\tilde{\bm{p}}^{(k)}\in \mathop{\arg\max}\limits_{\bm{p}\in\mathcal{P}}J_{\bm{\rho}}(\tilde{\bm{\pi}},\bm{p}), \;k=1,2,\cdots,\bar{N}
\end{equation}
Then, we have
\begin{align}\label{eq:part2_2}
    \Phi(\tilde{\bm{\pi}}) - \Phi(\bm{\pi}^{\star})&=\sum_{k=1}^{\bar{N}}\bar{\beta}_{k}\left(\Phi(\tilde{\bm{\pi}}) - \Phi(\bm{\pi}^{\star}) \right)\notag\\
    &\leq \frac{D}{1-\gamma} \sum_{k=1}^{\bar{N}}\bar{\beta}_{k}\left(\max _{\bar{\bm{\pi}}\in\Pi}(\tilde{\bm{\pi}}-\bar{\bm{\pi}})^{\top} \nabla_{\bm{\pi}} J(\tilde{\bm{\pi}},\tilde{\bm{p}}^{(k)})\right)\notag\\
    &\leq \frac{D}{1-\gamma} \sum_{k=1}^{\bar{N}}\bar{\beta}_{k}\langle\max _{\bar{\bm{\pi}}\in\Pi}(\bar{\bm{\pi}} - \tilde{\bm{\pi}}), -\nabla_{\bm{\pi}} J(\tilde{\bm{\pi}},\tilde{\bm{p}}^{(k)})\rangle\notag\\
    &\leq \frac{D}{1-\gamma} \sum_{k=1}^{\bar{N}}\bar{\beta}_{k}\langle(\bar{\bm{\pi}}_{k}-\tilde{\bm{\pi}}), -\nabla_{\bm{\pi}} J(\tilde{\bm{\pi}},\tilde{\bm{p}}^{(k)})\rangle,
\end{align}
where $\bar{\bm{\pi}}_{k}:=\arg\max _{\bar{\bm{\pi}}\in\Pi}\langle(\bar{\bm{\pi}}-\tilde{\bm{\pi}}), -\nabla_{\bm{\pi}} J(\tilde{\bm{\pi}},\tilde{\bm{p}}^{(k)})\rangle$, and the second step is obtained by using (\ref{eq:part2_1}).
Since the cost function is bounded, \ie, $0\leq c_{sas'}\leq1$ for any $(s,a,s')\in\mathcal{S}\times\mathcal{A}\times\mathcal{S}$, it implies 
the action value function $q^{\bm{\pi},\bm{p}}_{s,a}$ and the partial gradient $\nabla_{\bm{\pi}} J(\bm{\pi},\bm{p})$ are non-negative. Since $\Phi(\tilde{\bm{\pi}}) - \Phi(\bm{\pi}^{\star})$ and the partial gradient $\nabla_{\bm{\pi}} J(\bm{\pi},\bm{p})$ are both non-negative, we can denote the maximum element of the vector sequence $\{\bar{\bm{\pi}}_{k}-\tilde{\bm{\pi}}\}_{k=1}^{\bar{N}}$ as $\bar{\pi}_{sa}$ which satisfies $0<\bar{\pi_{sa}}\leq 1$. Then we get
\begin{align}\label{eq_part2_3}
    (\ref{eq:part2_2})
    &\leq \frac{D}{1-\gamma} \sum_{k=1}^{\bar{N}}\bar{\beta}_{k}\langle\bar{\pi}_{sa}\bm{e}, -\nabla_{\bm{\pi}} J(\tilde{\bm{\pi}},\tilde{\bm{p}}^{(k)})\rangle\\
    &= \frac{D}{1-\gamma}\langle\bar{\pi}_{sa}\bm{e}, \sum_{k=1}^{\bar{N}}\bar{\beta}_{k}\left(-\nabla_{\bm{\pi}} J(\tilde{\bm{\pi}},\tilde{\bm{p}}^{(k)})\right)\rangle\notag\\
    &\leq \frac{D}{1-\gamma}\langle\bm{e}, -\hat{\bm{\xi}}\rangle\leq \frac{D\sqrt{SA}}{1-\gamma}\|\nabla \Phi_{\frac{1}{2\ell_{\bm{\pi}}}}(\bm{\pi})\|.
\end{align}
Here, the last inequality follows from the definition of $\bar{\bm{d}}(\tilde{\bm{\pi}}_{t})$ which is mentioned in (\ref{def_part2}) and $\bm{e}$ is all-one vector defined in Section~\ref{sec1}. Remind that, Lemma~\ref{lem_sec3_1} implies $J_{\bm{\rho}}(\bm{\pi},\bm{p})$ is $L_{\bm{\pi}}$-Lipschitz in $\bm{\pi}$, and Lemma~\ref{lem_sec3_1} also shows that $\Phi(\bm{\pi})$ is $L_{\bm{\pi}}$-Lipschitz. Thus, combine this Lipschitz property and the above equation (\ref{eq_part2_3}), we get 
\begin{align}\label{eq:the1_2}
    \Phi(\bm{\pi}) - \Phi(\bm{\pi}^{\star}) 
    &=  \Phi(\bm{\pi})-\Phi(\tilde{\bm{\pi}})+\Phi(\tilde{\bm{\pi}})-\Phi(\bm{\pi}^{\star})\notag\\
    &\leq \frac{D\sqrt{SA}}{1-\gamma}\|\nabla \Phi_{\frac{1}{2\ell_{\bm{\pi}}}}(\bm{\pi})\| + \Phi(\bm{\pi})-\Phi(\tilde{\bm{\pi}})\notag\\
    &\leq \frac{D\sqrt{SA}}{1-\gamma}\|\nabla \Phi_{\frac{1}{2\ell_{\bm{\pi}}}}(\bm{\pi})\| + L_{\bm{\pi}}\|\bm{\pi}-\tilde{\bm{\pi}}\|\notag\\
    &= \frac{D\sqrt{SA}}{1-\gamma}\|\nabla \Phi_{\frac{1}{2\ell_{\bm{\pi}}}}(\bm{\pi})\| + L_{\bm{\pi}}\cdot \frac{\|\nabla \Phi_{\frac{1}{2\ell_{\bm{\pi}}}}(\bm{\pi})\|}{2\ell_{\bm{\pi}}},
\end{align}
where (\ref{eq:the1_2}) holds by using arguments of Lemma~\ref{lem_sec2_4} and Lemma~\ref{lem_sec2_3}.

\end{proof}

\begin{proof}[Proof of \cref{the_sec3_1}]
The proof is split into two parts. We first show our algorithm can reach a $\epsilon$-first stationary point of $\Phi(\bm{\pi}):=\max_{\bm{p}\in\mathcal{P}}J_{\bm{\rho}}(\bm{\pi},\bm{p})$. Then, we next prove that this $\epsilon$-first stationary point is close enough to the global minimum of $\Phi(\bm{\pi})$.

We begin by defining a policy $ \tilde{\bm{\pi}}_{t} = \arg\min_{\tilde{\bm{\pi}}\in\Pi} \Phi(\tilde{\bm{\pi}})+\ell_{\bm{\pi}}\|\bm{\pi}_{t}-\tilde{\bm{\pi}}\|^{2}$ where $\Phi(\bm{\pi})$ has been well defined as the objective function $J_{\bm{\rho}}(\bm{\pi},\bm{p})$ taking the worst-case transition probability, then, we have
\begin{align}\label{eq:mid1}
    \Phi_{\frac{1}{2\ell_{\bm{\pi}}}}(\bm{\pi}_{t+1}) &= \min_{\bm{\pi}}\Phi(\bm{\pi})+\ell_{\bm{\pi}}\|\bm{\pi}_{t+1}-\bm{\pi}\|^{2} \notag\\
    &\leq \Phi(\tilde{\bm{\pi}}_{t})+\ell_{\bm{\pi}}\|\bm{\pi}_{t+1}-\tilde{\bm{\pi}}_{t}\|^{2} \notag\\
    &= \Phi(\tilde{\bm{\pi}}_{t})+\ell_{\bm{\pi}}\|\mathcal{P}_{\Pi}(\bm{\pi}_{t}-\alpha\nabla_{\bm{\pi}}J_{\bm{\rho}}(\bm{\pi}_{t},\bm{p}_{t}))-\mathcal{P}_{\Pi}(\tilde{\bm{\pi}}_{t})\|^{2}\notag\\
    &\overset{(a)}{\leq} \Phi(\tilde{\bm{\pi}}_{t})+\ell_{\bm{\pi}}\|\bm{\pi}_{t}-\alpha\nabla_{\bm{\pi}}J_{\bm{\rho}}(\bm{\pi}_{t},\bm{p}_{t})-\tilde{\bm{\pi}}_{t}\|^{2} \notag\\
    &= \Phi(\tilde{\bm{\pi}}_{t})+\ell_{\bm{\pi}}\|\bm{\pi}_{t}-\tilde{\bm{\pi}}_{t}\|^{2} - 2\ell_{\bm{\pi}}\alpha\langle\nabla_{\bm{\pi}}J_{\bm{\rho}}(\bm{\pi}_{t},\bm{p}_{t}),\bm{\pi}_{t}-\tilde{\bm{\pi}}_{t}\rangle + \alpha^{2}\ell_{\bm{\pi}}\|\nabla_{\bm{\pi}}J_{\bm{\rho}}(\bm{\pi}_{t},\bm{p}_{t})\|^{2}\notag\\
    &\leq \Phi_{\frac{1}{2\ell_{\bm{\pi}}}}(\bm{\pi}_{t}) +2\ell_{\bm{\pi}}\alpha\left(\Phi(\tilde{\bm{\pi}}_{t}) - \Phi(\bm{\pi}_{t}) + \epsilon_{t} + \frac{\ell_{\bm{\pi}}}{2}\|\bm{\pi}_{t}-\tilde{\bm{\pi}}_{t}\|^{2}\right) + \alpha^{2}\ell_{\bm{\pi}}L_{\bm{\pi}}^{2},
\end{align}
where $(\bm{\pi}_{t},\bm{p}_{t})$ is produced from the DRPG scheme at iteration step $t$ and  $\Phi_{\frac{1}{2\ell_{\bm{\pi}}}}$ is the Moreau envelope fucntion of $\Phi$ with parameter $\lambda = \frac{1}{2\ell_{\bm{\pi}}}$. The inequality $(a)$ follows the basic projection property~\citep{rockafellar1976monotone}, \ie, for any $\bm{x}_{1},\bm{x}_{2}\in\mathbb{R}^{n}$, 
\begin{equation*}
    \|\mathcal{P}_{\mathcal{X}}(\bm{x}_{1})-\mathcal{P}_{\mathcal{X}}(\bm{x}_{2})\|\leq\|\bm{x}_{1}-\bm{x}_{2}\|,
\end{equation*}
and the last inequality holds due to the fact that $J_{\bm{\rho}}(\bm{\pi},\bm{p})$ is $\ell_{\bm{\pi}}$-weakly convex in $\bm{\pi}$, in the sense that, for the $\tilde{\bm{\pi}}_{t}$,
\begin{align*}
    \Phi(\tilde{\bm{\pi}}_{t}) \geq J_{\bm{\rho}}(\tilde{\bm{\pi}}_{t},\bm{p}_{t}) &\geq J_{\bm{\rho}}(\bm{\pi}_{t},\bm{p}_{t}) + \langle\nabla_{\bm{\pi}}J_{\bm{\rho}}(\bm{\pi}_{t},\bm{p}_{t}),\tilde{\bm{\pi}}_{t}-\bm{\pi}_{t}\rangle - \frac{\ell_{\bm{\pi}}}{2}\|\bm{\pi}_{t}-\tilde{\bm{\pi}}_{t}\|^{2}\\
    &\geq \underbrace{\max_{\bm{p}\in\mathcal{P}}J_{\bm{\rho}}(\bm{\pi}_{t},\bm{p})}_{\Phi(\bm{\pi}_{t})} - \epsilon_{t} + \langle\nabla_{\bm{\pi}}J_{\bm{\rho}}(\bm{\pi}_{t},\bm{p}_{t}),\tilde{\bm{\pi}}_{t}-\bm{\pi}_{t}\rangle - \frac{\ell_{\bm{\pi}}}{2}\|\bm{\pi}_{t}-\tilde{\bm{\pi}}_{t}\|^{2}.
\end{align*}
Next, by summing (\ref{eq:mid1}) up over $t$, we obtain,
\begin{equation*}
   \Phi_{\frac{1}{2\ell_{\bm{\pi}}}}(\bm{\pi}_{T-1}) \leq \Phi_{\frac{1}{2\ell_{\bm{\pi}}}}(\bm{\pi}_{0}) +2\ell_{\bm{\pi}}\alpha\sum^{T-1}_{t=0}\left(\Phi(\tilde{\bm{\pi}}_{t}) - \Phi(\bm{\pi}_{t}) + \epsilon_{t} + \frac{\ell_{\bm{\pi}}}{2}\|\bm{\pi}_{t}-\tilde{\bm{\pi}}_{t}\|^{2}\right) + T\alpha^{2}\ell_{\bm{\pi}}L_{\bm{\pi}}^{2}.
\end{equation*}
Rearranging this inequality yields
\begin{equation}\label{eq:the_sum1}
    \sum^{T-1}_{t=0}\left( \Phi(\bm{\pi}_{t}) -\Phi(\tilde{\bm{\pi}}_{t}) - \frac{\ell_{\bm{\pi}}}{2}\|\bm{\pi}_{t}-\tilde{\bm{\pi}}_{t}\|^{2}\right) \leq \frac{\Phi_{\frac{1}{2\ell_{\bm{\pi}}}}(\bm{\pi}_{0})-\Phi_{\frac{1}{2\ell_{\bm{\pi}}}}(\bm{\pi}_{T-1})}{2\ell_{\bm{\pi}}\alpha} + \frac{T\alpha L_{\bm{\pi}}^{2}}{2} + \sum^{T-1}_{t=0}\epsilon_{t}.
\end{equation}
Then, we have
\begin{align} \label{eq:part1_1}
    &\Phi(\bm{\pi}_{t})-\Phi(\tilde{\bm{\pi}}_{t}) - \frac{\ell_{\bm{\pi}}}{2}\|\bm{\pi}_{t}-\tilde{\bm{\pi}}_{t}\|^{2} \notag\\
    &= \Phi(\bm{\pi}_{t})+ \ell_{\bm{\pi}}\|\bm{\pi}_{t}-\bm{\pi}_{t}\|^{2}-\Phi(\tilde{\bm{\pi}}_{t}) -\ell_{\bm{\pi}}\|\bm{\pi}_{t}-\tilde{\bm{\pi}}_{t}\|^{2} + \frac{\ell_{\bm{\pi}}}{2}\|\bm{\pi}_{t}-\tilde{\bm{\pi}}_{t}\|^{2}\notag\\
    & = \Phi(\bm{\pi}_{t})+ \ell_{\bm{\pi}}\|\bm{\pi}_{t}-\bm{\pi}_{t}\|^{2}-\min_{\bm{\pi}\in\Pi}\left\{\Phi(\bm{\pi}) +\ell_{\bm{\pi}}\|\bm{\pi}_{t}-\bm{\pi}\|^{2}\right\} + \frac{\ell_{\bm{\pi}}}{2}\|\bm{\pi}_{t}-\tilde{\bm{\pi}}_{t}\|^{2}\notag\\
    &\overset{(a)}{\geq} \ell_{\bm{\pi}}\|\bm{\pi}_{t}-\tilde{\bm{\pi}}_{t}\|^{2} = \frac{1}{4\ell_{\bm{\pi}}}\|\nabla \Phi_{\frac{1}{2\ell_{\bm{\pi}}}}(\bm{\pi}_{t})\|^{2}.
\end{align}
The inequality $(a)$ is obtained by the Lemma~\ref{lem:appB1} and the last equality in (\ref{eq:part1_1}) is obtained by using the gradient of Moreau envelope function proposed in Lemma~\ref{lem_sec2_4}, \ie,
\begin{equation*}
    \nabla \Phi_{\frac{1}{2\ell_{\bm{\pi}}}}(\bm{\pi}_{t}) =  2\ell_{\bm{\pi}}\left(\bm{\pi}_{t}-\mathop{\arg\max}\limits_{\bm{\pi}\in\Pi} \left(\Phi(\bm{\pi})+\ell_{\bm{\pi}}\left\|\bm{\pi}_{t}-\bm{\pi}\right\|^{2} \right)\right)=2\ell_{\bm{\pi}}\left(\bm{\pi}_{t}-\tilde{\bm{\pi}}_{t}\right).
\end{equation*}
Let $\bar{\bm{\pi}}_{1}:= \arg\min_{\bar{\bm{\pi}}\in\Pi} \Phi(\bar{\bm{\pi}})+\ell_{\bm{\pi}}\|\bm{\pi}_{1}-\bar{\bm{\pi}}\|^{2}$ and $\bar{\bm{\pi}}_{2}:= \arg\min_{\bar{\bm{\pi}}\in\Pi} \Phi(\bar{\bm{\pi}})+\ell_{\bm{\pi}}\|\bm{\pi}_{2}-\bar{\bm{\pi}}\|^{2}$ for any $\bm{\pi}_{1},\bm{\pi}_{2}\in\Pi$, and then we have
\begin{align}\label{eq:set_dis}
\Phi_{\frac{1}{2\ell_{\bm{\pi}}}}(\bm{\pi}_{1})-\Phi_{\frac{1}{2\ell_{\bm{\pi}}}}(\bm{\pi}_{2}) &= \min_{\bar{\bm{\pi}}\in\Pi} \left(\Phi(\bar{\bm{\pi}})+\ell_{\bm{\pi}}\|\bm{\pi}_{1}-\bar{\bm{\pi}}\|^{2}\right) - \min_{\bar{\bm{\pi}}\in\Pi} \left(\Phi(\bar{\bm{\pi}})+\ell_{\bm{\pi}}\|\bm{\pi}_{2}-\bar{\bm{\pi}}\|^{2}\right)\\
&= \Phi(\bar{\bm{\pi}}_{1})+\ell_{\bm{\pi}}\|\bm{\pi}_{1}-\bar{\bm{\pi}}_{1}\|^{2} - \Phi(\bar{\bm{\pi}}_{2})-\ell_{\bm{\pi}}\|\bm{\pi}_{2}-\bar{\bm{\pi}}_{2}\|^{2}\\
&\leq \Phi(\bar{\bm{\pi}}_{2})+\ell_{\bm{\pi}}\|\bm{\pi}_{1}-\bar{\bm{\pi}}_{2}\|^{2} - \Phi(\bar{\bm{\pi}}_{2})-\ell_{\bm{\pi}}\|\bm{\pi}_{2}-\bar{\bm{\pi}}_{2}\|^{2}\\
&= \ell_{\bm{\pi}}\left(\|\bm{\pi}_{1}-\bar{\bm{\pi}}_{2}\|^{2} - \|\bm{\pi}_{2}-\bar{\bm{\pi}}_{2}\|^{2}\right)\\
&\leq2\ell_{\bm{\pi}}S.
\end{align}
Plug (\ref{eq:set_dis}) and (\ref{eq:part1_1}) into (\ref{eq:the_sum1}) and reach the first result that 
\begin{equation}\label{eq:the_sum2}
   \sum^{T-1}_{t=0}\|\nabla \Phi_{\frac{1}{2\ell_{\bm{\pi}}}}(\bm{\pi}_{t})\|^{2} \leq \frac{4\ell_{\bm{\pi}}S}{\alpha} + 2T\alpha\ell_{\bm{\pi}} L_{\bm{\pi}}^{2} + 4\ell_{\bm{\pi}}\sum^{T-1}_{t=0}\epsilon_{t}.
\end{equation}
Notice that, when the LHS is smaller than $T\epsilon^{2}$, \ie,
\begin{equation*}
    T\cdot\min_{t}\|\nabla \Phi_{\frac{1}{2\ell_{\bm{\pi}}}}(\bm{\pi}_{t})\|^{2}\leq\sum^{T-1}_{t=0}\|\nabla \Phi_{\frac{1}{2\ell_{\bm{\pi}}}}(\bm{\pi}_{t})\|^{2}\leq T\epsilon^{2},
\end{equation*}
there exists one $\hat{t}$ such that $\|\nabla \Phi_{\frac{1}{2\ell_{\bm{\pi}}}}(\bm{\pi}_{\hat{t}})\|\leq \epsilon$ and $\bm{\pi}_{\hat{t}}$ is a $\epsilon$-first order stationary point for $\Phi(\bm{\pi})$.\

We finished the first part of the proof, and the next step is to show this approximate stationary point is close to the global minimum of $\Phi(\bm{\pi})$. Formally, we next to show there exists some $t$ such that
\begin{equation}\label{eq;mid2}
    J_{\bm{\rho}}(\bm{\pi}^{\star},\bm{p}^{\star}) - \max_{\bm{p}\in\mathcal{P}}J_{\bm{\rho}}(\bm{\pi}_{t},\bm{p}) =  \Phi(\bm{\pi}^\star) - \Phi(\bm{\pi}_{t}) \leq \epsilon.
\end{equation}

Applying the result in Theorem~\ref{the_sec3_1_GD} for the iterative policy $\bm{\pi}_{t}$, we have
\begin{align}\label{eq:the_2}
    J(\bm{\pi}_{t},\bm{p}_{t}) - \min_{\bm{\pi}\in\Pi}\max_{\bm{p}\in\mathcal{P}}J_{\bm{\rho}}(\bm{\pi},\bm{p})
    &\leq\Phi(\bm{\pi}_{t}) - \Phi(\bm{\pi}^{\star})
    \leq \frac{D\sqrt{SA}}{1-\gamma}\|\nabla \Phi_{\frac{1}{2\ell_{\bm{\pi}}}}(\bm{\pi}_{t})\| + L_{\bm{\pi}}\cdot \frac{\|\nabla \Phi_{\frac{1}{2\ell_{\bm{\pi}}}}(\bm{\pi}_{t})\|}{2\ell_{\bm{\pi}}}.
\end{align}

Combined this two parts, we finally state the global convergence guarantee. Equation (\ref{eq:the_2}) implies that
\begin{align}\label{eq:the_3}
\min_{t\in\{0,\cdots,T-1\}}\left\{J(\bm{\pi}_{t},\bm{p}_{t}) - \min_{\bm{\pi}\in\Pi}\max_{\bm{p}\in\mathcal{P}}J_{\bm{\rho}}(\bm{\pi},\bm{p})\right\}
&\leq\frac{1}{T}\sum_{t=0}^{T-1}\left(J(\bm{\pi}_{t},\bm{p}_{t}) - \min_{\bm{\pi}\in\Pi}\max_{\bm{p}\in\mathcal{P}}J_{\bm{\rho}}(\bm{\pi},\bm{p})\right)\notag\\
&\leq\frac{1}{T}\sum_{t=0}^{T-1}\left(\Phi(\bm{\pi}_{t}) - \Phi(\bm{\pi}^{\star})\right)\notag\\
&\leq \frac{1}{T} \left(\frac{D\sqrt{SA}}{1-\gamma}+\frac{L_{\bm{\pi}}}{2\ell_{\bm{\pi}}}\right)\sum^{T-1}_{t=0}\left\|\nabla \Phi_{\frac{1}{2\ell_{\bm{\pi}}}}(\bm{\pi}_{t})\right\|
\end{align}
By Cauchy–Schwarz inequality, we can obtain
\begin{equation*}
    \frac{1}{\sqrt{T}}\sum^{T-1}_{t=0}\left\|\nabla \Phi_{\frac{1}{2\ell_{\bm{\pi}}}}(\bm{\pi}_{t})\right\| \leq \sqrt{\sum^{T-1}_{t=0}\|\nabla \Phi_{\frac{1}{2\ell_{\bm{\pi}}}}(\bm{\pi}_{t})\|^{2}}.
\end{equation*}

We then multiply the constant $\frac{D\sqrt{SA}}{1-\gamma}+\frac{L_{\bm{\pi}}}{2\ell_{\bm{\pi}}}$ on both sides and combine the inequality (\ref{eq:the_sum2}) to obtain the result that, if the iteration time $T$ satisfies
\begin{align*}
    (\ref{eq:the_3})&\leq\frac{1}{\sqrt{T}} \left(\frac{D\sqrt{SA}}{1-\gamma}+\frac{L_{\bm{\pi}}}{2\ell_{\bm{\pi}}}\right)\sqrt{\sum^{T-1}_{t=0}\|\nabla \Phi_{\frac{1}{2\ell_{\bm{\pi}}}}(\bm{\pi}_{t})\|^{2}} \\
    &= \frac{1}{\sqrt{T}} \left(\frac{D\sqrt{SA}}{1-\gamma}+\frac{L_{\bm{\pi}}}{2\ell_{\bm{\pi}}}\right)\sqrt{\left(\frac{4\ell_{\bm{\pi}}S}{\alpha} + 2T\alpha\ell_{\bm{\pi}} L_{\bm{\pi}}^{2} + 4\ell_{\bm{\pi}}\sum^{T-1}_{t=0}\epsilon_{t}\right)}\\ 
    &\overset{(a)}{\leq} \frac{1}{\sqrt{T}} \left(\frac{D\sqrt{SA}}{1-\gamma}+\frac{L_{\bm{\pi}}}{2\ell_{\bm{\pi}}}\right) \sqrt{\left(\frac{4\ell_{\bm{\pi}}S\sqrt{T}}{\delta} + 2\sqrt{T}\delta\ell_{\bm{\pi}} L_{\bm{\pi}}^{2} + \frac{4\ell_{\bm{\pi}}\epsilon_{0}}{1-\gamma}\right)}\\
    &\leq \frac{1}{\sqrt{T}} \left(\frac{D\sqrt{SA}}{1-\gamma}+\frac{L_{\bm{\pi}}}{2\ell_{\bm{\pi}}}\right) \sqrt{\left(\frac{4\ell_{\bm{\pi}}S\sqrt{T}}{\delta} + 2\sqrt{T}\delta\ell_{\bm{\pi}} L_{\bm{\pi}}^{2} + \frac{4\ell_{\bm{\pi}}\sqrt{T}}{1-\gamma}\right)}\\
    &= \epsilon
\end{align*}
where the inequality $(a)$ holds due to the adaptive tolerance sequence, in the sense that,
\begin{equation*}
    \sum^{T-1}_{t=0}\epsilon_{t} \leq \sum^{\infty}_{t=0}\epsilon_{t} \leq \epsilon_{0}\cdot\left(1+\gamma+\gamma^{2}+\cdots\right) \leq\frac{\epsilon_{0}}{1-\gamma},
\end{equation*}
which implies that
\begin{equation*}
    T \geq \frac{\left(\frac{D\sqrt{SA}}{1-\gamma}+\frac{L_{\bm{\pi}}}{2\ell_{\bm{\pi}}}\right)^{4}\left(\frac{4\ell_{\bm{\pi}}S}{\delta} + 2\delta\ell_{\bm{\pi}} L_{\bm{\pi}}^{2} + \frac{4\ell_{\bm{\pi}}}{1-\gamma}\right)^{2}}{\epsilon^{4}} = \mathcal{O}(\epsilon^{-4}),
\end{equation*}
then, we have
\begin{equation*}
\min_{t\in\{0,\cdots,T-1\}}\left\{J(\bm{\pi}_{t},\bm{p}_{t}) - \min_{\bm{\pi}\in\Pi}\max_{\bm{p}\in\mathcal{P}}J_{\bm{\rho}}(\bm{\pi},\bm{p})\right\} \leq \epsilon.
\end{equation*}
Intuitively, we have
\begin{equation*}
    \min_{t\in\{0,\cdots,T-1\}}\left\{\Phi(\bm{\pi}_{t}) - \min_{\bm{\pi}\in\Pi}\Phi(\bm{\pi})\right\}\leq \epsilon.
\end{equation*}
\end{proof}

\section{Discussion on R-contamination ambiguity set}\label{sec:R_contamination}
Recall that the R-contamination ambiguity set is a kind of $(s,a)$-rectangular set $\mathcal{P}=\underset{s \in \mathcal{S},a\in\mathcal{A}}{\times} \mathcal{P}_{s,a}$ where $\mathcal{P}_{s,a}$ is defined as
\begin{equation}
    \mathcal{P}_{s,a} := \{(1-R)\bar{\bm{p}}_{sa}+R\bm{q} \mid \bm{q}\in\Delta(S)\},s \in \mathcal{S},a\in\mathcal{A}.
\end{equation}

We have the following property of the R-contamination sets which illustrates their limited applicability.
\begin{proposition}\label{prop:r-contamination-useless}
Any RMDP with an R-contamination ambiguity set has the same optimal robust policy as a corresponding ordinary MDP with a reduced discount factor.
\end{proposition}
\begin{proof}[Proof of \cref{prop:r-contamination-useless}]
    The robust optimal bellman operator of a RMDP with R-contamination ambiguity can be written as 
\begin{align*}
    (\mathcal{T}^{r}\bm{v})_{s} :&= \min_{a\in\mathcal{A}}\max_{\bm{p}_{sa}\in\mathcal{P}_{s,a}}(c_{sa}+\gamma \bm{p}_{sa}^{\top}\bm{v})\\
    &=\min_{a\in\mathcal{A}}c_{sa}+\gamma\left[(1-R)\bar{\bm{p}}_{sa}^{\top}\bm{v} + R\max_{s'}v_{s'}\right]\\
    &=\left[\min_{a\in\mathcal{A}}c_{sa}+\gamma(1-R)\bar{\bm{p}}_{sa}^{\top}\bm{v}\right] + R\gamma\max_{s'}v_{s'}
\end{align*}
Consider an ordinary MDP with the same reward function, transition kernel $\bm{p}:=(\bar{\bm{p}}_{sa})_{s\in\mathcal{S},a\in\mathcal{A}}\in(\Delta^{S})^{S\times A}$ and discount factor $\gamma(1-R)$. The optimal bellman operator is defined as 
\begin{equation*}
    (\mathcal{T}\bm{v})_{s} := \min_{a\in\mathcal{A}}c_{sa}+\gamma(1-R)\bar{\bm{p}}_{sa}^{\top}\bm{v}.
\end{equation*}
Then, we have that
\begin{equation}
    (\mathcal{T}^{r}\bm{v})_{s} = (\mathcal{T}\bm{v})_{s} + R\gamma\|\bm{v}\|_{\infty}
\end{equation}
We define optimal value functions for $\mathcal{T}^{r}$ and $\mathcal{T}$ as follow
\begin{align*}
    \mathcal{T}^{r}\bm{v}^{r} = \bm{v}^{r},\quad
    \mathcal{T}\bm{v}^{nr} = \bm{v}^{nr},
\end{align*}
and consider the value iteration with the given initial value functions $\bm{v}^{r}$ first. Then we have that
\begin{align*}
    \mathcal{T}^{r}\bm{v}^{r} &= \mathcal{T}\bm{v}^{r} + R\gamma\|\bm{v}^{r}\|_{\infty}\bm{e}\\
    \Longleftrightarrow\;\;\; (\mathcal{T}^{r})^{2}\bm{v}^{r} &= (\mathcal{T})^{2}\bm{v}^{r} + R\gamma \|\mathcal{T}^{r}\bm{v}^{r}\|_{\infty}\bm{e}+R\gamma^{2}(1-R)\|\bm{v}^{r}\|_{\infty}\bm{e}\\
    &=(\mathcal{T})^{2}\bm{v}^{r} + \left[R\gamma +R\gamma^{2}(1-R)\right]\cdot\|\bm{v}^{r}\|_{\infty}\cdot\bm{e}\\
    \Longleftrightarrow\;\;\;(\mathcal{T}^{r})^{k}\bm{v}^{r} &= (\mathcal{T})^{k}\bm{v}^{r} + \left[R\gamma +R\gamma^{2}(1-R) + R\gamma^{3}(1-R)^{2}+\cdots\right]\cdot\|\bm{v}^{r}\|_{\infty}\cdot\bm{e}\\
    &=(\mathcal{T})^{k}\bm{v}^{r} + \sum_{n=1}^{k}R\gamma^{k}(1-R)^{k-1}\cdot\|\bm{v}^{r}\|_{\infty}\cdot\bm{e}.
\end{align*}
By taking the limitation for both side, we obtain
\begin{align*}
    \lim_{k\rightarrow\infty}(\mathcal{T}^{r})^{k}\bm{v}^{r} &= \bm{v}^{r}\\
    &=\lim_{k\rightarrow\infty}\left[(\mathcal{T})^{k}\bm{v}^{r}+ \sum_{n=1}^{k}R\gamma^{n}(1-R)^{n-1}\cdot\|\bm{v}^{r}\|_{\infty}\cdot\bm{e}\right]\\
    &= \bm{v}^{nr} + \lim_{k\rightarrow\infty}\left[\sum_{n=1}^{k}R\gamma^{n}(1-R)^{n-1}\cdot\|\bm{v}^{r}\|_{\infty}\cdot\bm{e}\right]\\
    &= \bm{v}^{nr} +  \lim_{k\rightarrow\infty}\left[\frac{1-(\gamma(1-R))^{k}}{1-\gamma(1-R)}\right]\cdot\|\bm{v}^{r}\|_{\infty}\cdot\bm{e}\\
    &=\bm{v}^{nr} + \frac{1}{1-\gamma(1-R)}\cdot\|\bm{v}^{r}\|_{\infty}\cdot\bm{e}.
\end{align*}
Each operation $\mathcal{T}^{r}$ on $\bm{v}^{r}$ will take the same optimal action due to the definition of $\bm{v}^{r}$, which implies operation $\mathcal{T}^{r}$ on $\bm{v}^{r}$ works with the same action is taken. This intuitive result shows that the RMDP with R-contamination ambiguity and its corresponding ordinary MDP with discount factor $\gamma(1-R)$ has the same optimal policy.
\end{proof}

\section{Proofs of \cref{sec:inner-loop-maxim}}
\begin{proof}[Proof of \cref{lem_sec4_1}]
\label{app: Proof of sec4}
Notice that
\begin{equation*}
    \frac{\partial J_{\bm{\rho}}(\bm{\pi},\bm{p})}{\partial p_{sas'}} = \sum_{\hat{s} \in \mathcal{S}} \frac{\partial v^{\bm{\pi},\bm{p}}_{\hat{s}}}{\partial p_{sas'}} \rho_{\hat{s}}.
\end{equation*}
Then, we discuss $\frac{\partial v^{\bm{\pi},\bm{p}}_{\hat{s}}}{\partial p_{sas'}}$ over two cases: $\hat{s}\neq s$ and $\hat{s}= s$
\begin{align*}
    \frac{\partial v^{\bm{\pi},\bm{p}}_{\hat{s}}}{\partial p_{sas'}}\Big\vert_{\hat{s}\neq s} &= \frac{\partial}{\partial p_{sas'}}\left[\sum_{\hat{a}}\pi_{\hat{s}\hat{a}}\sum_{\hat{s}' \in \mathcal{S}} p_{\hat{s}\hat{a}\hat{s}'}\left(c_{\hat{s}\hat{a}\hat{s}'}+\gamma v^{\bm{\pi},\bm{p}}_{\hat{s}'}\right)\right] = \gamma\sum_{\hat{a}}\pi_{\hat{s}\hat{a}}\sum_{\hat{s}' \in \mathcal{S}} p_{\hat{s}\hat{a}\hat{s}'} \frac{\partial v^{\bm{\pi},\bm{p}}_{\hat{s}'}}{\partial p_{sas'}};\\
    \frac{\partial v^{\bm{\pi},\bm{p}}_{\hat{s}}}{\partial p_{sas'}}\Big\vert_{\hat{s}= s} &= \gamma\sum_{\hat{a}}\pi_{s\hat{a}}\sum_{\hat{s}' \in \mathcal{S}} p_{s\hat{a}\hat{s}'} \frac{\partial v^{\bm{\pi},\bm{p}}_{\hat{s}'}}{\partial p_{sas'}} + \pi_{sa}\left(c_{sas'}+\gamma v^{\bm{\pi},\bm{p}}_{s'}\right);
\end{align*}
By condensing $\sum_{\hat{a}}\pi_{\hat{s}\hat{a}} p_{\hat{s}\hat{a}\hat{s}'}=p^{\bm{\pi}}_{\hat{s}\hat{s}'}(1)$, we can obtain,
\begin{align*}
     \frac{\partial v^{\bm{\pi},\bm{p}}_{\hat{s}}}{\partial p_{sas'}}\Big\vert_{\hat{s}\neq s} &= \gamma\sum_{\hat{s}' \neq s } p^{\bm{\pi}}_{\hat{s}\hat{s}'}(1) \frac{\partial v^{\bm{\pi},\bm{p}}_{\hat{s}'}}{\partial p_{sas'}} + \gamma\sum_{\hat{s}' = s } p^{\bm{\pi}}_{\hat{s}\hat{s}'}(1) \frac{\partial v^{\bm{\pi},\bm{p}}_{\hat{s}'}}{\partial p_{sas'}}\\
     &= \gamma\sum_{\hat{s}' \neq s } p^{\bm{\pi}}_{\hat{s}\hat{s}'}(1) \cdot \gamma\sum_{\hat{a}}\pi_{\hat{s}'\hat{a}}\sum_{\hat{s}'' \in \mathcal{S}} p_{\hat{s}'\hat{a}\hat{s}''} \frac{\partial v^{\bm{\pi},\bm{p}}_{\hat{s}''}}{\partial p_{sas'}} \\
     &\;\;\;\;+ \gamma p^{\bm{\pi}}_{\hat{s}s}(1) \cdot \left( \gamma\sum_{\hat{a}}\pi_{s\hat{a}}\sum_{\hat{s}' \in \mathcal{S}} p_{s\hat{a}\hat{s}'} \frac{\partial v^{\bm{\pi},\bm{p}}_{\hat{s}'}}{\partial p_{sas'}} + \pi_{sa}\left(c_{sas'}+\gamma v^{\bm{\pi},\bm{p}}_{s'}\right)\right)\\
     &= \gamma p^{\bm{\pi}}_{\hat{s}s}(1)\pi_{sa}\left(c_{sas'}+\gamma v^{\bm{\pi},\bm{p}}_{s'}\right) + \gamma^{2}\sum_{\hat{s}'} p^{\bm{\pi}}_{\hat{s}\hat{s}'}(2)\frac{\partial v^{\bm{\pi},\bm{p}}_{\hat{s}'}}{\partial p_{sas'}}\\
     &= \gamma p^{\bm{\pi}}_{\hat{s}s}(1)\pi_{sa}\left(c_{sas'}+\gamma v^{\bm{\pi},\bm{p}}_{s'}\right) + \gamma^{2} p^{\bm{\pi}}_{\hat{s}s}(2)\pi_{sa}\left(c_{sas'}+\gamma v^{\bm{\pi},\bm{p}}_{s'}\right) + \gamma^{3}\sum_{\hat{s}'} p^{\bm{\pi}}_{\hat{s}\hat{s}'}(3)\frac{\partial v^{\bm{\pi},\bm{p}}_{\hat{s}'}}{\partial p_{sas'}}\\
     &= \cdots\\
     &= \sum_{t=1}^{\infty}\gamma^{t}p^{\bm{\pi}}_{\hat{s}s}(t)\pi_{sa}\left(c_{sas'}+\gamma v^{\bm{\pi},\bm{p}}_{s'}\right) = \sum_{t=0}^{\infty}\gamma^{t}p^{\bm{\pi}}_{\hat{s}s}(t)\pi_{sa}\left(c_{sas'}+\gamma v^{\bm{\pi},\bm{p}}_{s'}\right).
\end{align*}
The last equality is from the initial assumption $\hat{s} \neq s$, \ie, $p^{\bm{\pi}}_{\hat{s}s}(0)=0$, and similarly for the case $\hat{s}= s$ we have,
\begin{equation*}
    \frac{\partial v^{\bm{\pi},\bm{p}}_{\hat{s}}}{\partial p_{sas'}}\Big\vert_{\hat{s}= s} = \sum_{t=0}^{\infty}\gamma^{t}p^{\bm{\pi}}_{ss}(t)\pi_{sa}\left(c_{sas'}+\gamma v^{\bm{\pi},\bm{p}}_{s'}\right).
\end{equation*}
Hence, the partial derivative for transition probability is obtained
\begin{align*}
    \frac{\partial J_{\bm{\rho}}(\bm{\pi},\bm{p})}{\partial p_{sas'}}  &= \frac{1}{1-\gamma}\left(\underbrace{(1-\gamma) \sum_{\hat{s}\in\mathcal{S}}\sum_{t=0}^{\infty} \gamma^{t} \rho_{\hat{s}}p^{\bm{\pi}}_{\hat{s}s}(t)}_{d_{\bm{\rho}}^{\bm{\pi},\bm{p}}(s)}\right)\pi_{sa}\left(c_{sas'}+\gamma v^{\bm{\pi},\bm{p}}_{s'}\right) \\
    &= \frac{1}{1-\gamma}d_{\bm{\rho}}^{\bm{\pi},\bm{p}}(s)\pi_{sa}\left(c_{sas'}+\gamma v^{\bm{\pi},\bm{p}}_{s'}\right).
\end{align*}
The uniformly bounded cost $c_{sas'}$ implies that, the absolute value of the value function is bounded for any policy $\bm{\pi}$ and transition kernel $\bm{p}$,
\begin{align*}
    \left| v^{\bm{\pi},\bm{p}}_{s}\right| = \left|\mathbb{E}_{\bm{\pi},\bm{p}}\left[\sum_{t=0}^{\infty} \gamma^{t} c_{s_{t}a_{t}s_{t+1}} \mid s_{0}=s \right]\right| \leq \sum_{t=0}^{\infty}\gamma^{t} = \frac{1}{1-\gamma},
\end{align*}
then we obtain that
\begin{equation*}
    \left|\pi_{sa}\left(c_{sas'}+\gamma v^{\bm{\pi},\bm{p}}_{s'}\right)\right| \leq \left|\pi_{sa}\right|\cdot\left|c_{sas'}+\gamma v^{\bm{\pi},\bm{p}}_{s'}\right| \leq 1 + \frac{\gamma1}{1-\gamma}\leq\frac{1}{1-\gamma}.
\end{equation*}
Therefore, by vectorizing the $\bm{p}$ as a $S^{2}A$-dimensional vector, we have
\begin{align*}
    \|\nabla_{\bm{p}}J_{\bm{\rho}}(\bm{\pi},\bm{p})\| &= \sqrt{\sum_{s,a,s'}\left(\frac{\partial J_{\bm{\rho}}(\bm{\pi},\bm{p})}{\partial p_{sas'}}\right)^{2}}\\
    &= \frac{1}{1-\gamma}\sqrt{\sum_{s,a,s'}\left[d_{\bm{\rho}}^{\bm{\pi},\bm{p}}(s)\pi_{sa}\left(c_{sas'}+\gamma v^{\bm{\pi},\bm{p}}_{s'}\right)\right]^{2}}\\
    &\leq \frac{1}{(1-\gamma)^{2}}\sqrt{\sum_{a,s'}\sum_{s}(d_{\bm{\rho}}^{\bm{\pi},\bm{\xi}}(s))^{2}} \leq \frac{\sqrt{SA}}{(1-\gamma)^{2}},
\end{align*}
where the last inequality holds since the discounted state occupancy measure satisfies
\begin{equation*}
    \sum_{s}(d_{\bm{\rho}}^{\bm{\pi},\bm{\xi}}(s))^{2} \leq \left(\sum_{s}(d_{\bm{\rho}}^{\bm{\pi},\bm{\xi}}(s))\right)^{2}=1.
\end{equation*}
\end{proof}

Notice that, the objective function $J_{\bm{\rho}}(\bm{\pi},\bm{p})$ is twice differentiable on $\bm{p}$. Hence, to prove the smoothness condition in Lemma~\ref{lem_sec4_2} is equal to show that there exists a constant $L\leq\infty$ such that
\begin{equation*}
    \nabla^{2}_{\bm{p}}J_{\bm{\rho}}(\bm{\pi},\bm{p}) \preceq L\bm{I} \;\;\Longleftrightarrow \;\;\forall \bm{x}\in\mathbb{R}^{AS^{2}},\; \bm{x}^{\top}\nabla^{2}_{\bm{p}}J_{\bm{\rho}}(\bm{\pi},\bm{p})\bm{x} \leq L \bm{x}^{\top}\bm{x}.
\end{equation*}
\begin{proof}[Proof of \cref{lem_sec4_2}]
Denote $\bm{p}(\alpha) := \bm{p} + \alpha\bm{z} \in\mathcal{P}$ where $\alpha\in\mathbb{R}$ is a small scalar, whereas $\bm{z}\in(\mathbb{R}^{\mathcal{S}})^{\mathcal{S\times\mathcal{A}}}$. Since $J_{\bm{\rho}}(\bm{\pi},\bm{p}) = \sum_{s}\rho_{s}v^{\bm{\pi},\bm{p}(\alpha)}_{s}$ with a known initial distribution $\bm{\rho}$, we turn to consider the derivative of value function $v^{\bm{\pi},\bm{p}(\alpha)}_{s}$ of the transition kernel $\bm{p}(\alpha)$ over $\alpha$,
\begin{equation}\label{le_valuef}
    v^{\bm{\pi},\bm{p}(\alpha)}_{s} = \sum_{a}\pi_{sa}\sum_{s'}[\bm{p}(\alpha)]_{sas'}c_{sas'}+\gamma\cdot\sum_{a}\pi_{sa}\sum_{s'}[\bm{p}(\alpha)]_{sas'} v^{\bm{\pi},\bm{p}(\alpha)}_{s'},
\end{equation}
First, let us simplify the form of $v^{\bm{\pi},\bm{p}(\alpha)}_{s}$. We define $\bm{P}(\alpha)\in(\Delta^{S})^{S}$ as the state transition kernel and for any $s,s'\in\mathcal{S}$,
\begin{equation}\label{le_def_P:eq2}
    [\bm{P}(\alpha)]_{ss'} = \sum_{a}\pi_{sa}[\bm{p}(\alpha)]_{sas'},
\end{equation}
and $\bm{c}(\alpha)\in\mathbb{R}^{S}$ where for any $s\in\mathcal{S}$,
\begin{equation}
    \left|[\bm{c}(\alpha)]_{s}\right| = \left|\sum_{a}\pi_{sa}\sum_{s'}[\bm{p}(\alpha)]_{sas'}c_{sas'}\right| \leq 1.
\end{equation}
Then, the value function (\ref{le_valuef}) can be written as,
\begin{align}\label{le_valuef2}
    v^{\bm{\pi},\bm{p}(\alpha)}_{s} &= \bm{e}^{\top}_{s}\underbrace{\left(\bm{I}- \gamma \bm{P}(\alpha)\right)^{-1}}_{\bm{M}(\alpha)}\bm{c}(\alpha),
    % v^{\bm{\pi},\bm{p}(\alpha)}_{s} &= \bm{e}^{\top}_{s}\left(I - \gamma \bm{P}(\alpha)\right)^{-1}\bm{c}(\alpha)
\end{align}
where $\bm{e}_{s}:= \left[0,\cdots,1,\cdots,0\right]^{\top}\in\mathbb{R}^{S}$ is a vector whose $s$-th element is $1$ and others are $0$. By using power series expansion technique~\citep{agarwal2021theory,mei2020global}, we can obtain that,
\begin{equation}
    \bm{M}(\alpha) = \left(\bm{I}- \gamma \bm{P}(\alpha)\right)^{-1} = \sum_{t}^{\infty}\gamma^{t}\bm{P}(\alpha)^{t},
\end{equation}
which implies that, for any $s,s'\in\mathcal{S}$, $[\bm{M}(\alpha)]_{ss'}\geq0$, and we have
\begin{equation}
    \bm{e} = \frac{1}{1-\gamma}\cdot\left(\bm{I}- \gamma \bm{P}(\alpha)\right)\bm{e} \Longleftrightarrow\bm{M}(\alpha)\bm{e} = \frac{1}{1-\gamma}\cdot\bm{e},
\end{equation}
which implies each row of $\bm{M}(\alpha)$ sums to $1/(1-\gamma)$. Therefore, for any vector $x\in\mathbb{R}^{S}$, we have
\begin{equation}\label{le_M:bound}
    \left\|\bm{M}(\alpha)\bm{x}\right\|_{\infty} = \max_{i}\left|\left[\bm{M}(\alpha)\bm{x}\right]_{i}\right| \leq \frac{1}{1-\gamma}\cdot\|\bm{x}\|_{\infty}.
\end{equation}
Taking derivative with respect to $\alpha$ on $v^{\bm{\pi},\bm{p}(\alpha)}_{s}$ defined in (\ref{le_valuef2}),
\begin{equation}
    \frac{\partial v^{\bm{\pi},\bm{p}(\alpha)}_{s}}{\partial\alpha} = \bm{e}^{\top}_{s}\bm{M}(\alpha)\frac{\partial\bm{c}(\alpha)}{\partial\alpha} + \gamma\bm{e}^{\top}_{s}\bm{M}(\alpha)\frac{\partial\bm{P}(\alpha)}{\partial\alpha}\bm{M}(\alpha)\bm{c}(\alpha).
\end{equation}
Then taking the twice derivative with respect to $\alpha$,
\begin{align}\label{le_p:eq1}
    \frac{\partial^{2} v^{\bm{\pi},\bm{p}(\alpha)}_{s}}{(\partial\alpha)^{2}} &= \bm{e}^{\top}_{s}\bm{M}(\alpha)\frac{\partial^{2}\bm{c}(\alpha)}{(\partial\alpha)^{2}} + 2\gamma\bm{e}^{\top}_{s}\bm{M}(\alpha)\frac{\partial\bm{P}(\alpha)}{\partial\alpha}\bm{M}(\alpha)\frac{\partial\bm{c}(\alpha)}{\partial\alpha}\notag\\
    &\;\;\;\;+ 2\gamma^{2}\bm{e}^{\top}_{s}\bm{M}(\alpha)\frac{\partial\bm{P}(\alpha)}{\partial\alpha}\bm{M}\frac{\partial\bm{P}(\alpha)}{\partial\alpha}\bm{M}(\alpha)\bm{c}(\alpha) + \gamma\bm{e}^{\top}_{s}\bm{M}(\alpha)\frac{\partial^{2}\bm{P}(\alpha)}{(\partial\alpha)^{2}}\bm{M}(\alpha)\bm{c}(\alpha).
\end{align}
Notice that, above two form of derivatives are obtained by using matrix calculus techniques, \ie, for any matrix $\bm{A},\bm{B}, \bm{U}(x)$ and scalar $x$,
\begin{align*}
    \frac{\partial \bm{A}\bm{U}(x)\bm{B}}{\partial x} &= \bm{A}\frac{\partial\bm{U}(x)}{\partial x}\bm{B}\quad\text{and}\quad\frac{\partial \bm{U}(x)^{-1}}{\partial x} = - \bm{U}(x)^{-1}\frac{\partial\bm{U}(x)}{\partial x}\bm{U}(x)^{-1}.
\end{align*}

So far, we get the derivative form of the value function. Then we'd like to bound $\left|\frac{\partial^{2} v^{\bm{\pi},\bm{p}(\alpha)}_{s}}{(\partial\alpha)^{2}}\Big\vert_{\alpha=0}\right|$.

For the first term in (\ref{le_p:eq1}), we have,
\begin{align}\label{le_term1:eq1}
    \left|\bm{e}^{\top}_{s}\bm{M}(\alpha)\frac{\partial^{2}\bm{c}(\alpha)}{(\partial\alpha)^{2}}\Big\vert_{\alpha=0}\right| &\leq \left\|\bm{e}^{\top}_{s}\right\|_{1}\cdot\left\|\bm{M}(\alpha)\frac{\partial^{2}\bm{c}(\alpha)}{(\partial\alpha)^{2}}\Big\vert_{\alpha=0}\right\|_{\infty}\notag\\
    &\leq \frac{1}{1-\gamma}\cdot\left\|\frac{\partial^{2}\bm{c}(\alpha)}{(\partial\alpha)^{2}}\Big\vert_{\alpha=0}\right\|_{\infty}\notag\\
    &= 0,
\end{align}
where the last but one inequality is obtained from (\ref{le_M:bound}) and the last equality holds since for any $\alpha\in\mathbb{R}$,
\begin{align}
    \left\|\frac{\partial^{2}\bm{c}(\alpha)}{(\partial\alpha)^{2}}\right\|_{\infty} &= \max_{s}\left|\frac{\partial}{\partial\alpha}\left(\frac{\partial[\bm{c}(\alpha)]_{s}}{\partial\alpha}\right)\right|\notag\\
    & = \max_{s}\left|\frac{\partial}{\partial\alpha}\left(\frac{\partial\left(\sum_{a}\pi_{sa}\sum_{s'}[\bm{p}(\alpha)]_{sas'}c_{sas'}\right)}{\partial\alpha}\right)\right|\notag\\
    &= \max_{s}\left|\frac{\partial}{\partial\alpha}\left(\sum_{a}\pi_{sa}\sum_{s'}z_{sas'}c_{sas'}\right)\right|\notag\\
    &= 0.
\end{align}
For the second term in (\ref{le_p:eq1}), we have
\begin{align}
\left|\bm{e}^{\top}_{s}\bm{M}(\alpha)\frac{\partial\bm{P}(\alpha)}{\partial\alpha}\bm{M}(\alpha)\frac{\partial\bm{c}(\alpha)}{\partial\alpha}\Big\vert_{\alpha=0}\right| &\leq \left\|\bm{e}^{\top}_{s}\right\|_{1}\cdot\left\|\bm{M}(\alpha)\frac{\partial\bm{P}(\alpha)}{\partial\alpha}\bm{M}(\alpha)\frac{\partial\bm{c}(\alpha)}{\partial\alpha}\Big\vert_{\alpha=0}\right\|_{\infty}\notag\\
&\leq \frac{1}{1-\gamma}\cdot \left\|\frac{\partial\bm{P}(\alpha)}{\partial\alpha}\bm{M}(\alpha)\frac{\partial\bm{c}(\alpha)}{\partial\alpha}\Big\vert_{\alpha=0}\right\|_{\infty}.
\end{align}
According to (\ref{le_def_P:eq2}), for any $\bm{x}\in\mathbb{R}^{S}$ and $s\in\mathcal{S}$, we have,
\begin{align*}
    \left[\frac{\partial\bm{P}(\alpha)}{\partial\alpha}\bm{x}\right]_{s} = \sum_{s'}\sum_{a}\pi_{sa}\frac{\partial [\bm{p}(\alpha)]_{sas'}}{\partial\alpha}x_{s'},
\end{align*}
and its $\ell_{\infty}$ norm can be upper bounded as
\begin{align}
    \left\|\frac{\partial\bm{P}(\alpha)}{\partial\alpha}\Big\vert_{\alpha=0}\bm{x}\right\|_{\infty} &= \max_{s}\left|\sum_{s'}\sum_{a}\pi_{sa}\frac{\partial [\bm{p}(\alpha)]_{sas'}}{\partial\alpha}\Big\vert_{\alpha=0}x_{s'}\right|\notag\\
    &\leq \max_{s}\sum_{s'}\sum_{a}\pi_{sa}\left|z_{sas'}\right| \left|x_{s'}\right|\notag\\
    &\leq \max_{s}\sum_{s'}\sum_{a}\pi_{sa}\left|z_{sas'}\right|\cdot\|x\|_{\infty}\notag\\
    &= \sum_{s'}\sum_{a}\pi_{\bar{s}a}\left|z_{\bar{s}as'}\right|\cdot\|x\|_{\infty}\notag\\
    &\leq \sum_{s'}\sum_{a}\pi_{\bar{s}a} \max_{s,a,s'}\left|z_{sas'}\right|\cdot\|x\|_{\infty}\notag\\
    &=\max_{s,a,s'}\left|z_{sas'}\right|\cdot\sum_{s'} \|x\|_{\infty}\notag\\
    &\leq S\cdot\|\bm{z}\|_{\infty}\cdot\|x\|_{\infty}\notag\\
    &\leq S\cdot\|\bm{z}\|_{2}\cdot\|x\|_{\infty}
\end{align}
Similarly, for any $\alpha\in\mathbb{R}$, we have
\begin{align}
    \left\|\frac{\partial\bm{c}(\alpha)}{\partial\alpha}\right\|_{\infty} &= \max_{s}\left|\frac{\partial\left(\sum_{a}\pi_{sa}\sum_{s'}[\bm{p}(\alpha)]_{sas'}c_{sas'}\right)}{\partial\alpha}\right| \notag\\
    &= \max_{s}\left|\sum_{a}\pi_{sa}\sum_{s'}z_{sas'}c_{sas'}\right|\notag\\
    &\leq S\cdot\|\bm{z}\|_{2}.
\end{align}
Then, we obtain an upper bound of the second term,
\begin{align}\label{le_term2:eq1}
    \left|\bm{e}^{\top}_{s}\bm{M}(\alpha)\frac{\partial\bm{P}(\alpha)}{\partial\alpha}\bm{M}(\alpha)\frac{\partial\bm{c}(\alpha)}{\partial\alpha}\Big\vert_{\alpha=0}\right| &\leq \frac{S}{1-\gamma}\cdot\left\|\bm{M}(\alpha)\frac{\partial\bm{c}(\alpha)}{\partial\alpha}\Big\vert_{\alpha=0}\right\|_{\infty}\cdot\|\bm{z}\|_{2}\notag\\
    &\leq \frac{S}{(1-\gamma)^{2}}\cdot\left\|\frac{\partial\bm{c}(\alpha)}{\partial\alpha}\Big\vert_{\alpha=0}\right\|_{\infty}\cdot\|\bm{z}\|_{2}\notag\\
    &\leq \frac{S^{2}}{(1-\gamma)^{2}}\cdot\|\bm{z}\|^{2}_{2}.
\end{align}
For the third term of in (\ref{le_p:eq1}), we can similarly bound it as
\begin{align}\label{le_term3:eq1}
    \left|\bm{e}^{\top}_{s}\bm{M}(\alpha)\frac{\partial\bm{P}(\alpha)}{\partial\alpha}\bm{M}(\alpha)\frac{\partial\bm{P}(\alpha)}{\partial\alpha}\bm{M}(\alpha)\bm{c}(\alpha)\Big\vert_{\alpha=0}\right| &\leq \left\|\bm{M}(\alpha)\frac{\partial\bm{P}(\alpha)}{\partial\alpha}\bm{M}(\alpha)\frac{\partial\bm{P}(\alpha)}{\partial\alpha}\bm{M}(\alpha)\bm{c}(\alpha)\Big\vert_{\alpha=0}\right\|_{\infty}\notag\\
    &\leq \frac{1}{1-\gamma}\cdot S\cdot \|\bm{z}\|_{2}\cdot\frac{1}{1-\gamma}\cdot S\cdot \|\bm{z}\|_{2}\cdot\frac{1}{1-\gamma}\notag\\
    &= \frac{S^{2}}{(1-\gamma)^{3}}\cdot\|\bm{z}\|^{2}_{2}.
\end{align}
Denote that, for any $\bm{x}\in\mathbb{R}^{S}$,
\begin{align}\label{le_term4:eq1}
    \left\|\frac{\partial^{2}\bm{P}(\alpha)}{(\partial\alpha)^{2}}\Big\vert_{\alpha=0}\bm{x}\right\|_{\infty} &= \max_{s}\left|\sum_{s'}\sum_{a}\pi_{sa}\frac{\partial^{2} [\bm{p}(\alpha)]_{sas'}}{\partial(\alpha)^{2}}\Big\vert_{\alpha=0}x_{s'}\right|=0.
\end{align}
Therefore, we combine (\ref{le_term1:eq1}), (\ref{le_term2:eq1}), (\ref{le_term3:eq1}) and (\ref{le_term4:eq1}),
\begin{align}
    \left|\frac{\partial^{2} v^{\bm{\pi},\bm{p}(\alpha)}_{s}}{(\partial\alpha)^{2}}\Big\vert_{\alpha=0}\right| &= \left|\bm{e}^{\top}_{s}\bm{M}(\alpha)\frac{\partial^{2}\bm{c}(\alpha)}{(\partial\alpha)^{2}}\Big\vert_{\alpha=0}\right| + 2\gamma^{2}\cdot\left|\bm{e}^{\top}_{s}\bm{M}(\alpha)\frac{\partial\bm{P}(\alpha)}{\partial\alpha}\bm{M}(\alpha)\frac{\partial\bm{P}(\alpha)}{\partial\alpha}\bm{M}(\alpha)\bm{c}(\alpha)\Big\vert_{\alpha=0}\right|\notag\\
    &\;\;\;\;+ 2\gamma\cdot\left|\bm{e}^{\top}_{s}\bm{M}(\alpha)\frac{\partial\bm{P}(\alpha)}{\partial\alpha}\bm{M}(\alpha)\frac{\partial\bm{c}(\alpha)}{\partial\alpha}\Big\vert_{\alpha=0}\right| + \gamma\cdot\left|\bm{e}^{\top}_{s}\bm{M}(\alpha)\frac{\partial^{2}\bm{P}(\alpha)}{(\partial\alpha)^{2}}\bm{M}(\alpha)\bm{c}(\alpha)\Big\vert_{\alpha=0}\right|\notag\\
    &\leq 2\gamma\cdot\frac{S^{2}}{(1-\gamma)^{2}}\cdot\|\bm{z}\|^{2}_{2} + 2\gamma^{2}\cdot\frac{S^{2}}{(1-\gamma)^{3}}\|\bm{z}\|^{2}_{2}\notag\\
    &= \frac{2\gamma S^{2}}{(1-\gamma)^{3}}\cdot\|\bm{z}\|^{2}_{2}.
\end{align}
Then, for any $\bm{y}\in\mathbb{R}^{AS^{2}}$, we have
\begin{align}
    \left|\bm{y}^{\top}\nabla_{\bm{p}}^{2}J_{\rho}(\bm{\pi},\bm{p})\bm{y}\right| &\leq \sum_{s}\rho_{s}\cdot \left|\bm{y}^{\top}\frac{\partial^{2} v^{\bm{\pi},\bm{p}}_{s}}{(\partial\bm{p})^{2}}\bm{y}\right|\notag\\
    &= \sum_{s}\rho_{s}\cdot \left|(\frac{\bm{y}}{\|\bm{y}\|_{2}})^{\top}\frac{\partial^{2} v^{\bm{\pi},\bm{p}}_{s}}{(\partial\bm{p})^{2}}(\frac{\bm{y}}{\|\bm{y}\|_{2}})\right|\cdot\|\bm{y}\|^{2}_{2}\notag\\
    &\leq \sum_{s}\rho_{s}\cdot \max_{\|\bm{z}\|_{2}=1}\left|\left\langle \frac{\partial^{2} v^{\bm{\pi},\bm{p}}_{s}}{(\partial\bm{p})^{2}}\bm{z}, \bm{z}\right\rangle\right|\cdot\|\bm{y}\|^{2}_{2}\notag\\
    &= \sum_{s}\rho_{s}\cdot \max_{\|\bm{z}\|_{2}=1}\left|\left\langle \frac{\partial^{2} v^{\bm{\pi},\bm{p}(\alpha)}_{s}}{(\partial\bm{p}(\alpha))^{2}}\Big\vert_{\alpha=0}\frac{\partial \bm{p}(\alpha)}{\partial\alpha}, \frac{\partial \bm{p}(\alpha)}{\partial\alpha}\right\rangle\right|\cdot\|\bm{y}\|^{2}_{2}\notag\\
    &= \sum_{s}\rho_{s}\cdot \max_{\|\bm{z}\|_{2}=1}\left|\frac{\partial^{2} v^{\bm{\pi},\bm{p}(\alpha)}_{s}}{(\partial\alpha)^{2}}\Big\vert_{\alpha=0}\right|\cdot\|\bm{y}\|^{2}_{2}\notag\\
    &\leq \frac{2\gamma S^{2}}{(1-\gamma)^{3}}\cdot\|\bm{y}\|^{2}_{2}.
\end{align}
\end{proof}

\begin{proof}[Proof of \cref{the_sec4_GD}]
By the definition of $J_{\bm{\rho}}(\bm{\pi},\bm{p})$, we have
\begin{equation*}
    J_{\bm{\rho}}(\bm{\pi},\bm{p})-J_{\bm{\rho}}(\bm{\pi},\bm{p}') = \sum_{s}\rho_{s}\left( v^{\bm{\pi},\bm{p}}_{s}-v^{\bm{\pi},\bm{p}'}_{s}\right).
\end{equation*}
For any $s\in\mathcal{S}$ and $\bm{p},\bm{p}' \in\mathcal{P}$, we have
\begin{align*} 
&v^{\bm{\pi},\bm{p}}_{s}-v^{\bm{\pi},\bm{p}'}_{s} \\
&= v^{\bm{\pi},\bm{p}}_{s}
-  \sum_{a}\pi_{sa}\sum_{s'}p'_{sas'}\left(c_{sas'}+\gamma v^{\bm{\pi},\bm{p}}_{s'}\right)
+\sum_{a}\pi_{sa}\sum_{s'}p'_{sas'}\left(c_{sas'}+\gamma v^{\bm{\pi},\bm{p}}_{s'}\right)-
v^{\bm{\pi},\bm{p}'}_{s}\\
&=\sum_{a}\pi_{sa}\sum_{s'}p_{sas'}\left(c_{sas'}+\gamma v^{\bm{\pi},\bm{p}}_{s'}\right)-  \sum_{a}\pi_{sa}\sum_{s'}p'_{sas'}\left(c_{sas'}+\gamma v^{\bm{\pi},\bm{p}}_{s'}\right)\\
&\;\;\;\;+\sum_{a}\pi_{sa}\sum_{s'}p'_{sas'}\left(c_{sas'}+\gamma v^{\bm{\pi},\bm{p}}_{s'}\right)-
\sum_{a}\pi_{sa}\sum_{s'}p'_{sas'}\left(c_{sas'}+\gamma v^{\bm{\pi},\bm{p}'}_{s'}\right)\\
&= \sum_{a}\pi_{sa}\sum_{s'}\left(p_{sas'}-p'_{sas'}\right)\left(c_{sas'}+\gamma v^{\bm{\pi},\bm{p}}_{s'}\right) + \gamma\sum_{a}\pi_{sa}\sum_{s'}p'_{sas'}\left(v^{\bm{\pi},\bm{p}}_{s'}-v^{\bm{\pi},\bm{p}'}_{s'}\right)\\
&= \cdots\\
&= \sum_{t=0}^{\infty}\gamma^{t}\sum_{s'}p'^{\bm{\pi}}_{ss'}(t)\left(\sum_{a'}\pi_{s'a'}\sum_{s''}\left(p_{s'a's''}-p'_{s'a's''}\right)\left(c_{s'a's''}+\gamma v^{\bm{\pi},\bm{p}}_{s''}\right)\right).
\end{align*}
Here, the last equation is obtained by the recursion and we then obtain 
\begin{align*}
    J_{\bm{\rho}}(\bm{\pi},\bm{p})-J_{\bm{\rho}}(\bm{\pi},\bm{p}') &= \sum_{s}\rho_{s}\left( v^{\bm{\pi},\bm{p}}_{s}-v^{\bm{\pi},\bm{p}'}_{s}\right)\\
    &= \sum_{s}\rho_{s}\sum_{t=0}^{\infty}\gamma^{t}\sum_{s'}p'^{\bm{\pi}}_{ss'}(t)\left(\sum_{a'}\pi_{s'a'}\sum_{s''}\left(p_{s'a's''}-p'_{s'a's''}\right)\left(c_{s'a's''}+\gamma v^{\bm{\pi},\bm{p}}_{s''}\right)\right)\\
    &= \sum_{s'}\left(\sum_{s}\sum_{t=0}^{\infty}\gamma^{t}\rho_{s}p'^{\bm{\pi}}_{ss'}(t)\right)\left(\sum_{a'}\pi_{s'a'}\sum_{s''}\left(p_{s'a's''}-p'_{s'a's''}\right)\left(c_{s'a's''}+\gamma v^{\bm{\pi},\bm{p}}_{s''}\right)\right)\\
    &= \frac{1}{1-\gamma} \sum_{s} d_{\bm{\rho}}^{\bm{\pi},\bm{p}'}(s) \left(\sum_{a}\pi_{sa}\sum_{s'}\left(p_{sas'}-p'_{sas'}\right)\left(c_{sas'}+\gamma v^{\bm{\pi},\bm{p}}_{s'}\right)\right).
\end{align*}
Let $\bm{p}' = \bm{p}^{\star}$ and then, we have
\begin{align}
0\leq J_{\bm{\rho}}(\bm{\pi},\bm{p}^{\star})-J_{\bm{\rho}}(\bm{\pi},\bm{p})
    &= \frac{1}{1-\gamma} \sum_{s} d_{\bm{\rho}}^{\bm{\pi},\bm{p}^{\star}}(s) \left(\sum_{a}\pi_{sa}\sum_{s'}\left(p^{\star}_{sas'}-p_{sas'}\right)\left(c_{sas'}+\gamma v^{\bm{\pi},\bm{p}}_{s'}\right)\right)\notag\\
    & = \frac{1}{1-\gamma} \sum_{s}\frac{d_{\bm{\rho}}^{\bm{\pi},\bm{p}^{\star}}(s)}{d_{\bm{\rho}}^{\bm{\pi},\bm{p}}(s)}\cdot d_{\bm{\rho}}^{\bm{\pi},\bm{p}}(s) \left(\sum_{a}\pi_{sa}\sum_{s'}\left(p^{\star}_{sas'}-p_{sas'}\right)\left(c_{sas'}+\gamma v^{\bm{\pi},\bm{p}}_{s'}\right)\right)\notag\\
    &\overset{(a)}{\leq} \frac{1}{1-\gamma}\cdot \left\|\frac{d_{\bm{\rho}}^{\bm{\pi},\bm{p}^{\star}}}{d_{\bm{\rho}}^{\bm{\pi},\bm{p}}}\right\|_{\infty}\cdot \sum_{s} d_{\bm{\rho}}^{\bm{\pi},\bm{p}}(s) \left(\sum_{a}\pi_{sa}\sum_{s'}\left(p^{\star}_{sas'}-p_{sas'}\right)\left(c_{sas'}+\gamma v^{\bm{\pi},\bm{p}}_{s'}\right)\right)\notag\\
    &= \left\|\frac{d_{\bm{\rho}}^{\bm{\pi},\bm{p}^{\star}}}{d_{\bm{\rho}}^{\bm{\pi},\bm{p}}}\right\|_{\infty}\cdot \sum_{s,a,s'}\left(\frac{1}{1-\gamma}d_{\bm{\rho}}^{\bm{\pi},\bm{p}}(s)\pi_{sa}\left(c_{sas'}+\gamma v^{\bm{\pi},\bm{p}}_{s'}\right)\right)\cdot\left(p^{\star}_{sas'}-p_{sas'}\right)\notag\\
    &\leq \left\|\frac{d_{\bm{\rho}}^{\bm{\pi},\bm{p}^{\star}}}{d_{\bm{\rho}}^{\bm{\pi},\bm{p}}}\right\|_{\infty}\cdot\max_{\bar{\bm{p}}\in\mathcal{P}} \left[\sum_{s,a,s'}\left(\frac{1}{1-\gamma}d_{\bm{\rho}}^{\bm{\pi},\bm{p}}(s)\pi_{sa}\left(c_{sas'}+\gamma v^{\bm{\pi},\bm{p}}_{s'}\right)\right)\cdot(\bar{p}_{sas'}-p_{sas'})\right]\notag\\
    &= \left\|\frac{d_{\bm{\rho}}^{\bm{\pi},\bm{p}^{\star}}}{d_{\bm{\rho}}^{\bm{\pi},\bm{p}}}\right\|_{\infty}\cdot\max_{\bar{\bm{p}}\in\mathcal{P}}\left\langle\bar{\bm{p}}-\bm{p},\frac{\partial J_{\bm{\rho}}(\bm{\pi},\bm{p})}{\partial \bm{p}}\right\rangle\tag{by Lemma~\ref{lem_sec4_1}}\\
    &\leq \frac{D}{1-\gamma}\max_{\bar{\bm{p}}\in\mathcal{P}} \left\langle\bar{\bm{p}}-\bm{p},\frac{\partial J_{\bm{\rho}}(\bm{\pi},\bm{p})}{\partial \bm{p}} \right\rangle.\notag
\end{align}
which completes the proof. The first inequality $(a)$ is obtained due to the fact that for any $s\in\mathcal{S}$,
\begin{equation*}
    \sum_{a}\pi_{sa}\sum_{s'}\left(p^{\star}_{sas'}-p_{sas'}\right)\left(c_{sas'}+\gamma v^{\bm{\pi},\bm{p}}_{s'}\right) \geq 0
\end{equation*}
holds under the $s$-rectangularity assumption.
\end{proof}

Now, we proceed to prove main theorem in section 4. Here we can define $f_{\bm{\pi}}(\bm{p}):=J_{\bm{\rho}}(\bm{\pi},\bm{p})$ for a fixed policy $\bm{\pi}\in\Pi$ and define the gradient mapping
\begin{equation}
    G^{\beta}(\bm{p}):= \frac{1}{\beta}\left(\text{Proj}_{\mathcal{P}}(\bm{p} + \beta\nabla f_{\bm{\pi}}(\bm{p})) - \bm{p}\right).
\end{equation}
Notice that $\mathcal{P}$ is convex and $f_{\bm{\pi}}(\bm{p})$ is $\ell_{\bm{p}}$-smooth, then the following lemma can be derived directly using existing classic results:
\begin{lemma}\citep[Theorem 10.15]{beck2017first}\label{sec4_lem_gradmap}
Let $\{\bm{p}_{t}\}_{t\geq0}$ be the sequence generated by Algorithm~\ref{alg:PGD_inner} for solving the inner problem with the constant step size $\beta := \frac{1}{\ell_{\bm{p}}}$, then
\begin{equation}
    \min_{t\in\{0,\cdots,T-1\}}\|G^{\beta}(\bm{p}_{t})\| \leq \sqrt{\frac{2\ell_{\bm{p}}\left(f_{\bm{\pi}}^{\star} - f_{\bm{\pi}}(\bm{p}_{0})\right)}{T}}
\end{equation}
\end{lemma}

\begin{proof}[Proof of \cref{the_sec4_global}]
It has been shown in Lemma 3 in~\citep{ghadimi2016accelerated} that if $\|G^{\beta}(\bm{p})\|\leq \epsilon$, then
\begin{equation}\label{sec4_lem_imp1}
    \nabla f_{\bm{\pi}}(\bm{p}^{+}) \in \mathcal{N}_{\mathcal{P}}(\bm{p}^{+}) + 2\epsilon\mathcal{B}(1),
\end{equation}
where $\bm{p}^{+} := \bm{p} + \beta G^{\beta}(\bm{p})$, $\mathcal{N}_{\mathcal{P}}$ is the norm cone of the set $\mathcal{P}$ and $\mathcal{B}(r):= \{\bm{x}\in\mathbb{R}^{n}: \|\bm{x}\|\leq r\}$. By the gradient dominance condition established in Lemma~\ref{the_sec4_GD}, 
\begin{align}\label{sec4_mid2}
    \min_{t\in\{0,\cdots,T-1\}}\left\{ f_{\bm{\pi}}(\bm{p}^{\star}) - f_{\bm{\pi}}(\bm{p}_{t})\right\} &\leq \frac{D}{1-\gamma}\min_{t\in\{0,\cdots,T-1\}}\max_{\bar{\bm{p}}\in\mathcal{P}} \left\langle\bar{\bm{p}}-\bm{p}_{t},\nabla f_{\bm{\pi}}(\bm{p}_{t})\right\rangle \notag\\
    &\leq \frac{D}{1-\gamma}\max_{\bar{\bm{p}}\in\mathcal{P}} \left\langle\bar{\bm{p}}-\bm{p}_{\hat{t}},\nabla f_{\bm{\pi}}(\bm{p}_{\hat{t}})\right\rangle,
\end{align}
where $\hat{t}:=1 + \arg\min_{t\leq T-1} \|G^{\beta}(\bm{p}_{t})\|$. Recall Lemma~\ref{sec4_lem_gradmap}, we showed that
\begin{equation*}
    \|G^{\beta}(\bm{p}_{\hat{t}-1})\| \leq \sqrt{\frac{2\ell_{\bm{p}}\left(f_{\bm{\pi}}^{\star} - f_{\bm{\pi}}(\bm{p}_{0})\right)}{T}} \leq\sqrt{\frac{2\ell_{\bm{p}}}{(1-\gamma)T}},
\end{equation*}
where the last inequality holds due to 
\begin{equation}
     v^{\bm{\pi}}_{s} = \mathbb{E}_{\bm{\pi},\bm{p}}\left[\sum_{t=0}^{\infty} \gamma^{t} c_{s_{t}a_{t}s_{t+1}} \mid s_{0}=s\right] \leq \sum_{t=0}^{\infty}\gamma^{t} = \frac{1}{1-\gamma}.
\end{equation}
If we set that
\begin{align*}
\sqrt{\frac{2\ell_{\bm{p}}}{(1-\gamma)T}} \leq \frac{(1-\gamma)\epsilon}{4D\sqrt{SA}} \Longleftrightarrow T \geq \frac{32\ell_{\bm{p}}D^{2}SA}{(1-\gamma)^{3}\epsilon^{2}} = \mathcal{O}(\epsilon^{-2}),
\end{align*}
then
\begin{equation*}
    \|G^{\beta}(\bm{p}_{\hat{t}-1})\| \leq \frac{(1-\gamma)\epsilon}{4D\sqrt{SA}}.
\end{equation*}
Hence, by applying the equation (\ref{sec4_lem_imp1}), we have
\begin{align*}
    (\ref{sec4_mid2}) \leq \frac{D}{1-\gamma}\max_{\bar{\bm{p}}\in\mathcal{P}}\|\bar{\bm{p}}-\bm{p}_{\hat{t}}\|\cdot2\cdot\frac{(1-\gamma)\epsilon}{4D\sqrt{SA}} = \epsilon,
\end{align*}
where for any $\bm{p}_{1},\bm{p}_{2}\in\mathcal{P}$,
\begin{equation}
    \|\bm{p}_{1}-\bm{p}_{2}\| \leq  \|\bm{p}_{1}\|+ \|\bm{p}_{2}\| \leq2\sqrt{SA}.
\end{equation}
\end{proof}

Then, we provide the standard proof of \cref{the_sec4_parame}.
\begin{proof}[Proof of \cref{the_sec4_parame}]
We first show that the inner problem gradient form. Notice that,
\begin{equation*}
    \frac{\partial J_{\bm{\rho}}(\bm{\pi},\bm{\xi})}{\partial \bm{\xi}} = \sum_{s \in \mathcal{S}} \frac{\partial v^{\bm{\pi},\bm{p}}_{s}}{\partial \bm{\xi}} \rho_{s}.
\end{equation*}
Then we consider the $\frac{\partial v^{\bm{\pi},\bm{p}}_{s}}{\partial \bm{\xi}}$ directly.
\begin{align*}
    \frac{\partial v^{\bm{\pi},\bm{p}}_{s}}{\partial \bm{\xi}} &= \frac{\partial}{\partial \bm{\xi}}\left[\sum_{a}\pi_{sa}Q^{\bm{\pi}_{sa},\bm{\xi}}\right]\\
    &= \sum_{a}\pi_{sa}\frac{\partial}{\partial \bm{\xi}}\left[\sum_{s'}p^{\bm{\xi}}_{sas'}\left(c_{sas'}+\gamma v^{\bm{\pi},\bm{\xi}}_{s'}\right)\right]\\
    &=\sum_{a}\pi_{sa}\sum_{s'}\left[\frac{\partial p^{\bm{\xi}}_{sas'}}{\partial \bm{\xi}}\left(c_{sas'}+\gamma v^{\bm{\pi},\bm{\xi}}_{s'}\right)+ \gamma p^{\bm{\xi}}_{sas'}\frac{\partial v^{\bm{\pi},\bm{\xi}}_{s'}}{\partial \bm{\xi}}\right]\\
    &=\sum_{a}\pi_{sa}\sum_{s'}\frac{\partial p^{\bm{\xi}}_{sas'}}{\partial \bm{\xi}}\left(c_{sas'}+\gamma v^{\bm{\pi},\bm{\xi}}_{s'}\right) + \gamma\sum_{a}\pi_{sa}\sum_{s'}p^{\bm{\xi}}_{sas'}\frac{\partial v^{\bm{\pi},\bm{\xi}}_{s'}}{\partial \bm{\xi}}.
\end{align*}
By condensing $\sum_{a}\pi_{sa}p^{\bm{\xi}}_{sas'}=p^{\bm{\pi},\bm{\xi}}_{ss'}(1)$, we can obtain,
\begin{align*}
     \frac{\partial v^{\bm{\pi},\bm{p}}_{s}}{\partial \bm{\xi}} &= \sum_{a}\pi_{sa}\sum_{s'}\frac{\partial p^{\bm{\xi}}_{sas'}}{\partial \bm{\xi}}\left(c_{sas'}+\gamma v^{\bm{\pi},\bm{\xi}}_{s'}\right) + \gamma\sum_{s'}p^{\bm{\pi},\bm{\xi}}_{ss'}(1)\frac{\partial v^{\bm{\pi},\bm{\xi}}_{s'}}{\partial \bm{\xi}}\\
     &= \sum_{a}\pi_{sa}\sum_{s'}\frac{\partial p^{\bm{\xi}}_{sas'}}{\partial \bm{\xi}}\left(c_{sas'}+\gamma v^{\bm{\pi},\bm{\xi}}_{s'}\right)\\
     &\;\;\;\;+ \gamma\sum_{s'}p^{\bm{\pi},\bm{\xi}}_{ss'}(1)\left[\sum_{a'}\pi_{s'a'}\sum_{s''}\frac{\partial p^{\bm{\xi}}_{s'a's''}}{\partial \bm{\xi}}\left(c_{s'a's''}+\gamma v^{\bm{\pi},\bm{\xi}}_{s''}\right)+\gamma\sum_{s''}p^{\bm{\pi},\bm{\xi}}_{s's''}(1)\frac{\partial v^{\bm{\pi},\bm{\xi}}_{s''}}{\partial \bm{\xi}}\right]\\
     &=\sum^{1}_{k=0}\gamma^{k}\sum_{s'}p^{\bm{\pi},\bm{\xi}}_{ss'}(k)\sum_{a'}\pi_{s'a'}\left[\sum_{s''}\frac{\partial p^{\bm{\xi}}_{s'a's''}}{\partial \bm{\xi}}\left(c_{s'a's''}+\gamma v^{\bm{\pi},\bm{\xi}}_{s''}\right)\right] + \gamma^{2}\sum_{s'}p^{\bm{\pi},\bm{\xi}}_{ss'}(2)\frac{\partial v^{\bm{\pi},\bm{\xi}}_{s'}}{\partial \bm{\xi}}\\
     &=\sum^{2}_{k=0}\gamma^{k}\sum_{s'}p^{\bm{\pi},\bm{\xi}}_{ss'}(k)\sum_{a'}\pi_{s'a'}\left[\sum_{s''}\frac{\partial p^{\bm{\xi}}_{s'a's''}}{\partial \bm{\xi}}\left(c_{s'a's''}+\gamma v^{\bm{\pi},\bm{\xi}}_{s''}\right)\right] + \gamma^{3}\sum_{s'}p^{\bm{\pi},\bm{\xi}}_{ss'}(3)\frac{\partial v^{\bm{\pi},\bm{\xi}}_{s'}}{\partial \bm{\xi}}\\
     &= \cdots\\
     &= \sum^{\infty}_{k=0}\gamma^{k}\sum_{s'}p^{\bm{\pi},\bm{\xi}}_{ss'}(k)\sum_{a'}\pi_{s'a'}\left[\sum_{s''}\frac{\partial p^{\bm{\xi}}_{s'a's''}}{\partial \bm{\xi}}\left(c_{s'a's''}+\gamma v^{\bm{\pi},\bm{\xi}}_{s''}\right)\right].
\end{align*}
So we have 
\begin{align*}
    \frac{\partial J_{\bm{\rho}}(\bm{\pi},\bm{\xi})}{\partial \bm{\xi}} &= \sum_{s \in \mathcal{S}} \frac{\partial v^{\bm{\pi},\bm{\xi}}_{s}}{\partial \bm{\xi}} \rho_{s}\\
    &= \sum_{s}\rho_{s}\sum^{\infty}_{k=0}\gamma^{k}\sum_{s'}p^{\bm{\pi},\bm{\xi}}_{ss'}(k)\sum_{a'}\pi_{s'a'}\left[\sum_{s''}\frac{\partial p^{\bm{\xi}}_{s'a's''}}{\partial \bm{\xi}}\left(c_{s'a's''}+\gamma v^{\bm{\pi},\bm{\xi}}_{s''}\right)\right]\\
    &= \frac{1}{1-\gamma}\sum_{s'}\underbrace{(1-\gamma)\sum_{s}\rho_{s}\sum^{\infty}_{k=0}\gamma^{k}p^{\bm{\pi},\bm{\xi}}_{ss'}(k)}_{d^{\bm{\pi},\bm{\xi}}_{s'}}\sum_{a'}\pi_{s'a'}\left[\sum_{s''}\frac{\partial p^{\bm{\xi}}_{s'a's''}}{\partial \bm{\xi}}\left(c_{s'a's''}+\gamma v^{\bm{\pi},\bm{\xi}}_{s''}\right)\right]\\
    &= \frac{1}{1-\gamma}\sum_{s}d^{\bm{\pi},\bm{\xi}}_{s}\sum_{a}\pi_{sa}\left[\sum_{s'}\frac{\partial p^{\bm{\xi}}_{sas'}}{\partial \bm{\xi}}\left(c_{sas'}+\gamma v^{\bm{\pi},\bm{\xi}}_{s'}\right)\right]\\
    &= \frac{1}{1-\gamma}\sum_{s}d^{\bm{\pi},\bm{\xi}}_{s}\sum_{a}\pi_{sa}\sum_{s'}p^{\bm{\xi}}_{sas'}\left[\frac{\partial p^{\bm{\xi}}_{sas'}}{\partial \bm{\xi}}\cdot\frac{1}{p^{\bm{\xi}}_{sas'}}\cdot\left(c_{sas'}+\gamma v^{\bm{\pi},\bm{\xi}}_{s'}\right)\right]\\
    &= \frac{1}{1-\gamma}\sum_{s}d^{\bm{\pi},\bm{\xi}}_{s}\sum_{a}\pi_{sa}\sum_{s'}p^{\bm{\xi}}_{sas'}\left[\frac{\partial \log p^{\bm{\xi}}_{sas'}}{\partial \bm{\xi}}\left(c_{sas'}+\gamma v^{\bm{\pi},\bm{\xi}}_{s'}\right)\right]\\
    &= \frac{1}{1-\gamma}\mathbb{E}_{s\sim\bm{d}^{\bm{\pi},\bm{\xi}}_{\rho}}\mathbb{E}_{a\sim\bm{\pi}_{s\cdot}}\mathbb{E}_{s'\sim\bm{p}_{sa\cdot}}\left[\frac{\partial\log p^{\bm{\xi}}_{sas'}}{\partial \bm{\xi}}\left(c_{sas'}+\gamma v^{\bm{\pi},\bm{\xi}}_{s'}\right)\right].
\end{align*}
Then, we consider the partial derivative on $\bm{\theta}$ and $\bm{\lambda}$ separately. Notice that
$$
\begin{cases}
\frac{\partial J_{\rho}(\bm{\pi},\bm{\xi})}{\partial \theta_{i}} = \frac{1}{1-\gamma}\mathbb{E}_{s\sim\bm{d}^{\bm{\pi},\bm{\xi}}_{\rho}}\mathbb{E}_{a\sim\bm{\pi}_{s\cdot}}\mathbb{E}_{s'\sim\bm{p}_{sa\cdot}}\left[\frac{\partial\log p^{\bm{\xi}}_{sas'}}{\partial \theta_{i}}\left(c_{sas'}+\gamma v^{\bm{\pi},\bm{\xi}}_{s'}\right)\right]\\
\\
\frac{\partial J_{\rho}(\bm{\pi},\bm{\xi})}{\partial \lambda_{sa}} = \frac{1}{1-\gamma}\mathbb{E}_{s\sim\bm{d}^{\bm{\pi},\bm{\xi}}_{\rho}}\mathbb{E}_{a\sim\bm{\pi}_{s\cdot}}\mathbb{E}_{s'\sim\bm{p}_{sa\cdot}}\left[\frac{\partial\log p^{\bm{\xi}}_{sas'}}{\partial \lambda_{sa}}\left(c_{sas'}+\gamma v^{\bm{\pi},\bm{\xi}}_{s'}\right)\right]\\
\end{cases}
$$
We found that for all $(s,a,s')\in\mathcal{S}\times\mathcal{A}\times\mathcal{S}$, 
$\theta_{i}$ will appear in the parametrization form of $p^{\bm{\xi}}_{sas'}$. Hence we consider partial derivative of $\log p^{\bm{\xi}}_{sas'}$ then.
\begin{align*}
    \frac{\partial\log p^{\bm{\xi}}_{sas'}}{\partial \theta_{i}} &= \frac{\partial}{\partial \theta_{i}}\left[\log \bar{\bm{p}}_{sas'} + \frac{\bm{\theta}^{\top}\bm{\phi(s')}}{\lambda_{sa}}\right] - \frac{\partial}{\partial \theta_{i}} \left[\log\left(\sum_{k}\bar{p}_{sak}\cdot\exp({\frac{\bm{\theta}^{\top}\bm{\phi}(k)}{\lambda_{sa}}})\right)\right]\\
    &= \frac{\phi_{i}(s')}{\lambda_{sa}}-\frac{\sum_{k}\bar{p}_{sak}\cdot\exp({\frac{\bm{\theta}^{\top}\bm{\phi}(k)}{\lambda_{sa}}})\cdot\frac{\phi_{i}(k)}{\lambda_{sa}}}{\sum_{k}\bar{p}_{sak}\cdot\exp({\frac{\bm{\theta}^{\top}\bm{\phi}(k)}{\lambda_{sa}}})}\\
    &= \frac{\phi_{i}(s')}{\lambda_{sa}}-\sum_{j}\frac{\bar{p}_{saj}\cdot\exp({\frac{\bm{\theta}^{\top}\bm{\phi}(j)}{\lambda_{sa}}})}{\sum_{k}\bar{p}_{sak}\cdot\exp({\frac{\bm{\theta}^{\top}\bm{\phi}(k)}{\lambda_{sa}}})}\cdot\frac{\phi_{i}(j)}{\lambda_{sa}}\\
    &= \frac{\phi_{i}(s')}{\lambda_{sa}} - \sum_{j}p^{\bm{\xi}}_{saj}\cdot\frac{\phi_{i}(j)}{\lambda_{sa}}.
\end{align*}
Now we can obtain that
\begin{equation*}
    \frac{\partial J_{\rho}(\bm{\pi},\bm{\xi})}{\partial \theta_{i}} = \frac{1}{1-\gamma}\sum_{s}d^{\bm{\pi},\bm{\xi}}_{s}\sum_{a}\pi_{sa}\sum_{s'}p^{\bm{\xi}}_{sas'}\left[\left(\frac{\phi_{i}(s')}{\lambda_{sa}} - \sum_{j}p^{\bm{\xi}}_{saj}\cdot\frac{\phi_{i}(j)}{\lambda_{sa}}\right)\cdot\left(c_{sas'}+\gamma v^{\bm{\pi},\bm{\xi}}_{s'}\right)\right].
\end{equation*}
Similarly we can derive the partial derivative on $\lambda_{sa}$ for any state-action pair $(s,a)$. Interestingly, we notice that for $(\bar{s},\bar{a}) \neq (s,a)$, $\frac{\partial \log(p^{\bm{\xi}}_{\bar{s}\bar{a}s'})}{\partial \lambda_{sa}} = 0$. Therefore, we can consider the case $(\bar{s},\bar{a}) = (s,a)$.
\begin{align*}
    \frac{\partial\log p^{\bm{\xi}}_{sas'}}{\partial \lambda_{sa}} &= \frac{\partial}{\partial \lambda_{sa}}\left[\log \bar{p}_{sas'} + \frac{\bm{\theta}^{\top}\bm{\phi(s')}}{\lambda_{sa}}\right] - \frac{\partial}{\partial \lambda_{sa}} \left[\log\left(\sum_{k}\bar{p}_{sak}\cdot\exp({\frac{\bm{\theta}^{\top}\bm{\phi}(k)}{\lambda_{sa}}})\right)\right]\\
    &= \frac{\sum_{k}\bar{p}_{sak}\cdot\exp\left({\frac{\bm{\theta}^{\top}\bm{\phi}(k)}{\lambda_{sa}}}\right)\cdot\frac{\bm{\theta}^{\top}\bm{\phi}(k)}{\lambda^{2}_{sa}}}{\sum_{k}\bar{p}_{sak}\cdot\exp\left({\frac{\bm{\theta}^{\top}\bm{\phi}(k)}{\lambda_{sa}}}\right)} - \frac{\bm{\theta}^{\top}\bm{\phi}(s')}{\lambda_{sa}^{2}}\\
    &=\sum_{j}\frac{\bar{p}_{saj}\cdot\exp\left({\frac{\bm{\theta}^{\top}\bm{\phi}(j)}{\lambda_{sa}}}\right)}{\sum_{k}\bar{p}_{sak}\cdot\exp\left({\frac{\bm{\theta}^{\top}\bm{\phi}(k)}{\lambda_{sa}}}\right)}\cdot\frac{\bm{\theta}^{\top}\bm{\phi}(j)}{\lambda^{2}_{sa}} - \frac{\bm{\theta}^{\top}\bm{\phi}(s')}{\lambda_{sa}^{2}}\\
    &=  \sum_{j}p^{\bm{\xi}}_{saj}\cdot\frac{\bm{\theta}^{\top}\bm{\phi}(j)}{\lambda^{2}_{sa}}-\frac{\bm{\theta}^{\top}\bm{\phi}(s')}{\lambda_{sa}^{2}}.
\end{align*}
Then we can obtain that
\begin{equation*}
    \frac{\partial J_{\rho}(\bm{\pi},\bm{\xi})}{\partial \lambda_{sa}} = \frac{1}{1-\gamma}d^{\bm{\pi},\bm{\xi}}_{s}\cdot\pi_{sa}\cdot\sum_{s'}p^{\bm{\xi}}_{sas'}\left[\left(\sum_{j}p^{\bm{\xi}}_{saj}\cdot\frac{\bm{\theta}^{\top}\bm{\phi}(j)}{\lambda^{2}_{sa}}-\frac{\bm{\theta}^{\top}\bm{\phi}(s')}{\lambda_{sa}^{2}}\right)\cdot\left(c_{sas'}+\gamma v^{\bm{\pi},\bm{\xi}}_{s'}\right)\right].
\end{equation*}
\end{proof}

\section{Experiment Details}
\label{app: Details in experiment}
\subsection{Details on the Garnet problem example}
\begin{figure}[H]
\centering
%\vspace{.3in}
\includegraphics[height=2.3in,width=0.5\textwidth]{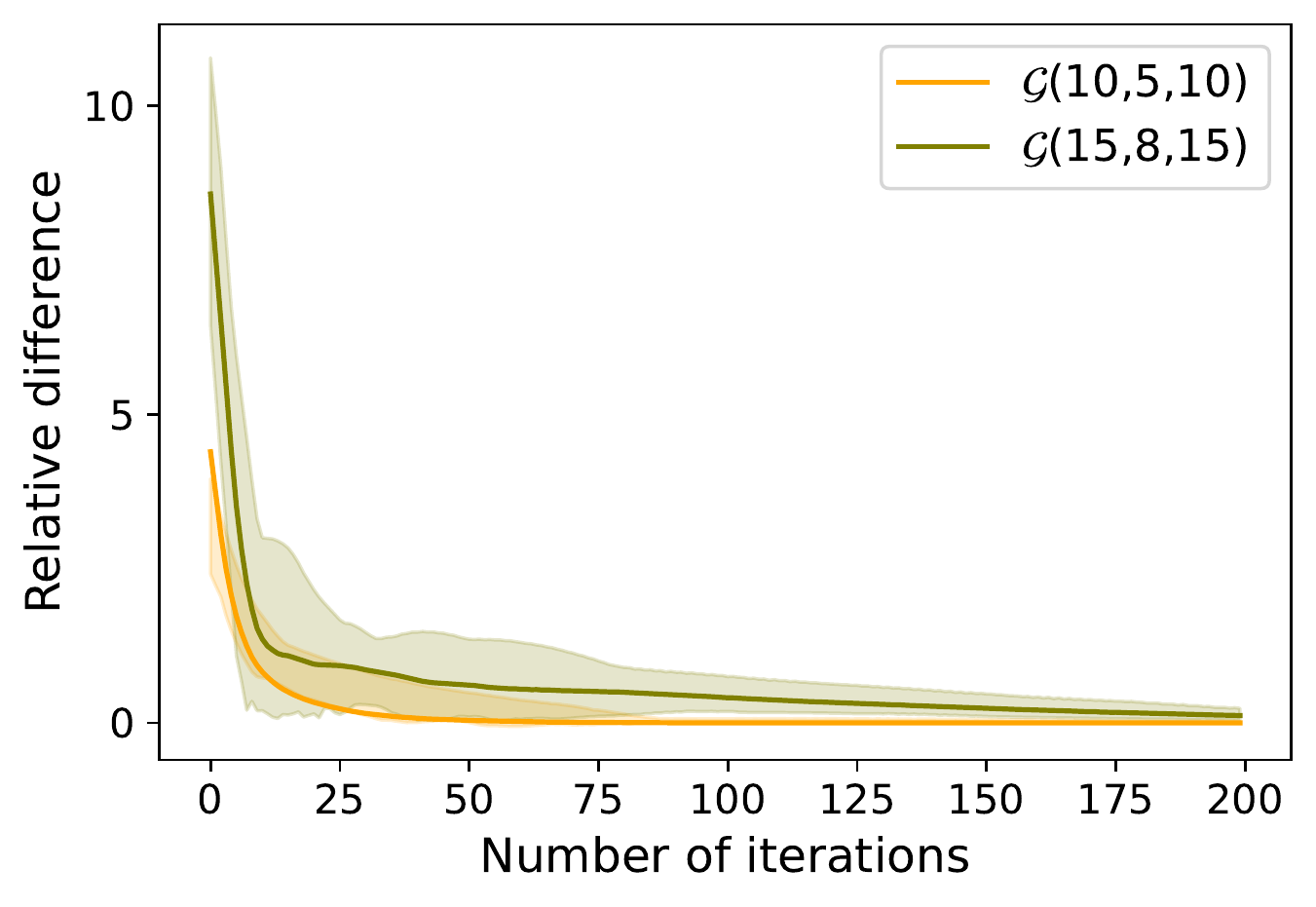}
%\vspace{.3in}
\caption{The error of value function computed by non-parametric DRPG for two Garnet problems with $s$-rectangular ambiguity.}
\label{fig_appen_s_comparison}
\end{figure}
Note that, in our simulations, we test our algorithm for both high connectivity (\ie, b = $|\mathcal{S}|$) in $s$-rectangular case, and low connectivity (\ie, b = $|\mathcal{S}|/5$) in $(s,a)$-rectangular case. We also apply DRPG on random RMDPs with $L_1$-constrained $s$-rectangular ambiguity, which generally assumes the uncertain in transition probabilities is independent for each state-action pair and are defined as 
\begin{equation*}
\mathcal{P}=\underset{s \in \mathcal{S}}{\times} \mathcal{P}_{s} \quad \text { where } \quad \mathcal{P}_{s}:=\left\{\left(\bm{p}_{s 1}, \ldots, \bm{p}_{s A}\right)\in(\Delta^{S})^{A} \mid \sum_{a\in\mathcal{A}}\|\bm{p}_{sa}-\bar{\bm{p}}_{sa}\|_{1} \leq \kappa_{s}\right\}.
\end{equation*}

We run DRPG with a sample size of 50 for 200 iteration times on Garnet problems with three sizes for the $(s,a)$-rectangular case and two medium sizes for the $s$-rectangular case. We record the absolute value of gaps between objective values of DRPG and robust value iteration at each iteration time step, and then plot the relative difference under the $s$-rectangular assumption in Figure~\ref{fig_appen_s_comparison}. The upper and lower envelopes of the curves correspond to the 95 and 5 percentiles of the 50 samples, respectively. From Figure~\ref{fig_appen_s_comparison}, we can obtain similar results with the $(s,a)$-rectangular case that DRPG converges to a nearly identical value computed by
the value iteration computed by the robust value iteration.

\subsection{Details on the inventory management example}
In our inventory management example, we present a specific example of this problem with eight states  and three actions. 

We draw the cost for each $(s,a,s)\in\mathcal{S}\times\mathcal{A}\times\mathcal{S}$ at random uniformly in $[0,1]$, and we fix a discount factor $\gamma = 0.95$. Below we give details about the nominal transitions and the parameter $\bm{\kappa}$.

The feature function we use is the radial-type features which is introduced in~\cite{sutton2018reinforcement}, \ie, $\bm{\phi}_{i}(s) = \exp \left(-\frac{\left\|s-c_i\right\|^2}{2 \sigma_i^2}\right)$. We define a two-dimension state feature with deterministic $c_{i}$ and $\sigma_{i}$. Our parameters also share the same dimensions as these two features from our parameterization form. 

The ambiguity set $\Xi$ in our problem is simply chosen as a $L_{1}$-norm constrained set, that is,
\begin{equation}
    \Xi:= \{(\bm{\theta},\bm{\lambda})|\|\bm{\theta} - \bm{\theta}_{c}\|_{1}\leq\kappa_{\theta},\|\bm{\lambda}-\bm{\lambda}_{c}\|_{1}\leq\kappa_{\lambda}\}.
\end{equation}
The updating step size for $\bm{\xi}=(\bm{\theta},\bm{\lambda})$ on the inner problem are taken 0.01. For simplicity, we choose all elements of $\bm{\lambda}_{c}$ as one and $\bm{\theta}_{c}:=[0.4,0.9]^{\top}$, and set $\kappa_{\theta}=1,\kappa_{\lambda}=1$ in this problem. Other parameters are included in the published codes. Note that the instances for a larger number of states are constructed in the same fashion by adding
some condition states. 

%%%%%%%%%%%%%%%%%%%%%%%%%%%%%%%%%%%%%%%%%%%%%%%%%%%%%%%%%%%%%%%%%%%%%%%%%%%%%%%
%%%%%%%%%%%%%%%%%%%%%%%%%%%%%%%%%%%%%%%%%%%%%%%%%%%%%%%%%%%%%%%%%%%%%%%%%%%%%%%

\end{document}